\documentclass[journal,twoside,twocolumn]{IEEEtran}

\usepackage{amsmath,amsfonts,amssymb,amsthm,mathtools}
\usepackage{bbm,dsfont,verbatim}
\usepackage{cite}
\usepackage{url}            
\usepackage{booktabs}       
\usepackage{array}
\usepackage[pdftex]{graphicx}
\usepackage[utf8]{inputenc} 
\usepackage{hyperref}  
\hypersetup{colorlinks = true, citecolor = blue, linkcolor = red, urlcolor = blue}     
\usepackage[all]{hypcap}
\usepackage{xcolor}					

\usepackage{breakurl}
\usepackage{caption}
\usepackage{subcaption}
\usepackage{eqparbox} 
\usepackage{algorithm}
\usepackage{algpseudocode} 
\algnewcommand\True{\textbf{true}\space}
\algnewcommand\False{\textbf{false}\space}
\makeatletter 
\algnewcommand{\StatexInd}[1]{\Statex \hskip\ALG@thistlm #1} 
\usepackage{flushend}

\allowdisplaybreaks
\def\R{{\mathbb{R}}}
\def\N{{\mathbb{N}}}
\def\Z{{\mathbb{Z}}}
\def\1{{\textbf{1}}}
\def\0{{\textbf{0}}}
\def\I{{\mathds{1}}}
\def\P{{\mathbb{P}}}
\def\E{{\mathbb{E}}}

\def\A{{\mathcal{A}}}
\def\cP{{\mathcal{P}}}

\def\S{{\mathcal{S}}}
\def\L{{\mathcal{L}}}
\def\net{{\mathcal{N}}}
\def\M{{\mathcal{M}}}

\def\FedAve{{\mathsf{FedAve}}}
\def\lrgd{{\mathsf{FedLRGD}}}
\DeclareMathOperator\diff{{d \!}}

\DeclareMathOperator\rank{{rank}}
\DeclareMathOperator\vol{{vol}}

\DeclareMathOperator\subpoly{{sub-poly}}

\DeclareMathOperator*{\argmin}{arg\,min}

\newcommand{\T}{\mathrm{T}}
\newcommand{\Fro}{\mathrm{F}}
\newcommand{\op}{\mathrm{op}}

\newtheorem{theorem}{Theorem}
\newtheorem{lemma}{Lemma}
\newtheorem{proposition}{Proposition}

\theoremstyle{definition}
\newtheorem{definition}{Definition}

\newcommand{\figref}[1]{Figure~\ref{#1}}

\newcommand{\thmref}[1]{Theorem~\ref{#1}}
\newcommand{\propref}[1]{Proposition~\ref{#1}}

\newcommand{\lemref}[1]{Lemma~\ref{#1}}

\newcommand{\defref}[1]{Definition~\ref{#1}}
\newcommand{\algoref}[1]{Algorithm~\ref{#1}}
\newcommand{\secref}[1]{Section~\ref{#1}}
\newcommand{\appref}[1]{Appendix~\ref{#1}}

\begin{document}

\bstctlcite{IEEEexample:BSTcontrol} 

\title{Federated Optimization of Smooth Loss Functions}

\author{Ali~Jadbabaie,~Anuran~Makur,~and~Devavrat~Shah%
\thanks{The author ordering is alphabetical. This research was supported in part by the ONR under Grant N000142012394, in part by the ARO MURI under Grant W911NF-19-1-0217, and in part by the NSF TRIPODS Foundations of Data Science.}%
\thanks{A. Jadbabaie is with the Laboratory for Information and Decision Systems, Massachusetts Institute of Technology, Cambridge, MA 02139, USA (e-mail: jadbabai@mit.edu).}%
\thanks{A. Makur is with the Department of Computer Science and the Elmore Family School of Electrical and Computer Engineering, Purdue University, West Lafayette, IN 47907, USA (e-mail: amakur@purdue.edu).}%
\thanks{D. Shah is with the Laboratory for Information and Decision Systems, Massachusetts Institute of Technology, Cambridge, MA 02139, USA (e-mail: devavrat@mit.edu).}}%

\maketitle

\begin{abstract}
In this work, we study empirical risk minimization (ERM) within a federated learning framework, where a central server seeks to minimize an ERM objective function using $n$ samples of training data that is stored across $m$ clients and the server. The recent flurry of research in this area has identified the Federated Averaging ($\FedAve$) algorithm as the staple for determining $\epsilon$-approximate solutions to the ERM problem. Similar to standard optimization algorithms, e.g., stochastic gradient descent, the convergence analysis of $\FedAve$ and its variants only relies on smoothness of the loss function in the optimization parameter. However, loss functions are often very smooth in the training data too. To exploit this additional smoothness in data in a federated learning context, we propose the Federated Low Rank Gradient Descent ($\lrgd$) algorithm. Since smoothness in data induces an approximate low rank structure on the gradient of the loss function, our algorithm first performs a few rounds of communication between the server and clients to learn weights that the server can use to approximate clients' gradients using its own gradients. Then, our algorithm solves the ERM problem at the server using an inexact gradient descent method. To theoretically demonstrate that $\lrgd$ can have superior performance to $\FedAve$, we present a notion of federated oracle complexity as a counterpart to canonical oracle complexity in the optimization literature. Under some assumptions on the loss function, e.g., strong convexity and smoothness in the parameter, $\eta$-H\"{o}lder class smoothness in the data, etc., we prove that the federated oracle complexity of $\lrgd$ scales like $\phi m (p/\epsilon)^{\Theta(d/\eta)}$ and that of $\FedAve$ scales like $\phi m (p / \epsilon)^{3/4}$ (neglecting typically sub-dominant factors), where $\phi \gg 1$ is the ratio of client-to-server communication time to gradient computation time, $p$ is the parameter dimension, and $d$ is the data dimension. Then, we show that when $d$ is small compared to $n$ and the loss function is sufficiently smooth in the data, i.e., $\eta = \Theta(d)$, $\lrgd$ beats $\FedAve$ in federated oracle complexity. Finally, in the course of analyzing $\lrgd$, we also establish a general result on low rank approximation of smooth latent variable models.
\end{abstract}


\section{Introduction}
\label{Introduction}

Optimization for machine learning has become a prevalent subject of study due to the advent of large scale model learning tasks. Indeed, while optimization theory is a relatively mature field in applied mathematics (cf. \cite{BoydVandenberghe2004,Nesterov2004}), many of its recent developments have focused on its application to specialized machine learning settings. For instance, several common machine learning problems, such as classification and regression, can be posed as \emph{empirical risk minimization} (ERM) problems. In the unconstrained case, ERM problems have the form \cite{HastieTibshiraniFriedman2009}, \cite[Section 1]{Bubeck2015}:
\begin{equation}
\label{Eq: Standard formulation}
\min_{\theta \in \R^p}{F(\theta)} \triangleq \frac{1}{n} \sum_{i = 1}^{n}{f(x^{(i)};\theta)} \, ,
\end{equation}
where $\theta \in \R^p$ is the parameter over which we optimize, $x^{(1)},\dots,x^{(n)} \in [0,1]^d$ are some given training data, $f : [0,1]^d \times \R^p \rightarrow \R$, $f(x;\theta)$ denotes the loss function, and $F : \R^p \rightarrow \R$ defines the empirical risk (or objective function). There is a gargantuan literature that develops and analyzes iterative gradient-based optimization methods that solve the specific optimization in \eqref{Eq: Standard formulation}, e.g., (exact and projected) GD \cite{Nesterov2004,Bubeck2015}, inexact GD \cite{FriedlanderSchmidt2012,DevolderGlineurNesterov2013,SoZhou2017,SchmidtLeRouxBach2011}, stochastic gradient descent (SGD) \cite{Nemirovskietal2009}, GD with momentum \cite{Polyak1964}, accelerated GD \cite{Nesterov1983}, mini-batch SGD \cite{Dekeletal2012}, variance reduction methods \cite{LeRouxSchmidtBach2012,SchmidtLeRouxBach2017,ShalevShwartzZhang2013,JohnsonZhang2013}, and methods with adaptive learning rates \cite{DuchiHazanSinger2011,TielemanHinton2012,KingmaBa2015} (also see the recent monographs \cite{Bubeck2015,JainKar2017,BottouCurtisNocedal2018,Netrapalli2019}, and the references therein). 

However, as observed in \cite{JadbabaieMakurShah2021}, the optimization for machine learning literature typically analyzes convergence and oracle complexity of iterative gradient-based algorithms for \eqref{Eq: Standard formulation} by posing the problem as, cf. \cite[Section 1]{Bubeck2015}, \cite{HastieTibshiraniFriedman2009}:
\begin{equation}
\label{Eq: Standard formulation 2}
\min_{\theta \in \R^p}{\frac{1}{n} \sum_{i = 1}^{n}{f_i(\theta)}} \, ,
\end{equation}
where each $f_i(\theta) = f(x^{(i)};\theta)$. Hence, canonical analyses of iterative methods for the representation \eqref{Eq: Standard formulation 2} do not effectively utilize information about the smoothness of the loss function $f(x;\theta)$ with respect to the data $x$. As noted in \cite{JadbabaieMakurShah2021}, such smoothness properties are actually quite pervasive in ERM problems, e.g., (regularized) linear regression \cite[Chapter 3]{HastieTibshiraniFriedman2009}, (regularized) logistic regression \cite[Section 4.4]{HastieTibshiraniFriedman2009}, deep neural network training with smooth loss and activation functions \cite[Chapter 5]{Bishop2006}, etc. To address the paucity of theoretical analyses that exploit this additional smoothness information, \cite{JadbabaieMakurShah2021} initiates a study of inexact gradient descent methods for ERM problems that learn the gradient oracle at every iteration by performing local polynomial regressions \cite{FanGijbels1996,ClevelandLoader1996,Tsybakov2009}.

In this paper, we continue this important thread of exploiting the smoothness of loss functions in training data to improve optimization algorithms. But in contrast to \cite{JadbabaieMakurShah2021}, we consider the newly developed federated learning paradigm here, and utilize the approximate low rank structure induced by smoothness in data to theoretically ameliorate optimization performance in this paradigm. We briefly introduce the federated learning architecture and formally abstract it in \secref{Federated Learning}, and we subsequently explain our main contributions in \secref{Main Contributions}. 

\subsection{Federated Learning}
\label{Federated Learning}

The recent escalation of mobile phone usage and emergence of the Internet of Things (IoT) has engendered unprecedented volumes of collected data on such devices. While it is clearly of utility to train machine models based on this data, various new challenges pertaining to privacy, scalability, distributed computation, etc. must be addressed before doing so. For instance, data collected on a user's mobile phone must be kept private, and even if privacy is not a concern, it is infeasible to communicate large quantities of data from millions of users to a central server for computation. To deal with such issues, a new paradigm known as \emph{federated learning} has been formulated in the literature \cite{McMahanetal2017}. There are several classes of federated learning architectures (see, e.g., \cite{Kairouzetal2021}): those with a synchronous central server, fully decentralized networks with synchronized clocks, and asynchronous versions of these topologies. We consider synchronous architectures with a central server in this paper.

In this synchronous setting, the architecture consists of a single central \emph{server} and a collection of several \emph{clients} (that all have synchronized clocks), cf. \cite{McMahanetal2017,Reisizadehetal2020a,Reisizadehetal2020b}, and the references therein. Typically, the server and clients hold their own private data. So, to train machine learning models, \emph{federated optimization} is carried out by clients training possibly inaccurate models based on their private data, and the server periodically aggregating client models to produce a more accurate model at the server. Several key dimensions are considered when studying such federated learning and optimization systems:
\begin{enumerate}
\item \textbf{Privacy:} The training data resides with the clients and is not shared with the server (and vice versa) \cite{McMahanetal2017}. However, the server must leverage the data stored at the clients via several communication rounds to improve the statistical accuracy of the learnt model. 
\item \textbf{Large client population with possibly inactive clients:} Federated learning systems usually have a very large number of clients so that the client population size often dominates the number of training samples per client \cite{McMahanetal2017}. On the other hand, a nontrivial fraction of clients could be inactive during the learning process, e.g., mobile phones only send updates when they are online and being charged \cite{Reisizadehetal2020a}. So, federated optimization algorithms must be robust to participation from smaller subsets of clients.
\item \textbf{Communication cost:} There is significant cost associated with clients communicating with the server. Typically, clients have limited upload bandwidth and servers have limited download bandwidth. In particular, if a large number of clients seek to simultaneously communicate with the server, the server will put the clients in a queue to receive their updates. So, federated optimization schemes often try to increase the amount of (cheap) computation at each client in order to reduce the frequency of communication with the server (e.g., by employing several local updates at the clients), and utilize smaller fractions of active clients in certain rounds \cite{McMahanetal2017,Reisizadehetal2020a}. (Moreover, such schemes may also utilize quantization to compress updates from clients to servers \cite{Reisizadehetal2020a}; we do not consider data compression in this work.)
\item \textbf{Latency:} Beyond the communication costs delineated above, the duration of client-to-server updates is determined by the slowest client among the set of active clients. Hence, using smaller fractions of active clients in certain rounds is also advantageous to eliminate stragglers. In fact, recent algorithms and analysis suggest that stragglers should be used at later stages of the learning process in order to reduce latency \cite{Reisizadehetal2020b}. 
\item \textbf{Data heterogeneity:} Different clients may have data that is drawn from different population distributions \cite{McMahanetal2017,Kairouzetal2021}. Although some works in the literature, e.g., \cite{Reisizadehetal2020a,Reisizadehetal2020b}, assume for simplicity of analysis that clients are homogeneous, i.e., the training data at every client is drawn from a common (unknown) probability distribution, we do not need to impose such a stringent assumption in this work. In particular, our new algorithm in \algoref{Alg: FedLRGD} and its theoretical analysis in \thmref{Thm: Federated Oracle Complexity of FedLRGD} and \secref{Analysis of Federated Low Rank Gradient Descent} will hold in the general heterogeneous setting. Usually, data homogeneity simplifies the analysis of stochastic federated optimization methods since clients' gradients become unbiased estimators of true full gradients \cite{Reisizadehetal2020a}. However, since we analyze the (sample-version of the) ERM problem given below in \eqref{Eq: Standard formulation 3}, which has the same form in both homogeneous and heterogeneous settings, and we will not use randomized estimators of true full gradients at the clients, our analysis is agnostic to whether the server and client data is homogeneous or heterogeneous.
\end{enumerate}

Formally, we model synchronous federated learning with a single central server as follows. We assume that there are $m \in \N$ clients, indexed by $[m]$. The server is endowed with $r \in \Z_+$ training data samples $x^{(0,1)},\dots,x^{(0,r)} \in [0,1]^d$, and each client $i \in [m]$ is given $s \in \N$ training data samples $x^{(i,1)},\dots,x^{(i,s)} \in [0,1]^d$, where $n = ms + r$ so that $\{x^{(i,j)} : i \in \{0\} \cup [m], \, j \in [s]\} = \{x^{(k)} : k \in [n]\}$.\footnote{It is assumed for simplicity that every client has the same number of training data samples.} Under the assumptions outlined in \secref{Assumptions on Loss Function}, the goal of the server is to utilize the clients to compute an $\epsilon$-approximate solution (in the sense of \eqref{Eq: Approx solution 1} or \eqref{Eq: Approx solution 2}) to the ERM problem \eqref{Eq: Standard formulation}, which may be rewritten as
\begin{equation}
\label{Eq: Standard formulation 3}
\min_{\theta \in \R^p}{F(\theta)} = \frac{1}{n} \sum_{j = 1}^{r}{f(x^{(0,j)};\theta)} + \frac{1}{n} \sum_{i = 1}^{m}{\sum_{j = 1}^{s}{f(x^{(i,j)};\theta)}} \, .
\end{equation}

To achieve this goal, we consider a specific family of federated optimization algorithms. In particular, any such algorithm runs over $T \in \N$ synchronized epochs (or rounds). During each epoch, the server may first communicate with $\tau m \in \N$ clients, where $\tau \in [0,1]$ denotes the proportion of active clients that participate in computation. Then, the server and the active clients perform gradient computations, or other less expensive computations, towards finding an $\epsilon$-approximate solution to \eqref{Eq: Standard formulation 3}, where we typically assume that all active clients perform the same number of gradient computations. Finally, each active client may communicate only one vector in $\R^p$ to the server (while taking care not to reveal their private data by transmitting vectors from which it could be easily extracted). This client-to-server communication has some nontrivial cost associated with it. On the other hand, it is usually assumed that server-to-client communication has negligible cost, because there are often no critical constraints on client download bandwidths and the server broadcasting bandwidth. So, the server broadcasts vectors in $\R^p$ to active clients during epochs at essentially no cost. We also note that every epoch need not contain the three aforementioned phases, i.e., the server need not always communicate with active clients or perform gradient computations, and clients need not always perform gradient computations and communicate with the server. 

Several federated optimization algorithms for the synchronous central server setting have been proposed in the literature. The benchmark method in the field is known as \emph{Federated Averaging} ($\FedAve$) \cite{McMahanetal2017}, where clients perform several iterations of (mini-batch) SGD using their private data, and the server aggregates client updates by periodically averaging them over several epochs. Needless to say, $\FedAve$ belongs to the aforementioned family of algorithms, and has been extended in several directions. For example, the authors of \cite{Huoetal2020} introduce an accelerated version of the algorithm by introducing momentum, and the authors of \cite{Reisizadehetal2020a} include a compression (or quantization) step before any client-to-server communication in the algorithm to reduce communication cost. In particular, the iteration complexity and convergence of $\FedAve$ are carefully analyzed in \cite{Reisizadehetal2020a}. More generally, a unified analysis of the class of communication-efficient SGD algorithms is presented in \cite{WangJoshi2019}. Various other federated optimization methods have also been proposed that address different drawbacks of $\FedAve$. For instance, federated optimization methods with better fixed point properties are developed in \cite{PathakWainwright2020}, latency of federated optimization algorithms is improved via adaptive node participation in \cite{Reisizadehetal2020b}, a federated optimization algorithm for training with only positive labels is presented in \cite{Yuetal2020}, and adaptive algorithms that incorporate more sophisticated gradient-based iterative methods, such as Adam \cite{KingmaBa2015}, into the federated learning framework are proposed in \cite{Reddietal2021}. In a different vein, one-shot federated learning is considered in \cite{GuhaTalwakarSmith2019}, where only one round of client-to-server communication is used in the optimization process (also see follow-up work). We refer readers to \cite{Kairouzetal2021} and the references therein for an in-depth further discussion of federated learning, distributed optimization, and related ideas.

Finally, we note that in the traditional federated learning literature, only the clients possess training data which the server cannot access. However, as mentioned above, our work assumes that the server also possesses some training data samples. As explained in \cite{Maietal2023}, this is a reasonable assumption in many federated learning settings because the server could also function as a client and have access to its own private training data, it may have to collect some training data from clients (despite privacy concerns) for system monitoring purposes, it may collect data from test devices (or test clients) which do not participate in training, and it could even possess simulated training data from models. For these reasons, there is a growing recent literature on federated learning with server data, cf. \cite{YangShen2022,MaiLaZhang2023,Maietal2023} and the references therein. Furthermore, as remarked earlier, the server and client datasets can be heterogeneous in our work since the smoothness (or approximate low rankness) of loss functions will circumvent the need for data homogeneity in our analysis.

\subsection{Main Contributions}
\label{Main Contributions}

As mentioned earlier, the vast literature on federated optimization, including the aforementioned works, deals with the ERM problem representation in \eqref{Eq: Standard formulation 2}. Thus, these works do not exploit smoothness of loss functions in the training data for federated learning. In this paper, we fill this void by making the following key theoretical contributions: 
\begin{enumerate}
\item We define a notion of \emph{federated oracle complexity} in \defref{Def: Federated oracle complexity} to capture the running time of federated optimization algorithms, and argue its suitability using some simple probabilistic analysis. This formally generalizes classical oracle complexity of iterative optimization methods (see \cite[Section 1.1]{Nesterov2004}) to a federated learning setting. Our definition allows us to theoretically compare our proposed federated optimization algorithm (mentioned next) with the benchmark $\FedAve$ method from the literature. We believe that this definition will be of further utility in future research in the area.
\item We propose a new \emph{Federated Low Rank Gradient Descent} ($\lrgd$) method in \algoref{Alg: FedLRGD} to solve the ERM problem \eqref{Eq: Standard formulation}, or equivalently, \eqref{Eq: Standard formulation 3}, in the federated learning setting. Our method exploits the approximate low rankness of the gradient of the loss function, which is a consequence of its smoothness in data, by using a few rounds of communication to first learn weights that allow the server to approximate clients' gradients by computing weighted sums of its own gradients. Then, the server can solve \eqref{Eq: Standard formulation} by running an inexact GD method on its own, where the server estimates full gradients at each iteration from gradients corresponding to its private data. Moreover, like other federated optimization algorithms in the literature, our method preserves the privacy of the server and clients by never exchanging data samples, and it works in settings where server and client data is heterogeneous.
\item Under the smoothness, strong convexity, and non-singularity assumptions outlined in \secref{Assumptions on Loss Function} and \secref{Federated Oracle Complexity of FedLRGD}, we analyze the $\lrgd$ algorithm and show in \thmref{Thm: Federated Oracle Complexity of FedLRGD} that its federated oracle complexity scales like $\phi m (p/\epsilon)^{\Theta(d/\eta)}$ (neglecting typically sub-dominant factors), where $\phi \gg 1$ is the ratio of communication cost to computation cost (see \secref{Federated Oracle Complexity}), $\epsilon > 0$ is the desired approximation accuracy of the solution (see \secref{Assumptions on Loss Function}), and $\eta > 0$ denotes a H\"{o}lder class parameter which determines the level of smoothness of the loss function in the data (see \secref{Assumptions on Loss Function}). 
\item We also calculate the federated oracle complexity of the $\FedAve$ algorithm in \thmref{Thm: Federated Oracle Complexity of FedAve} by building on known results from the literature, cf. \cite{Reisizadehetal2020a}, and show that it scales like $\phi m (p/ \epsilon)^{3/4}$. Using this calculation, we demonstrate in \propref{Prop: Comparison between FedLRGD and FedAve} that when the data dimension is small, e.g., $d = O(\log(n)^{0.99})$, and the gradient of the loss function is sufficiently smooth in the training data so that it is approximately low rank, i.e., $\eta = \Theta(d)$, the federated oracle complexity of $\lrgd$ is much smaller than that of the benchmark $\FedAve$ algorithm.
\item In the process of analyzing the $\lrgd$ algorithm, we also derive a new statistical result in \thmref{Thm: Low Rank Approximation} that utilizes the smoothness of loss functions in the training data to bound the Frobenius norm error in estimating a latent variable model (or graphon) corresponding to the loss function with its low rank approximation. This result generalizes the graphon approximation results in \cite{Chatterjee2015,Xu2018,Agarwaletal2021}, and could be of independent utility for future research in high-dimensional statistics.
\end{enumerate}

\subsection{Outline}
\label{Outline}

We briefly delineate the remainder of our discussion here. In \secref{Formal Setup}, we precisely state and explain the notation and assumptions used in our subsequent analysis, and also introduce and expound upon the concept of federated oracle complexity. In \secref{Main Results and Discussion}, we present and discuss our main technical results and the $\lrgd$ algorithm mentioned above. Then, we provide proofs of our results on latent variable model approximation and the federated oracle complexity of $\lrgd$ in \secref{Analysis of Latent Variable Model Approximation} and \secref{Analysis of Federated Low Rank Gradient Descent}, respectively. The proofs of all other results are deferred to the appendices. Although the focus of this work is to theoretically analyze $\lrgd$ and demonstrate the resulting gain in federated oracle complexity, we also present some pertinent simulation results in \appref{Low Rank Structure in Neural Networks}. These simulations illustrate that gradients of loss functions corresponding to neural networks trained on the MNIST database for handwritten digit recognition (see \cite{LeCunCortesBurgesMNIST}) and CIFAR-10 dataset of tiny images (see \cite{KrizhevskyNairHintonCIFAR10}) exhibit approximate low rankness. This suggests that algorithms which exploit low rank structure, such as $\lrgd$, could have potential impact in applications. Finally, we conclude our discussion in \secref{Conclusion}.

\section{Formal Setup}
\label{Formal Setup}

Before deriving our main results, we need to present some essential formalism. To that end, \secref{Notational Preliminaries} contains various notational conventions used in this work, \secref{Assumptions on Loss Function} illustrates the precise conditions needed for our results to hold, \secref{Reasoning and Background Literature for Assumptions} explains the reasoning behind these conditions and provides some related literature, and \secref{Federated Oracle Complexity} develops the notion of federated oracle complexity which measures the running time of federated optimization algorithms.

\subsection{Notational Preliminaries}
\label{Notational Preliminaries}

In this paper, we let $\N \triangleq \{1,2,3,\dots\}$, $\Z_+ = \N \cup \!\{0\}$, and $\R_+$ denote the sets of natural numbers, non-negative integers, and non-negative real numbers, respectively, and for $m \in \N$, we let $[m] \triangleq \{1,\dots,m\}$. We let $\P(\cdot)$ and $\E[\cdot]$ denote the probability and expectation operators, respectively, where the underlying probability laws will be clear from context. Furthermore, $\exp(\cdot)$ and $\log(\cdot)$ denote the natural exponential and logarithm functions with base $e$, respectively, $\I\{\cdot\}$ denotes the indicator function that equals $1$ if its input proposition is true and equals $0$ otherwise, and $\lceil \cdot \rceil$ denotes the ceiling function. Lastly, we will utilize Bachmann-Landau asymptotic notation throughout this paper, e.g., $O(\cdot)$, $\Omega(\cdot)$, $\Theta(\cdot)$, etc., where the limiting variable will be clear from context.

In addition, for any matrix $A \in \R^{m \times k}$ with $m,k \in \N$, $A_{i,j}$ denotes the $(i,j)$th entry of $A$ for $i \in [m]$ and $j \in [k]$, $A^{\T}$ denotes the transpose of $A$, $A^{-1}$ denotes the inverse of $A$ (when $m = k$ and $A$ is non-singular), $\rank(A)$ denotes the rank of $A$, $\|A\|_{\Fro}$ denotes the Frobenius norm of $A$, and $\|A\|_{\op}$ denotes the spectral or (induced $\ell^2$) operator norm of $A$. Similarly, for any (column) vector $y \in \R^m$, $y_i$ denotes the $i$th entry of $y$ for $i \in [m]$, and $\|y\|_q$ denotes the $\ell^q$-norm of $y$ for $q \in [1,+\infty]$. Moreover, for any continuously differentiable function $h:\R^m \rightarrow \R$, $h(y)$, we define $\partial h/\partial y_i : \R^m \rightarrow \R$ to be its partial derivative with respect to $y_i$ for $i \in [m]$, and $\nabla_y h : \R^m \rightarrow \R^m$, $\nabla_y h = \left[\partial h/\partial y_1 \, \cdots \, \partial h/\partial y_m\right]^{\T}$ to be its gradient. As in \cite{JadbabaieMakurShah2021}, we will also require some multi-index notation. So, for any $m$-tuple $s = (s_1,\dots,s_m) \in \Z_+^m$, we let $|s| = \|s\|_{1} = s_1 + \cdots + s_m$, $s! = s_1! \cdots s_m!$, $y^s = y_1^{s_1} \cdots y_m^{s_m}$ for $y \in \R^m$, and $\nabla^s h = \nabla_y^s h = \partial^{|s|} h/(\partial y_1^{s_1} \cdots \partial y_m^{s_m})$ be the $s$th partial derivative of $h$ with respect to $y$. 

\subsection{Assumptions on Loss Function}
\label{Assumptions on Loss Function}

As mentioned earlier, for any $n,d,p \in \N$, assume that we are given $n$ samples of training data $x^{(1)},\dots,x^{(n)} \in [0,1]^d$ and a loss function $f : [0,1]^d \times \R^p \rightarrow \R$, $f(x;\theta)$, where $\theta \in \R^p$ is the parameter vector. In our ensuing analysis, we will impose the following assumptions on the loss function (cf. \cite{Nesterov2004,Bubeck2015,JadbabaieMakurShah2021}):
\begin{enumerate}
\item \textbf{Smoothness in parameter $\theta$:} There exists $L_1 > 0$ such that for all $x \in [0,1]^{d}$, $\nabla_{\theta} f(x;\cdot) : \R^p \rightarrow \R^p$ exists and is $L_1$-Lipschitz continuous:
\begin{align}
\forall \theta^{(1)},\theta^{(2)} \in \R^p, \, & \left\|\nabla_{\theta} f(x;\theta^{(1)}) - \nabla_{\theta} f(x;\theta^{(2)})\right\|_2 \nonumber \\ 
& \qquad \qquad \enspace \leq L_1 \left\|\theta^{(1)} - \theta^{(2)}\right\|_2 . 
\end{align}
\item \textbf{Strong convexity:} There exists $\mu > 0$ such that for all $x \in [0,1]^{d}$, $f(x;\cdot):\R^p \rightarrow \R$ is $\mu$-strongly convex:
\begin{align}
\forall \theta^{(1)},\theta^{(2)} & \in \R^p, \nonumber \\
f(x;\theta^{(1)}) & \geq f(x;\theta^{(2)}) + \nabla_{\theta} f (x;\theta^{(2)})^{\T} (\theta^{(1)} - \theta^{(2)}) \nonumber \\
& \quad \, + \frac{\mu}{2} \left\|\theta^{(1)} - \theta^{(2)}\right\|_2^2 . 
\end{align}
Furthermore, we let $\kappa \triangleq L_1/\mu \geq 1$ be the \emph{condition number} of the loss function.
\item \textbf{Smoothness in data $x$:} There exist $\eta > 0$ and $L_2 > 0$ such that for all fixed $\vartheta \in \R^{p}$ and for all $i \in [p]$, the $i$th partial derivative of $f$ at $\vartheta$, denoted $g_{i}(\cdot;\vartheta) \triangleq \frac{\partial f}{\partial \theta_i}(\cdot;\vartheta) : [0,1]^d \rightarrow \R$, belongs to the $(\eta,L_2)$-\emph{H\"{o}lder class} (cf. \cite[Definition 1.2]{Tsybakov2009}), i.e., $g_{i}(\cdot;\vartheta) : [0,1]^d \rightarrow \R$ is $l = \lceil \eta \rceil - 1$ times differentiable, and for every $s = (s_1,\dots,s_d) \in \Z_+^d$ such that $|s| = l$, we have
\begin{align}
\forall y^{(1)},y^{(2)} \in [0,1]^d, & \enspace \left|\nabla_x^s g_i(y^{(1)};\vartheta) - \nabla_x^s g_i(y^{(2)};\vartheta)\right| \nonumber \\
& \quad \quad \leq L_2 \left\|y^{(1)} - y^{(2)}\right\|_1^{\eta - l} . 
\label{Eq: Holder condition}
\end{align}
\item \textbf{Non-singularity:} For any (large) size $r \in \N$, any index $i \in [p]$, any sequence of distinct data samples $y^{(1)},\dots,y^{(r)} \in [0,1]^d$, and any sequence of distinct parameter values $\vartheta^{(1)},\dots,\vartheta^{(r)} \in \R^p$,\footnote{More formally, this assumption only holds for $\vartheta^{(1)},\dots,\vartheta^{(r)}$ \emph{almost everywhere}, but we neglect this detail here and in the sequel for simplicity.} the matrix of partial derivatives $G^{(i)} \in \R^{r \times r}$ given by:
\begin{equation}
\label{Eq: Matrix of partial derivatives}
\forall j,k \in [r], \enspace G^{(i)}_{j,k} = g_i(y^{(j)};\vartheta^{(k)}) \, ,
\end{equation}
is invertible. Note that we suppress the dependence of $G^{(i)}$ on $r$ and the sequences of data and parameters for convenience.
\item \textbf{Conditioning:} There exists an absolute constant $\alpha \geq 1$ such that for any (large) $r \in \N$, any $i \in [p]$, any sequence of distinct data samples $y^{(1)},\dots,y^{(r)} \in [0,1]^d$, and any sequence of distinct parameter values $\vartheta^{(1)},\dots,\vartheta^{(r)} \in \R^p$,\footnote{More formally, one would need to impose additional distributional assumptions to prove that this assumption holds for random $\vartheta^{(1)},\dots,\vartheta^{(r)}$ with high probability, but we also neglect this detail here and in the sequel for simplicity.}
the matrix $G^{(i)} \in \R^{r \times r}$ given in \eqref{Eq: Matrix of partial derivatives} satisfies
\begin{equation}
\label{Eq: Min singular value bound}
\left\|(G^{(i)})^{-1} \right\|_{\op} \leq r^{\alpha} \, ,
\end{equation}
i.e., the minimum singular value of $G^{(i)}$ decreases to $0$ at most polynomially fast in the dimension $r$.
\item \textbf{Boundedness:} There exists an absolute constant $B > 0$ such that for all $i \in [p]$, we have
\begin{equation}
\label{Eq: Bounded gradient}
\sup_{x \in [0,1]^d} \sup_{\vartheta \in \R^p} \left|g_i(x;\vartheta)\right| \leq B \, . 
\end{equation}
\end{enumerate}

Next, let $\epsilon > 0$ be the desired approximation accuracy of the solution, $\theta^{\star} \in \R^p$ be the unique global minimizer of the objective function $F$ in \eqref{Eq: Standard formulation}, and 
\begin{equation}
\label{Eq: ERM Minimum}
F_* \triangleq \min_{\theta \in \R^p}{F(\theta)} = F(\theta^{\star}) \, .
\end{equation}
(Note that $\theta^{\star}$ exists and $F_*$ is finite due to the strong convexity assumption.) In the sequel, we consider federated optimization algorithms that output an \emph{$\epsilon$-approximate solution} $\theta^* \in \R^p$ to \eqref{Eq: Standard formulation} such that:
\begin{enumerate}
\item If the algorithm generates a deterministic solution $\theta^*$, we have
\begin{equation}
\label{Eq: Approx solution 1}
F(\theta^*) - F_* \leq \epsilon \, .
\end{equation}
\item If the (stochastic) algorithm produces a random solution $\theta^*$, we have
\begin{equation}
\label{Eq: Approx solution 2}
\E\!\left[F(\theta^*)\right] - F_* \leq \epsilon \, ,
\end{equation}
where the expectation is computed with respect to the law of $\theta^*$.
\end{enumerate}

\subsection{Motivation and Background Literature for Assumptions}
\label{Reasoning and Background Literature for Assumptions}

We briefly explain some of the important assumptions in \secref{Assumptions on Loss Function} and contextualize them within the optimization and statistics literatures. 

Firstly, \emph{smoothness in parameter} is a standard assumption in the optimization for machine learning literature (see, e.g.,\cite{Nesterov2004,Bubeck2015}, and the references therein). This assumption implies that the empirical risk $F(\theta)$ has $L_1$-Lipschitz continuous gradient (which is used in the proof of \thmref{Thm: Federated Oracle Complexity of FedLRGD} in \secref{Analysis of Federated Low Rank Gradient Descent}, cf. \cite[Section 4.1]{JadbabaieMakurShah2021}). In particular, when the gradient of the objective function $F(\theta)$ is $L_1$-Lipschitz continuous, it can be shown that each iteration of GD reduces the objective value:
\begin{equation}
\label{Eq: Greedy GD}
F\!\left(\theta - \frac{1}{L_1} \nabla_{\theta} F(\theta)\right)- F(\theta) \leq -\frac{1}{2 L_1} \left\|\nabla_{\theta} F(\theta)\right\|_2^2
\end{equation}
for all $\theta \in \R^p$ \cite[Equation (3.5)]{Bubeck2015}. The inequality \eqref{Eq: Greedy GD} illustrates the greedy nature in which GD decreases its objective value and is fundamental in the analysis of iterative descent methods. For this reason, this assumption has been used extensively in the optimization literature to prove convergence guarantees in a variety of settings, e.g., \cite{Nemirovskietal2009,SchmidtLeRouxBach2011,FriedlanderSchmidt2012,Dekeletal2012,LeRouxSchmidtBach2012,JohnsonZhang2013,DevolderGlineurNesterov2013,SchmidtLeRouxBach2017,BottouCurtisNocedal2018}, etc. We also impose this assumption since the analysis to obtain \lemref{Lemma: Inexact GD Bound} in \secref{Analysis of Federated Low Rank Gradient Descent}, which is a restatement of \cite[Lemma 2.1]{FriedlanderSchmidt2012}, requires it.

Similarly, \emph{strong convexity} (in parameter) is another standard assumption in the optimization for machine learning literature (see, e.g.,\cite{Nesterov2004,Bubeck2015}, and the references therein). As before, this assumption implies that the empirical risk $F(\theta)$ is strongly convex (which is used in the proof of \thmref{Thm: Federated Oracle Complexity of FedLRGD} in \secref{Analysis of Federated Low Rank Gradient Descent}, cf. \cite[Lemma 2.1.4]{Nesterov2004}, \cite[Section 4.1]{JadbabaieMakurShah2021}). As explained in \cite{Bubeck2015}, the main utility of imposing strong convexity of the objective function $F(\theta)$ is a notable and provable speed-up in the convergence rate of gradient-based iterative methods. Specifically, GD can achieve linear convergence rate (i.e., exponentially fast convergence) when the objective function is strongly convex \cite{Nesterov2004,Bubeck2015}. In addition, the strong convexity assumption often leads to far simpler proofs of convergence and oracle complexity, and easier illustration of the benefits of improvements like momentum (or heavy-ball method) \cite{Polyak1964} and acceleration \cite{Nesterov1983} on oracle complexity (cf. \cite{Nesterov2004,BottouCurtisNocedal2018}). For these reasons, this assumption has also been used extensively in the optimization literature to prove convergence guarantees in a variety of settings, e.g., \cite{Nemirovskietal2009,SchmidtLeRouxBach2011,FriedlanderSchmidt2012,LeRouxSchmidtBach2012,JohnsonZhang2013,DevolderGlineurNesterov2013,SchmidtLeRouxBach2017,BottouCurtisNocedal2018}, etc. We also assume strong convexity for these reasons; indeed, strong convexity is needed for \lemref{Lemma: Inexact GD Bound} in \secref{Analysis of Federated Low Rank Gradient Descent}, which is a restatement of \cite[Lemma 2.1]{FriedlanderSchmidt2012}, to hold and the first term in the bound in \lemref{Lemma: Inexact GD Bound} illustrates the aforementioned exponential convergence.

\emph{Smoothness in data} is a standard assumption in the non-parametric estimation literature (see, e.g.,\cite{Tsybakov2009,ChenShah2018,Wasserman2019}, and the references therein). In simple terms, we can think of a function in a H\"{o}lder class with positive integer parameter $\eta$ as an $(\eta-1)$-times differentiable function with Lipschitz continuous $(\eta-1)$th partial derivatives (or even an $\eta$-times differentiable function with bounded $\eta$th partial derivatives). Hence, the size of the parameter $\eta$ controls the degree of smoothness of functions in a H\"{o}lder class. Such H\"{o}lder class assumptions are used in various non-parametric regression techniques, e.g., kernel regression such as Nadaraya-Watson estimators \cite{Nadaraya1964,Watson1964,Tsybakov2009}, local polynomial regression \cite{FanGijbels1996,ClevelandLoader1996,Tsybakov2009}, and nearest neighbor methods \cite{FixHodges1951,CoverHart1967,ChenShah2018}. At a high-level, the basic idea of imposing smoothness assumptions for regression is the following: If an unknown function living in a ``large'' (or non-parametric) class is known to be smoother, then it can be better estimated from samples using Taylor polynomials. When defining the H\"{o}lder condition in \eqref{Eq: Holder condition}, one can in principle use any $\ell^q$-norm (see, e.g., \cite{Xu2018,Wasserman2019,Agarwaletal2021,JadbabaieMakurShah2021}). We use the $\ell^1$-norm to define \eqref{Eq: Holder condition} because this yields the largest class of functions for fixed parameters $(\eta,L_2)$; indeed, if \eqref{Eq: Holder condition} is satisfied for any $\ell^q$-norm, then it is also satisfied for the $\ell^1$-norm using the monotonicity of $\ell^q$-norms. In this sense, our definition includes the largest class of functions. In \appref{App: Logistic Example}, we provide an example of a commonly used loss function in machine learning that satisfies the aforementioned smoothness and strong convexity assumptions. 

In a recent sequence of works in high-dimensional statistics, it has been shown that matrices defined by bivariate graphon functions or latent variable models are approximately low rank when the function satisfies H\"{o}lder class assumptions on one of the variables \cite{Chatterjee2015,Xu2018,Agarwaletal2021}. (We generalize these results in \thmref{Thm: Low Rank Approximation}.) Partial derivatives of loss functions $g_i$ can be construed as bivariate functions of the data $x$ and parameter $\theta$, and we impose H\"{o}lder class smoothness in data on them in order to exploit the resulting approximate low rank structure for federated optimization (as we will elaborate in \secref{Main Results and Discussion}). As noted earlier and in \cite{JadbabaieMakurShah2021}, smoothness in data exists in a variety of ERM problems, e.g., (regularized) linear regression \cite[Chapter 3]{HastieTibshiraniFriedman2009}, (regularized) logistic regression \cite[Section 4.4]{HastieTibshiraniFriedman2009}, deep neural network training with smooth loss and activation functions \cite[Chapter 5]{Bishop2006}, etc. The close connection between smoothness of bivariate functions and approximate low rank structure may seem peculiar at first glance. Perceiving a bivariate function as a kernel of an integral operator, the results in \cite{Chatterjee2015,Xu2018,Agarwaletal2021} suggest that smoothness of the kernel is closely related to approximate low rankness of the associated integral operator. This relation has its roots in more familiar ideas from harmonic analysis. Indeed, in the setting of difference (or convolution) kernels that are defined by a univariate function, increasing the smoothness of the function is essentially equivalent to faster decay of its Fourier transform \cite[Theorem 60.2]{Korner1988} (also see \cite{SteinShakarchi2003Fourier,SteinShakarchi2003Complex}). This, in turn, can be interpreted as approximate low rankness of the associated convolution operator, because the Fourier transform is precisely the spectrum of the convolution operator. 

Approximate low rank structure, broadly construed, has been utilized in a variety of areas, e.g., matrix estimation \cite{CandesPlan2010}, reinforcement learning \cite{Shahetal2020}, function approximation via ridge functions \cite{LoganShepp1975,DonohoJohnstone1989}, semidefinite programming \cite{ChenYangZhu2014}, uncertainty quantification \cite{ConstantineDowWang2014}, etc. In this work, we seek to exploit approximate low rank structure in the context of federated optimization. Just as smoothness is prevalent in many ERM problems, a litany of empirical investigations in the literature have directly demonstrated various forms of approximate low rank structure in large-scale machine learning tasks. For instance, the empirical risk and loss functions corresponding to deep neural network training problems display approximate low rank structure in their gradients (which live in low-dimensional subspaces) or Hessians (which have decaying eigenvalues, thereby inducing low rank gradients via Taylor approximation arguments) \cite{GurAriRobertsDyer2018,Sagunetal2018,Papyan2019,Cuietal2020,Wuetal2021,SinghBachmannHofmann2021}. In particular, \cite{GurAriRobertsDyer2018} shows that when training deep neural networks for $K$-ary classification, the computed gradients belong to $K$-dimensional eigenspaces of the Hessian corresponding to its dominant eigenvalues. On a related front, low rankness has also been observed in weight matrices of deep neural networks \cite{LeJegelka2022,Galantietal2023}, neural collapse has been demonstrated when training deep neural networks \cite{PapyanHanDonoho2020,Fangetal2021,HuiBelkinNakkiran2022}, and training memory for neural networks has been saved using low rank parameterizations of gradient weight matrices \cite{Gooneratneetal2020}. Our work on exploiting approximate low rank structure for federated optimization, albeit different, is inspired by many of these observations.

The non-singularity, conditioning, and boundedness assumptions in \secref{Assumptions on Loss Function} and \secref{Federated Oracle Complexity of FedLRGD} have been imposed for technical purposes for the proofs. Variants of such assumptions pertaining to matrix invertibility and spectra are often made for analytical tractability in theoretical studies in high-dimensional statistics (see, e.g., the invertibility and spectral boundedness assumptions made in \cite[Section 2]{HsuKakadeZhang2012} to analyze linear regression). While it is difficult to provide simple and easily interpretable sufficient conditions that ensure the non-singularity assumptions in \secref{Assumptions on Loss Function} and \secref{Federated Oracle Complexity of FedLRGD}, we provide some intuition for when we expect such assumptions to hold. To this end, consider the non-singularity assumption in \secref{Assumptions on Loss Function}. As before, let us construe $g_i(x,\theta)$ as the kernel of an integral operator. Then, the spectral theory (specifically, singular value decomposition) of compact operators states that the kernel can be written as \cite[Theorem 9]{Shahetal2020} (also see \cite{SteinShakarchi2005}):
\begin{equation}
\label{Eq: SVD of g}
g_i(x,\theta) = \sum_{j = 1}^{\infty}{\sigma_j u_j(x) v_j(\theta)} \, ,
\end{equation}
where $\{\sigma_j \geq 0 : j \in \N\}$ are singular values that converge to $0$ as $j \rightarrow \infty$, $\{u_j: j \in \N\}$ and $\{v_j : j \in \N\}$ are orthonormal bases of singular vector functions, and although the equality above usually holds in an $\L^2$-sense, it can also hold in a pointwise sense under some assumptions. Hence, for data samples $y^{(1)},\dots,y^{(r)}$ and parameter values $\vartheta^{(1)},\dots,\vartheta^{(r)}$, we may write the matrix $G^{(i)}$ as 
\begin{equation}
\label{Eq: Infinite sum vector form of G}
G^{(i)} = \sum_{j = 1}^{\infty} \sigma_j \underbrace{\left[ \begin{array}{c} u_j(y^{(1)}) \\ \vdots \\ u_j(y^{(r)}) \end{array}\right]}_{= \, \mathsf{u}_j} \underbrace{ \left[ v_j(\vartheta^{(1)}) \, \cdots \, v_j(\vartheta^{(r)}) \right]  }_{= \, \mathsf{v}_j^{\T}} \, ,
\end{equation}
using \eqref{Eq: Matrix of partial derivatives} and \eqref{Eq: SVD of g}, where we let $\mathsf{u}_j = \big[ u_j(y^{(1)}) \, \cdots \, u_j(y^{(r)}) \big]^{\T}$ and $\mathsf{v}_j = \big[ v_j(\vartheta^{(1)}) \, \cdots \, v_j(\vartheta^{(r)}) \big]^{\T}$. Since the integral operator corresponding to $g_i$ is a map between infinite-dimensional spaces, we typically have infinitely many non-zero $\sigma_j$'s. Moreover, since $\{u_j: j \in \N\}$ and $\{v_j : j \in \N\}$ are linearly independent sets of functions, we  typically have $r$ non-zero unit-rank terms $\sigma_j \mathsf{u}_j \mathsf{v}_j^{\T}$ in \eqref{Eq: Infinite sum vector form of G}, with linearly independent $\mathsf{u}_j$'s and $\mathsf{v}_j$'s, whose linear independence is not cancelled out by the other terms. Thus, $G^{(i)}$ has rank $r$ and is invertible. This provides some intuition behind the non-singularity assumption in \secref{Assumptions on Loss Function}. The non-singularity of monomials and approximation assumptions in \secref{Federated Oracle Complexity of FedLRGD} have similar underlying intuition based on sampling of linearly independent functions (cf. \eqref{Eq: Phi def} and \eqref{Eq: Rank r form}, \eqref{Eq: cP def}, respectively). In closing this subsection, we remark that one well-known setting where all sampled matrices of a symmetric bivariate function are invertible, and in fact, positive definite, is when the bivariate function itself defines a valid positive definite kernel (in the reproducing kernel Hilbert space sense), cf. \cite[Section 4-2-4.6]{MantonAmblard2014}. In the context of \secref{Assumptions on Loss Function}, we cannot directly assume that every $g_i$ is a positive definite kernel since $g_i$ is asymmetric in general and we only require invertibility of sampled matrices, not positive definiteness, which is a much stronger condition. However, alternative versions of such kernel conditions could be developed in future work to provide cleaner assumptions that imply various non-singularity conditions. 

Finally, we note that, strictly speaking, the boundedness and strong convexity assumptions make sense simultaneously when the parameter space is compact. Although our analysis is carried out with the parameter space $\R^p$, since iterates of inexact GD in \algoref{Alg: FedLRGD} only live in some compact space, our analysis carries over to the compact setting. For example, one way to rigorously carry over our analysis when it is assumed that the parameter space is compact is to run projected inexact GD instead of inexact GD in \algoref{Alg: FedLRGD}. Then, the parameter iterates of \algoref{Alg: FedLRGD} remain in the compact space, and we can execute the analysis of our algorithm using the convergence analysis in, e.g., \cite[Proposition 3]{SchmidtLeRouxBach2011}. For simplicity of exposition, however, we do not use projected inexact GD and simply take the parameter space as $\R^p$ in this paper.

\subsection{Federated Oracle Complexity}
\label{Federated Oracle Complexity}

In the traditional (non-distributed) optimization literature, the running time of first-order iterative methods, such as GD and SGD, is abstractly measured by the notion of (first-order) \emph{oracle complexity}, i.e., the total number of queries to an oracle that returns the gradient of the loss function with respect to the parameter \cite{NemirovskiiYudin1983} (also see \cite[Section 1.1]{Nesterov2004}). This is because gradient computations are usually the most expensive operations in such algorithms. However, the problem of measuring running time is quite different for federated optimization algorithms, because the cost of clients communicating with servers is typically much higher than the running time of gradient computations. Currently, simulations based on several informal principles (such as the aforementioned one) are often used to empirically compare the efficiency of different federated optimization algorithms, rather than theoretically analyzing any formal notion of complexity; see, e.g., \cite{Reisizadehetal2020a} and the references therein. To partially remedy this, inspired by \cite{Reisizadehetal2020a}, we propose a \emph{new} definition of federated oracle complexity by formalizing some of the informal rules delineated in \secref{Federated Learning} using simple probabilistic approximations. We note that other notions of oracle complexity for distributed optimization have either implicitly or explicitly been posed in the literature, cf. graph oracle models \cite{Woodworthetal2018} or \cite{LeeMichelusiScutari2021} and the references therein. In fact, our definition of federated oracle complexity is reminiscent of the complexity measure introduced in \cite{Berahasetal2019}, but our definition is adapted to the centralized, synchronous federated learning setting rather than distributed optimization frameworks.

Consider any algorithm belonging to the family outlined in \secref{Federated Learning}. In the sequel, we will approximate the running times of the two time intensive phases of each epoch: the gradient computation phase and the client-to-server communication phase.

\subsubsection{Gradient Computation Phase}

In the algorithms we consider in this work, there will be two kinds of gradient computation phases; either only active clients will simultaneously perform $b \in \N$ gradient computations each, or only the server will perform $b$ gradient computations. In the former case, each active client computes $b$ gradients of the loss function with respect to the parameter in an epoch. As in the traditional setting, these gradient computations dominate the running time of the client. For each client, its $b$ successive gradient computations are modeled as $b$ successive delayed \emph{Poisson arrivals}, cf. \cite{Leeetal2018}. Equivalently, each gradient computation is assigned a statistically independent (abstract) running time of $1 + X$ units, where $1$ is the delay and $X$ is an exponential random variable with rate $1$ \cite{Leeetal2018}.\footnote{For simplicity, we assume that the shift and scale parameters of the shifted exponential distribution are equal.} Thus, the running time of each client in an epoch is $b + Y$ units, where the random variable $Y$ has an \emph{Erlang} (or \emph{gamma}) \emph{distribution} with shape parameter $b$ and rate $1$ \cite[Section 2.2.2]{Gallager2013}. Since only $\tau m$ clients are active in any epoch, the total computation time of an epoch is determined by the slowest client, and is equal to $b + \max\{Y_1,\dots,Y_{\tau m}\}$ units, where $Y_1,\dots,Y_{\tau m}$ are independent and identically distributed (i.i.d.) Erlang random variables with shape parameter $b$ and rate $1$. 

In most federated learning contexts, $\tau m$ is very large and we can employ \emph{extreme value theory} to approximate the probability distribution of $\max\{Y_1,\dots,Y_{\tau m}\}$. To this end, observe that the cumulative distribution function (CDF) $F_Y : [0,\infty) \rightarrow [0,1]$ of $Y$ is a \emph{von Mises function} with auxiliary function, cf. \cite[Section 3.3.3, Example 3.3.21]{EmbrechtsKluppelbergMikosch1997}:
\begin{equation}
\label{Eq: Auxiliary function}
\forall x > 0, \, a(x) = \sum_{k = 0}^{b-1}{\frac{(b-1)!}{(b-k-1)!} x^{-k}} = 1 + \frac{b-1}{x} + O\!\left(\frac{1}{x^2}\right) ,
\end{equation}
where the second equality holds when $x \gg b$ (as will be the case for us). Hence, employing the \emph{Fisher-Tippett-Gnedenko theorem} \cite[Theorem 3.2.3]{EmbrechtsKluppelbergMikosch1997}, we have convergence in distribution to the \emph{Gumbel distribution} (or the \emph{generalized extreme value distribution} of Type I):
\begin{align}
\forall x \in \R, \enspace \lim_{k \rightarrow \infty}&{\P\!\left(\frac{\max\{Y_1,\dots,Y_{k}\} - F_Y^{-1}(1 - k^{-1})}{a(F_Y^{-1}(1 - k^{-1}))} \leq x \right)} \nonumber \\
& \qquad \qquad \qquad \qquad \qquad = \exp(-e^{-x}) \, ,
\label{Eq: EVT}
\end{align}
where $F_Y^{-1} : [0,1) \rightarrow [0,\infty)$ is the inverse of $F_Y$, i.e., the quantile function, and we also utilize \cite[Proposition 3.3.25]{EmbrechtsKluppelbergMikosch1997}. According to \propref{Prop: Quantile Function of Erlang Distribution} in \appref{Bounds on Erlang Quantile Function}, $\Theta(\log(k)) + \Omega(\log(b)) \leq F_Y^{-1}(1 - k^{-1}) \leq \Theta(\log(k)) + O(b \log(b))$ and we estimate $F_Y^{-1}(1 - k^{-1})$ as follows for simplicity:
\begin{equation}
\label{Eq: Inverse CDF approximation}
F_Y^{-1}(1 - k^{-1}) \approx \log(k) + b \, .
\end{equation}
Thus, using \eqref{Eq: EVT}, the expected total computation time of an epoch can be approximated as
\begin{align}
& b + \E\!\left[\max\{Y_1,\dots,Y_{\tau m}\}\right] \nonumber \\
& \quad \approx \gamma a(F_Y^{-1}(1 - (\tau m)^{-1})) + F_Y^{-1}(1 - (\tau m)^{-1}) + b \nonumber \\
& \quad \approx \gamma + \frac{\gamma (b-1)}{F_Y^{-1}(1 - (\tau m)^{-1})} + F_Y^{-1}(1 - (\tau m)^{-1}) + b \nonumber \\
& \quad \approx \log(\tau m) + 2 b + \underbrace{\gamma + \frac{\gamma b}{\log(\tau m) + b}}_{= \Theta(1)} \nonumber \\
& \quad \approx \log(\tau m) + 2 b \, ,
\label{Eq: Total computation time}
\end{align}
where the \emph{Euler-Mascheroni constant} $\gamma = -\Psi(1) = 0.57721\dots$ (with $\Psi(\cdot)$ denoting the digamma function) is the expected value of the Gumbel distribution with CDF $x \mapsto \exp(-e^{-x})$, the second line follows from \eqref{Eq: Auxiliary function} (neglecting lower order terms), the third line follows from \eqref{Eq: Inverse CDF approximation}, and we neglect the constant in the final line. Since the server does not compute any gradients in this case, any other computation performed by the server or clients is assumed to have negligible cost compared to that in \eqref{Eq: Total computation time}.

On the other hand, in the case where only the server performs $b$ gradients of the loss function with respect to the parameter in an epoch, we may construe the server as a single active client. So, the running time of the server in an epoch is $b + Y$ units, where $Y$ has an Erlang distribution with shape parameter $b$ and rate $1$. Hence, expected total computation time of an epoch is
\begin{equation}
\label{Eq: Total computation time 2}
b + \E\!\left[Y\right] = 2 b \, .
\end{equation}
Combining \eqref{Eq: Total computation time} and \eqref{Eq: Total computation time 2}, the \emph{expected total computation time} per epoch is given by
\begin{equation}
\label{Eq: Expected total computation time}
\begin{aligned}
& \E\!\left[\text{total computation time}\right] \\
& = 
\begin{cases}
2 b + \log(\tau m) \, , & \parbox[]{12.5em}{if active clients simultaneously compute gradients,} \\
2 b \, , & \text{if server computes gradients} .
\end{cases} 
\end{aligned}
\end{equation}

\subsubsection{Client-to-Server Communication Phase}

The total communication cost in any epoch is dominated by the client-to-server communication. Since the server has limited bandwidth and often uses a queue to receive updates from clients, this cost is proportional to the number of active clients $\tau m$. To see this, let us model the client-to-server communication as a first in, first out (FIFO) \emph{M/M/1 queue} \cite{Gallager2013}.\footnote{We utilize \emph{Kendall's notation}, where the first ``M'' implies that arrivals obey a Poisson process, the second ``M'' implies that the service time is memoryless, and the ``1'' denotes that we only have one server.} In such a queue, the updates from active clients arrive at the server according to a \emph{Poisson process} with rate $1$ (for simplicity) \cite[Chapter 2]{Gallager2013}. Moreover, the server only accepts the transmission of any one active client at any instant in time, and from the instant an update arrives at the server, it takes a total transmission time given by an independent exponential distribution with rate $\nu > 1$.\footnote{The assumption that $\nu > 1$ ensures that the M/M/1 queue is ``stable,'' and the number of active clients in the queue does not diverge over time.} Any other active client updates that arrive during this transmission time are put in the FIFO queue for eventual transmission to the server. Finally, let us assume for simplicity that the M/M/1 queue is in steady state.

We seek to calculate the expected time it takes for all active clients to complete transmitting their updates to the server. So, let time $t = 0$ denote the beginning of the (synchronized) client-to-server communication phase, and define the departure process $\{D(t) \in \Z_+: t \geq 0\}$, where $D(t)$ is the number of client updates that have arrived at the server and been entirely transmitted to the server up to time $t \geq 0$ (this excludes any client updates in the queue at time $t$). By \emph{Burke's theorem} \cite[Theorem 7.6.5]{Gallager2013}, $D(t)$ is a Poisson process with rate $1$, i.e., the departure process has the same rate as the arrival process. Hence, the total time it takes for all $\tau m$ active clients to complete transmitting their updates to the server,
\begin{equation}
S \triangleq \inf\!\left\{t \geq 0 : D(t) = \tau m \right\} \in \R_+ \, ,
\end{equation}
has an Erlang distribution with shape parameter $\tau m$ and rate $1$. Therefore, we have 
\begin{equation}
\label{Eq: Erlang mean}
\E[S] = \tau m \, .
\end{equation} 

Next, let $\phi > 0$ be a possibly $p$-dependent hyper-parameter that represents the \emph{communication-to-computation ratio} (cf. \cite{Reisizadehetal2020a}); $\phi$ is the ratio of the average time it takes for a client to communicate a vector in $\R^p$ with the server and the average time it takes a client to compute a single gradient vector in $\R^p$. In most federated learning contexts, $\phi \gg 1$ because the average communication cost of sending a vector to the server far outweighs the average complexity of a gradient computation. Then, using \eqref{Eq: Erlang mean}, the \emph{expected total communication time} per epoch is given by
\begin{equation}
\label{Eq: Expected total communication time}
\E\!\left[\text{total communication time}\right] = \phi \tau m \, ,
\end{equation}
which is in the ``same scale'' as the expected total computation time in \eqref{Eq: Expected total computation time} because we multiply by $\phi$.

\subsubsection{Overall Running Time}

In the algorithms we consider in this work, there will be three kinds of epochs with non-negligible running time:
\begin{enumerate}
\item \underline{Type A:} Active clients simultaneously compute gradients and then communicate with the server. The server does not compute gradients.
\item \underline{Type B:} The server computes gradients. Clients do not compute gradients and do not communicate with the server. 
\item \underline{Type C:} The server and clients do not compute gradients, but active clients communicate with the server.
\end{enumerate}
Combining \eqref{Eq: Expected total computation time} and \eqref{Eq: Expected total communication time}, the \emph{expected total running time} per epoch is
\begin{align}
& \E\!\left[\text{total running time}\right] \nonumber \\
& \qquad \qquad = 
\begin{cases}
2 b + \log(\tau m) + \phi \tau m \, , & \text{if Type A} \\
2 b \, , & \text{if Type B} \\
\phi \tau m \, , & \text{if Type C} 
\end{cases} \label{Eq: Pre expected total running time} \\
& \qquad \qquad =
\begin{cases}
\Theta(b + \phi \tau m) \, , & \text{if Type A} \\
\Theta(b) \, , & \text{if Type B} \\
\Theta(\phi \tau m) \, , & \text{if Type C} 
\end{cases} .
\label{Eq: Expected total running time}
\end{align}

The ensuing definition translates \eqref{Eq: Expected total running time} into a formal notion of federated oracle complexity by summing the expected running times over all epochs. (Note that we neglect the constants $2$ and the sub-dominant logarithmic term in \eqref{Eq: Pre expected total running time}, because we are only concerned with how the complexity asymptotically scales with important problem parameters.)

\begin{definition}[Federated Oracle Complexity]
\label{Def: Federated oracle complexity}
For any federated optimization algorithm $\A$ belonging to the family of algorithms outlined in \secref{Federated Learning}, we define its (first-order) \emph{federated oracle complexity} as 
$$ \Gamma(\A) \triangleq \sum_{t = 1}^{T}{b_t + \I\!\left\{\text{clients communicate in epoch } t\right\} \phi \tau m} \, , $$
where $T \in \N$ is the number of epochs in $\A$, $b_t \in \Z_+$ (which is possibly $0$) is the number of gradients computed in epoch $t \in [T]$ by each active client or the server, $m \in \N$ is the total number of clients, $\tau \in (0,1]$ is the proportion of active clients that participate in each epoch, and $\phi > 0$ denotes the communication-to-computation ratio.
\end{definition}

We next make some pertinent remarks about \defref{Def: Federated oracle complexity}. Firstly, our model of federated oracle complexity neglects the effect of straggling clients on the latency of federated optimization algorithms. This is because the effect of stragglers is captured by a lower order logarithmic term in \eqref{Eq: Expected total computation time} (and \eqref{Eq: Pre expected total running time}). So, we do not study latency issues in this work. Secondly, \defref{Def: Federated oracle complexity} assumes that in any epoch, all clients compute the same number of gradients, because both the theoretical algorithms we analyze in this work, $\lrgd$ and $\FedAve$ \cite{Reisizadehetal2020a}, satisfy this assumption. In general settings, however, different clients may have to compute different numbers of gradients in an epoch. In such cases, one would need to carry out the earlier extreme value theoretic analysis of the gradient computation phase again to determine the expected total computation time of this phase and alter \defref{Def: Federated oracle complexity} appropriately (e.g., $b_t$ could represent the maximum number of gradients computed by a client across all clients in epoch $t$). Thirdly, note that \defref{Def: Federated oracle complexity} reduces to the usual notion of oracle complexity in non-federated settings \cite{NemirovskiiYudin1983,Nesterov2004}. Indeed, since non-federated settings have no communication cost, we can set $\phi = 0$; in this case, our definition simply counts the total number of gradient oracle calls. Finally, for simplicity in this work, we will assume in all cases that the proportion of active clients $\tau = \Theta(1)$ is constant (with respect to any asymptotically growing problem parameters). In the sequel, we evaluate the efficiency of federated optimization algorithms using \defref{Def: Federated oracle complexity}.

\section{Main Results and Discussion}
\label{Main Results and Discussion}

We discuss our main results in this section. Specifically, we outline some key auxiliary results on low rank latent variable model approximation in \secref{Latent Variable Model Approximation}, describe our new Federated Low Rank Gradient Descent algorithm in \secref{Federated Low Rank Gradient Descent}, present the federated oracle complexity (as defined in \defref{Def: Federated oracle complexity}) of this algorithm in \secref{Federated Oracle Complexity of FedLRGD}, and finally, demonstrate that our algorithm has better federated oracle complexity than the vanilla $\FedAve$ algorithm in \secref{Comparison to Federated Averaging}.

\subsection{Latent Variable Model Approximation}
\label{Latent Variable Model Approximation}

To establish the federated oracle complexity of our algorithm, we will require a basic result about uniformly approximating a smooth function using piecewise polynomials. To illustrate this result, consider any function $g:[0,1]^d \times \R^p \rightarrow \R$, $g(x;\theta)$ such that for all $\vartheta \in \R^p$, the function $g(\cdot;\vartheta):[0,1]^d \rightarrow \R$ belongs to the $(\eta,L_2)$-H\"{o}lder class as defined in \eqref{Eq: Holder condition} in \secref{Assumptions on Loss Function}. Moreover, for any $q > 0$, consider a minimal $\ell^1$-covering of the hypercube $[0,1]^d$ using closed $\ell^1$-balls (or cross-polytopes) with radius $q^{-1}$ and centers $\net = \{z^{(1)},\dots,z^{(|\net|)}\} \subseteq [0,1]^d$, which forms the associated $q^{-1}$-net, such that:
\begin{equation}
\label{Eq: covering}
\forall x \in [0,1]^d, \, \exists z \in \net, \enspace \left\|x - z\right\|_1 \leq \frac{1}{q} \, .
\end{equation}
In particular, \eqref{Eq: covering} shows that the union of the closed $\ell^1$-balls with radius $q^{-1}$ and centers $\net$ contains the set $[0,1]^d$, and the $q^{-1}$-net $\net$ is minimal in the sense that its cardinality $|\net|$ is as small as possible while ensuring \eqref{Eq: covering}, i.e., $|\net|$ is the $\ell^1$-\emph{covering number} of $[0,1]^d$ (cf. \cite[Definition 4.2.2]{Vershynin2018}). Let $B_1,\dots,B_{|\net|} \subseteq \R^d$ be an enumeration of the closed $\ell^1$-balls with radius $q^{-1}$ and centers $z^{(1)},\dots,z^{(|\net|)} \in \net$, respectively, and construct the sets $I_1,\dots,I_{|\net|} \subseteq [0,1]^d$ such that $I_1 = B_1 \cap [0,1]^d$ and $I_i = (B_i\backslash(B_1 \cup \cdots \cup B_{i-1})) \cap [0,1]^d$ for $i \in \{2,\dots,|\net|\}$. Then, 
\begin{equation}
\label{Eq: partition}
\S_{d,q} \triangleq \left\{I_1,\dots,I_{|\net|}\right\}
\end{equation}
forms a partition of the hypercube $[0,1]^d$ such that every point in $I_j$ is at an $\ell^1$-distance of at most $q^{-1}$ from $z^{(j)}$. With this $\ell^1$-ball based partition in place, the next lemma conveys that $g$ can be uniformly approximated using a piecewise polynomial function akin to the proofs of \cite[Proposition 1]{Xu2017} and \cite[Proposition 3.3]{Agarwaletal2021}.

\begin{lemma}[Uniform Piecewise Polynomial Approximation]
\label{Lemma: Uniform Piecewise Polynomial Approximation}
For any $q > 0$ and any $\vartheta \in \R^p$, there exists a piecewise polynomial function $P_{q,l}(\cdot;\vartheta) : [0,1]^d \rightarrow \R$ such that
$$ \sup_{y \in [0,1]^d}{\left|g(y;\vartheta) - P_{q,l}(y;\vartheta)\right|} \leq  \frac{L_2}{l! q^{\eta}} \, , $$
where $l = \lceil \eta \rceil - 1$, the piecewise polynomial function $P_{q,l}(\cdot;\vartheta) : [0,1]^d \rightarrow \R$ is given by the following summation:
$$ \forall y \in [0,1]^d, \enspace P_{q,l}(y;\vartheta) = \sum_{I \in \S_{d,q}}{P_I(y;\vartheta) \I\!\{y \in I\}} \, , $$
and $P_I(\cdot;\vartheta) : [0,1]^d \rightarrow \R$ are Taylor polynomials with (total) degree at most $l$ as shown in \eqref{Eq: Taylor polynomials}.
\end{lemma}

\lemref{Lemma: Uniform Piecewise Polynomial Approximation} is proved in \secref{Proof of Lemma Uniform Piecewise Polynomial Approximation}. While this lemma will suffice to establish federated oracle complexity in our problem, we also present a corollary of it in the context of latent variable models for completeness. 

For any $n,k \in \N$, any $y^{(1)},\dots,y^{(n)} \in [0,1]^d$, and any $\vartheta^{(1)},\dots,\vartheta^{(k)} \in \R^p$, define the matrix $M \in \R^{n \times k}$ such that:
\begin{equation}
\label{Eq: LVM}
\forall i \in [n], \, \forall j \in [k], \enspace M_{i,j} \triangleq g(y^{(i)},\vartheta^{(j)}) \, .
\end{equation}
We refer to $M$ as a \emph{latent variable model} with latent parameters $y^{(1)},\dots,y^{(n)} \in [0,1]^d$ and $\vartheta^{(1)},\dots,\vartheta^{(k)} \in \R^p$, and latent function $g$ (which satisfies the H\"{o}lder smoothness assumptions outlined earlier); see, e.g., \cite{Agarwaletal2021} and the references therein. Furthermore, in the special case where $g$ is symmetric in its two inputs and $M$ is a symmetric matrix, $g$ is known as a \emph{graphon function} (which naturally appears in Aldous-Hoover representations of exchangeable random graphs \cite{Aldous1981}), cf. \cite{Chatterjee2015,ZhangLevinaZhu2017,Xu2018}. Next, let
\begin{equation}
M = \sum_{i = 1}^{\min\{n,k\}}{\sigma_{i} u_i v_i^{\T}}
\end{equation}
be the singular value decomposition of $M$ with singular values $\sigma_1 \geq \sigma_2 \geq \cdots \geq \sigma_{\min\{n,k\}} \geq 0$ and orthonormal bases of singular vectors $\{u_1,\dots,u_n\} \subseteq \R^n$ and $\{v_1,\dots,v_k\} \subseteq \R^k$. Then, by the Eckart-Young-Mirsky theorem, for any $r \in [\min\{n,k\}]$, the best rank $r$ approximation of $M$ in Frobenius norm is given by 
\begin{equation}
M_r \triangleq \sum_{i = 1}^{r}{\sigma_{i} u_i v_i^{\T}} \, ,
\end{equation}
i.e., $M_r = \argmin_{A \in \R^{n \times k}: \, \rank(A) \leq r}{\|M - A\|_{\Fro}}$, cf. \cite[Section 7.4.2]{HornJohnson2013}. Akin to \cite[Proposition 1]{Xu2018} and \cite[Proposition 3.3]{Agarwaletal2021}, the ensuing theorem demonstrates that the latent variable model $M$ can be well-estimated by its low rank approximation due to the smoothness of $g$. We believe that this result may be of broader interest to the matrix and tensor estimation community.

\begin{theorem}[Low Rank Approximation]
\label{Thm: Low Rank Approximation}
Consider the latent variable model $M \in \R^{n \times k}$ in \eqref{Eq: LVM} defined by the latent function $g:[0,1]^d \times \R^p \rightarrow \R$ such that $g(\cdot;\vartheta):[0,1]^d \rightarrow \R$ belongs to the $(\eta,L_2)$-H\"{o}lder class for all $\vartheta \in \R^p$. Then, for any $r \geq e \sqrt{d} \, 2^d (l + d)^d$, we have
$$ \frac{1}{n k} \left\|M - M_r \right\|_{\Fro}^2 \leq \frac{L_2^2 e^{2\eta/d} d^{\eta/d} 4^{\eta} (l + d)^{2 \eta}}{(l!)^2} \left( \frac{1}{r} \right)^{\! 2\eta/d} , $$
where $l = \lceil \eta \rceil - 1$.
\end{theorem}

\thmref{Thm: Low Rank Approximation} is proved in \secref{Proof of Thm Low Rank Approximation}; we will use ideas from this proof in our analysis of federated oracle complexity. It is worth comparing this result with the related results in \cite{Chatterjee2015}, \cite[Proposition 1]{Xu2017}, \cite[Proposition 1]{Xu2018}, and \cite[Proposition 3.3]{Agarwaletal2021}. The author of \cite{Chatterjee2015} proves an analog of \thmref{Thm: Low Rank Approximation} in the special case where $g$ is a graphon function and $g(\cdot;\vartheta)$ is Lipschitz continuous with respect to the $\ell^{\infty}$-norm for all $\vartheta \in \R^p$. In this special case, a piecewise constant function approximation of $g$ suffices for the proof. The results in \cite[Proposition 1]{Xu2017} and \cite[Proposition 1]{Xu2018} extend the result of \cite{Chatterjee2015} to the setting where $g$ is still a graphon function, but $g(\cdot;\vartheta)$ belongs to an $(\eta,L_2)$-H\"{o}lder class with respect to the $\ell^{\infty}$-norm. The key insight in \cite{Xu2017,Xu2018} is to utilize piecewise polynomial function approximations instead of piecewise constant functions; we also use this key idea in \lemref{Lemma: Uniform Piecewise Polynomial Approximation}. Finally, the authors of \cite[Proposition 3.3]{Agarwaletal2021} apply the argument of \cite{Xu2017,Xu2018} to extend their result further to the setting where $g$ is a general latent variable model, similar to our result, but $g(\cdot;\vartheta)$ only belongs to an $(\eta,L_2)$-H\"{o}lder class with respect to the $\ell^{\infty}$-norm. In contrast, our result in \thmref{Thm: Low Rank Approximation} applies to general latent variable models $g$ such that $g(\cdot;\vartheta)$ belongs to an $(\eta,L_2)$-H\"{o}lder class \emph{with respect to the $\ell^{1}$-norm}, which \emph{subsumes} the settings of \cite{Chatterjee2015,Xu2017,Xu2018,Agarwaletal2021}. To prove this \emph{more general result}, we resort to using a modified \emph{minimal $\ell^1$-covering} to construct our piecewise polynomial function approximation in the intermediate result in \lemref{Lemma: Uniform Piecewise Polynomial Approximation}; this differs from the $\ell^{\infty}$-coverings used in \cite{Chatterjee2015}, \cite[Lemma 3]{Xu2017}, \cite{Xu2018}, and \cite{Agarwaletal2021}, and it can be verified that $\ell^{\infty}$-coverings produce a worse bound in \lemref{Lemma: Uniform Piecewise Polynomial Approximation}.\footnote{We note that our modified $\ell^1$-covering of $[0,1]^d$ is more sophisticated than the $\ell^{\infty}$-coverings used in \cite{Chatterjee2015,Xu2017,Xu2018,Agarwaletal2021}, because $\ell^{\infty}$-balls can be used to form a (geometric) \emph{honeycomb} in the hypercube $[0,1]^d$, but $\ell^{1}$-balls do not form such a honeycomb in $[0,1]^d$ for general $d \in \N$.} Therefore, in contrast to \cite{Chatterjee2015,Xu2017,Xu2018,Agarwaletal2021}, our proofs of \lemref{Lemma: Uniform Piecewise Polynomial Approximation} and \thmref{Thm: Low Rank Approximation} are adapted to deal with the $\ell^{1}$-norm H\"{o}lder class assumption on $g(\cdot;\vartheta)$ and the modified minimal $\ell^1$-covering of the hypercube $[0,1]^d$.

\subsection{Federated Low Rank Gradient Descent}
\label{Federated Low Rank Gradient Descent}

We now describe the new \emph{Federated Low Rank Gradient Descent} ($\lrgd$) algorithm for federated optimization under the setup of \secref{Federated Learning} and \secref{Assumptions on Loss Function}. Recall that we are given a central server with $r$ private training data samples $x^{(0,1)},\dots,x^{(0,r)} \in [0,1]^d$, and $m$ clients with $s$ private training data samples each; client $c \in [m]$ has samples $x^{(c,1)},\dots,x^{(c,s)} \in [0,1]^d$. The number of samples at the server, $r$, depends on the approximate rank of the loss function, as we will see in the next subsection. The server's objective is to compute an $\epsilon$-approximate solution in the sense of \eqref{Eq: Approx solution 1} to the ERM problem \eqref{Eq: Standard formulation 3}.

In the traditional non-federated setting, this objective could be achieved by running GD (or any of its many variants). Each iteration of GD would involve calculating the gradient:
\begin{equation}
\label{Eq: Full gradient}
\begin{aligned}
\forall \theta \in \R^p, \enspace \nabla_{\theta} F(\theta) & = \frac{1}{n} \sum_{k = 1}^{r}{\nabla_{\theta} f(x^{(0,k)};\theta)} \\
& \quad \, + \frac{1}{n} \sum_{c = 1}^{m}{\sum_{j = 1}^{s}{\nabla_{\theta} f(x^{(c,j)};\theta)}} \, ,
\end{aligned}
\end{equation}
at different parameter values. The main idea of this work is to utilize the smoothness of $\nabla_{\theta} f$ in the data to approximate the gradient in \eqref{Eq: Full gradient} by exploiting the approximate low rank structure induced by the smoothness (cf. \secref{Latent Variable Model Approximation}). Specifically, for every index $i \in [p]$, we determine a set of weights $\{w_{j,k}^{(i,c)} \in \R : c \in [m], \, j \in [s], \, k \in [r]\}$ such that:
\begin{equation}
\label{Eq: Key approx}
\forall \theta \in \R^p, \, \frac{\partial F}{\partial \theta_i} (\theta) \approx \frac{1}{n} \sum_{k = 1}^{r}{g_i (x^{(0,k)};\theta) \left( 1 + \sum_{c = 1}^{m}{\sum_{j = 1}^{s}{w_{j,k}^{(i,c)}}} \right)} ,
\end{equation}
where for each client $c \in [m]$ and each private training sample $x^{(c,j)} \in [0,1]^d$, we re-weigh the partial derivative $g_i (x^{(0,k)};\cdot)$ at the server's training sample $x^{(0,k)} \in [0,1]^d$ by an additional amount $w_{j,k}^{(i,c)}$. In our $\lrgd$ algorithm, the clients compute the aforementioned weights after receiving some initial information from the server, and the server then uses these weights to run inexact GD. Intuitively, when $\nabla_{\theta} f$ is very smooth with respect to the data, the number of private samples $r$ required at the server to ensure that \eqref{Eq: Key approx} is a good approximation is reasonably small. Hence, the oracle complexity of GD at the server also remains reasonably small since it only computes gradients at $r$ data samples in every iteration (instead of $n$).

Our $\lrgd$ algorithm runs over $T = r+2$ synchronized epochs and computes an $\epsilon$-approximate solution to \eqref{Eq: Standard formulation 3} at the server. In the first epoch, all clients are idle and there is no communication between the server and clients. The server arbitrarily chooses $r$ parameter values $\vartheta^{(1)}, \dots,\vartheta^{(r)} \in \R^p$, and performs $r^2$ gradient computations corresponding to these $r$ parameter values and its $r$ private training samples $x^{(0,1)},\dots,x^{(0,r)}$. Since each gradient computation produces a $p$-dimensional vector, for every index $i \in [p]$, the server obtains a matrix of partial derivatives $G^{(i)} \in \R^{r \times r}$ as in \eqref{Eq: Matrix of partial derivatives}:
\begin{equation}
\label{Eq: Matrix of partial derivatives 2}
\forall j,k \in [r], \enspace G^{(i)}_{j,k} = g_i(x^{(0,j)};\vartheta^{(k)}) = \frac{\partial f}{\partial \theta_i}(x^{(0,j)};\vartheta^{(k)}) \, .
\end{equation}
Moreover, due to the non-singularity assumption in \secref{Assumptions on Loss Function}, each $G^{(i)}$ is an invertible matrix. So, the server computes the inverses of these $p$ matrices to get $(G^{(1)})^{-1},\dots,(G^{(p)})^{-1} \in \R^{r \times r}$.

The second epoch proceeds in three phases. First, the server broadcasts the $r$ parameter values $\vartheta^{(1)},\dots,\vartheta^{(r)}$ and the $p$ matrices $(G^{(1)})^{-1},\dots,(G^{(p)})^{-1}$ to all $m$ clients. We assume for simplicity that $\tau = 1$, i.e., all clients are active, although our analysis would hold in the case $\tau = \Theta(1)$. In the second phase, every client performs $rs$ gradient computations corresponding to its $s$ private training samples and the $r$ common parameter values. In particular, each client $c \in [m]$ computes the following set of partial derivatives $\{g_i(x^{(c,j)};\vartheta^{(k)}) : i \in [p], \, j \in [s] , \, k \in [r]\}$. Then, each client $c$ calculates the weight vectors $v^{(i,c)} \in \R^{r}$:
\begin{equation}
\label{Eq: Cumulative weights}
\begin{aligned}
& (v^{(i,c)})^{\T} \\
& = \left(\!\sum_{j = 1}^{s}{\underbrace{\left[g_i(x^{(c,j)};\vartheta^{(1)}) \enspace g_i(x^{(c,j)};\vartheta^{(2)}) \, \cdots \, g_i(x^{(c,j)};\vartheta^{(r)})\right]}_{\text{row vector of } r \text{ partial derivatives at } x^{(c,j)}}}\!\right) \\
& \quad \enspace \cdot (G^{(i)})^{-1}
\end{aligned}
\end{equation}
for every index $i \in [p]$. Note that the $k$th weight $v^{(i,c)}_k \in \R$ for the $i$th partial derivative at client $c$ represents the accumulation of the aforementioned weights $\{w_{j,k}^{(i,c)} \in \R : j \in [s]\}$ (as explained in \eqref{Eq: v and w}). 

To intuitively understand \eqref{Eq: Cumulative weights}, for each index $i \in [p]$, note that the induced approximate low rankness of $g_i(x;\theta)$ implies that for every client $c \in [m]$ and data sample $x^{(c,j)}$:
\begin{equation}
\label{Eq: Intuition for approx low rank}
\forall \theta \in \R^p, \enspace g_i(x^{(c,j)};\theta) \approx \sum_{k = 1}^{r}{w_{j,k}^{(i,c)} g_i(x^{(0,k)};\theta)} \, ,
\end{equation}
for some choices of weights $\{w_{j,k}^{(i,c)} \in \R : k \in [r]\}$ (which we formally define soon). Indeed, if we construe each $g_i(x;\theta)$ as an infinite ``matrix'' with $x$ indexing its ``rows'' and $\theta$ indexing its ``columns,'' then if the approximate rank of $g_i$ is $r$, the ``rows'' of $g_i$ corresponding to data samples at the clients can we approximately represented as linear combinations of the $r$ ``rows'' of $g_i$ corresponding to the data samples at the server $x^{(0,1)},\dots,x^{(0,r)}$. Furthermore, this property in \eqref{Eq: Intuition for approx low rank} implies the approximation in \eqref{Eq: Key approx}. Formally, we capture the approximate low rankness in \eqref{Eq: Intuition for approx low rank} by sampling $g_i(x^{(c,j)};\theta)$ and each $g_i(x^{(0,k)};\theta)$ at the parameter values $\vartheta^{(1)},\dots,\vartheta^{(r)}$:
\begin{equation}
\begin{aligned}
& \left[g_i(x^{(c,j)};\vartheta^{(1)}) \enspace g_i(x^{(c,j)};\vartheta^{(2)}) \, \cdots \, g_i(x^{(c,j)};\vartheta^{(r)})\right] \\
& \qquad \qquad = \left[w_{j,1}^{(i,c)} \enspace w_{j,2}^{(i,c)} \, \cdots \, w_{j,r}^{(i,c)}  \right] G^{(i)} ,
\end{aligned}
\end{equation}
which, after matrix inversion, equivalently defines the weights $\{w_{j,k}^{(i,c)} \in \R : k \in [r]\}$ via the relation
\begin{equation}
\label{Eq: Non-cumulative weights}
\begin{aligned}
& \left[w_{j,1}^{(i,c)} \enspace w_{j,2}^{(i,c)} \, \cdots \, w_{j,r}^{(i,c)}  \right] = \\
& \left[g_i(x^{(c,j)};\vartheta^{(1)}) \enspace g_i(x^{(c,j)};\vartheta^{(2)}) \, \cdots \, g_i(x^{(c,j)};\vartheta^{(r)})\right] (G^{(i)})^{-1} .
\end{aligned}
\end{equation}
Additionally, \eqref{Eq: Key approx} conveys that we only require the sums of weights, $\sum_{j = 1}^{s}{w_{j,k}^{(i,c)}}$, for all $i \in [p]$ and $k \in [r]$ from each client $c$ to estimate the gradient of the empirical risk. So, it suffices for client $c$ to directly compute the accumulated weights:
\begin{equation}
\label{Eq: v and w}
\forall i \in [p], \, \forall k \in [r], \enspace v^{(i,c)}_k = \sum_{j = 1}^{s}{w_{j,k}^{(i,c)}} \, .
\end{equation}
Combining \eqref{Eq: v and w} with \eqref{Eq: Non-cumulative weights} produces \eqref{Eq: Cumulative weights}. This explains why each client $c$ directly calculates \eqref{Eq: Cumulative weights}. In fact, the accumulated weights in \eqref{Eq: Cumulative weights} simplify \eqref{Eq: Key approx} and yield the approximation:
\begin{equation}
\label{Eq: Key approx 2}
\forall \theta \in \R^p, \enspace \frac{\partial F}{\partial \theta_i} (\theta) \approx \frac{1}{n} \sum_{k = 1}^{r}{g_i (x^{(0,k)};\theta) \left( 1 + \sum_{c = 1}^{m}{v_{k}^{(i,c)}} \right)} 
\end{equation}
for all $i \in [p]$. 

After calculating \eqref{Eq: Cumulative weights} for all $i \in [p]$, each client $c$ possesses a set of $rp$ weights $\{v^{(i,c)}_k : k \in [r], \,  i \in [p]\}$. Hence, in the third phase of the second epoch, every client $c$ communicates the first $p$ weights $\{v^{(i,c)}_1 : i \in [p]\}$ from its set of $rp$ weights to the server. In the next $r-1$ epochs, every client $c \in [m]$ communicates $p$ weights per epoch from its set of $rp$ weights to the server. Specifically, in epoch $t \in \{3,\dots,r+1\}$, client $c$ communicates the weights $\{v^{(i,c)}_{t-1} : i \in [p]\}$ to the server. There are no gradient computations performed in these epochs.

Finally, in the last epoch, the server runs \emph{inexact gradient descent} using the weights $\{v^{(i,c)}_k : k \in [r], \,  i \in [p], \, c \in [m]\}$ it has received from the clients. There is no communication between the server and clients, and all clients remain idle at this stage. To precisely describe the inexact GD iterations, for any $\theta \in \R^p$, define the approximation $\widehat{\nabla F}(\theta) = \big[\widehat{\nabla F}_1(\theta) \, \cdots\, \widehat{\nabla F}_p(\theta)\big]^{\T} \in \R^p$ of $\nabla_{\theta}F(\theta)$ such that:
\begin{align}
\forall i \in [p], \enspace \widehat{\nabla F}_i(\theta) & \triangleq \frac{1}{n} \sum_{k = 1}^{r}{g_i (x^{(0,k)};\theta) \left( 1 + \sum_{c = 1}^{m}{v_{k}^{(i,c)}} \right)} \nonumber \\
& = \frac{1}{n} \sum_{k = 1}^{r}{\frac{\partial f}{\partial \theta_i} (x^{(0,k)};\theta) \left( 1 + \sum_{c = 1}^{m}{v_{k}^{(i,c)}} \right)} ,
\label{Eq: Inexact Gradient}
\end{align}
as indicated in \eqref{Eq: Key approx 2}. Moreover, let $S \in \N$ be the number of iterations for which the server runs inexact GD; $S$ depends on the desired approximation accuracy $\epsilon > 0$ in \secref{Assumptions on Loss Function} and other problem parameters, as we will see in the next subsection. Then, starting with an arbitrary parameter vector $\theta^{(0)} \in \R^p$, the server performs the inexact GD updates
\begin{equation}
\label{Eq: Inexact GD update}
\theta^{(\gamma)} = \theta^{(\gamma-1)} - \frac{1}{L_1} \, \widehat{\nabla F}(\theta^{(\gamma-1)}) 
\end{equation}
for $\gamma = 1,\dots,S$, where $L_1 > 0$ is the common Lipschitz constant of the gradient $\nabla_{\theta} f(x;\cdot)$ for all $x \in [0,1]^d$ (see \secref{Assumptions on Loss Function}). Note that at the $\gamma$th update, the server performs $r$ gradient computations corresponding to its private data $\{x^{(0,j)} \in [0,1]^d : j \in [r]\}$ at the parameter value $\theta^{(\gamma-1)}$. Lastly, the output of the $\lrgd$ algorithm at the server is $\theta^* = \theta^{(S)} \in \R^p$. 

In the next subsection, we analyze the federated oracle complexity of this algorithm and specify what values $r$ and $S$ should take to ensure that $\theta^{(S)} \in \R^p$ is an $\epsilon$-approximate solution in the sense of \eqref{Eq: Approx solution 1} to the ERM problem \eqref{Eq: Standard formulation 3}. We also summarize the $\lrgd$ algorithm in \algoref{Alg: FedLRGD} for the readers' convenience.

\begin{algorithm}[t]
{\footnotesize
\begin{algorithmic}[1]
\renewcommand{\algorithmicrequire}{\textbf{Input:}}
\Require training data $\{x^{(0,j)} \in [0,1]^d : j \in [r]\} \cup \{x^{(c,j)} \in [0,1]^d : c \in [m], \, j \in [s]\}$
\Require approximate rank $r \in \N$ 
\Require number of inexact GD iterations $S \in \N$
\renewcommand{\algorithmicensure}{\textbf{Output:}}
\Ensure approximate solution $\theta^* \in \R^p$ to \eqref{Eq: Standard formulation 3}
\Statex \textbf{\emph{Epoch $t = 1$: Server pre-computation}}
\State Server chooses $r$ arbitrary parameter vectors $\vartheta^{(1)},\dots,\vartheta^{(r)} \in \R^p$
\State Server performs $r^2$ gradient computations corresponding to $\vartheta^{(1)},\dots,\allowbreak \vartheta^{(r)}$ and its private data $\{x^{(0,j)} \in [0,1]^d : j \in [r]\}$ to obtain the set of partial derivatives $\{g_i(x^{(0,j)};\vartheta^{(k)}) : i \in[p], \, j \in [r], \, k \in [r]\}$ 
\State Server constructs the matrices $G^{(1)},\dots,G^{(p)} \in \R^{r \times r}$ according to \eqref{Eq: Matrix of partial derivatives 2} using its partial derivatives
\State Server calculates the inverse matrices $(G^{(1)})^{-1},\dots,(G^{(p)})^{-1} \in \R^{r \times r}$
\Statex \textbf{\emph{Epoch $t = 2$: Client weights computation}}
\State Server broadcasts $\vartheta^{(1)},\dots,\vartheta^{(r)}$ to all $m$ clients
\State Server broadcasts $(G^{(1)})^{-1},\dots,(G^{(p)})^{-1}$ to all $m$ clients
\State Every client $c \in [m]$ performs $r s$ gradient computations corresponding to $\vartheta^{(1)},\dots,\vartheta^{(r)}$ and its private data $\{x^{(c,j)} \in [0,1]^d : j \in [s]\}$ to obtain the set of partial derivatives $\{g_i(x^{(c,j)};\vartheta^{(k)}) : i \in[p], \, j \in [s], \, k \in [r]\}$ 
\State Every client $c \in [m]$ calculates $p$ weight vectors $v^{(1,c)},\dots,v^{(p,c)} \in \allowbreak \R^r$ according to \eqref{Eq: Cumulative weights} using its partial derivatives and $(G^{(1)})^{-1},\dots,\allowbreak (G^{(p)})^{-1}$
\State Every client $c \in [m]$ communicates the $p$ weights $\{v_1^{(i,c)} : i \in [p]\}$ to the server
\Statex \textbf{for \emph{Epochs $t = 3$ to $r+1$: Client-to-server communication}}
\State Every client $c \in [m]$ communicates the $p$ weights $\{v_{t-1}^{(i,c)} : i \in [p]\}$ to the server
\Statex \textbf{end for}
\Statex \textbf{\emph{Epoch $t = r+2$: Server inexact GD}}
\State Server initializes an arbitrary parameter vector $\theta^{(0)} \in \R^p$
\For{$\gamma = 1$ to $S$} \Comment{Iterations of inexact GD} 
\State \parbox[t]{\dimexpr\linewidth-\algorithmicindent-0.5mm}{Server performs $r$ gradient computations corresponding to its private data $\{x^{(0,j)} \in [0,1]^d : j \in [r]\}$ at the parameter value $\theta^{(\gamma-1)}$ to obtain the set of partial derivatives $\{g_i(x^{(0,j)};\theta^{(\gamma-1)}) : i \in[p], \, j \in [r]\}$\strut}
\State \parbox[t]{\dimexpr\linewidth-\algorithmicindent-0.5mm}{Server constructs approximation of gradient $\widehat{\nabla F}(\theta^{(\gamma-1)})$ according to \eqref{Eq: Inexact Gradient} using $\{g_i(x^{(0,j)};\theta^{(\gamma-1)}) : i \in[p], \, j \in [r]\}$ and the weights $\{v^{(i,c)}_k : k \in [r], \,  i \in [p], \, c \in [m]\}$\strut}
\State \parbox[t]{\dimexpr\linewidth-\algorithmicindent-0.5mm}{Server updates the parameter vector according to \eqref{Eq: Inexact GD update}: $\theta^{(\gamma)} = \theta^{(\gamma-1)} - \frac{1}{L_1} \widehat{\nabla F}(\theta^{(\gamma-1)})$\strut}
\EndFor
\State \Return $\theta^* = \theta^{(S)}$
\end{algorithmic}
}
\caption{Federated Low Rank Gradient Descent ($\lrgd$) algorithm to approximately solve the ERM problem \eqref{Eq: Standard formulation 3}.}
\label{Alg: FedLRGD}
\end{algorithm}

\subsection{Federated Oracle Complexity of $\lrgd$}
\label{Federated Oracle Complexity of FedLRGD}

In addition to the assumptions on the loss function in \secref{Assumptions on Loss Function}, we will require several assumptions on the piecewise polynomial approximations introduced in \secref{Latent Variable Model Approximation}. Recall from \secref{Assumptions on Loss Function} that the $i$th partial derivative $g_i : [0,1]^d \times \R^p \rightarrow \R$ satisfies the $(\eta,L_2)$-H\"{o}lder class condition. So, for any $q > 0$ and any fixed index $i \in [p]$, let $P_{q,l} : [0,1]^d \times \R^p \rightarrow \R$ be the piecewise polynomial approximation of $g_i : [0,1]^d \times \R^p \rightarrow \R$ satisfying \lemref{Lemma: Uniform Piecewise Polynomial Approximation}:
\begin{equation}
\forall x \in [0,1]^d, \, \forall \theta \in \R^p, \enspace P_{q,l}(x;\theta) = \sum_{I \in \S_{d,q}}{P_I(x;\theta) \I\!\{x \in I\}} \, , 
\end{equation}
where $\S_{d,q}$ is the partition of the hypercube $[0,1]^d$ given in \eqref{Eq: partition}, $P_I(\cdot;\theta) : [0,1]^d \rightarrow \R$ are Taylor polynomials with degree at most $l = \lceil \eta \rceil - 1$ as defined in \eqref{Eq: Taylor polynomials}, and we suppress the dependence of $P_{q,l}$ on $i$ for simplicity. Then, for any sufficiently large $r \in \N$, fix $q > 0$ such that $r = \binom{l + d}{d} |\S_{d,q}|$. As shown in the proofs of \thmref{Thm: Low Rank Approximation} and \thmref{Thm: Federated Oracle Complexity of FedLRGD} in \secref{Proof of Thm Low Rank Approximation} and \secref{Analysis of Federated Low Rank Gradient Descent}, respectively, we can write $P_{q,l}$ in the form (cf. \eqref{Eq: low rank form}):
\begin{equation}
\label{Eq: Rank r form}
\forall x \in [0,1]^d, \, \forall \theta \in \R^p, \enspace P_{q,l}(x;\theta) = \sum_{w = 1}^{r}{\phi_w(x) \psi_w(\theta)} \, , 
\end{equation}
where $\{\phi_w : w \in [r]\}$ is a family of linearly independent monomial functions on $[0,1]^d$ (see \eqref{Eq: monomial}), and $\{\psi_w : w \in [r]\}$ is another family of functions on $\R^p$. Note that the form \eqref{Eq: Rank r form} conveys that $P_{q,l}$, when perceived as the bivariate kernel of an integral operator, has rank at most $r$. For every large enough $r \in \N$ and any fixed index $i \in [p]$, we make the following assumptions about the approximation $P_{q,l}$:
\begin{enumerate}
\item \textbf{Non-singularity of monomials:} The matrix $\Phi \in \R^{r\times r}$ with entries:
\begin{equation}
\label{Eq: Phi def}
\forall w,j \in [r] , \enspace \Phi_{w,j} = \phi_w(x^{(0,j)}) \, ,
\end{equation}
is invertible. Note that we suppress the dependence of $\Phi$ on $i$, $r$, and the server's training data for simplicity.
\item \textbf{Non-singularity of approximation:} For any sequence of distinct parameter values $\vartheta^{(1)},\dots,\vartheta^{(r)} \in \R^p$, the matrix $\cP^{(i)} \in \R^{r \times r}$ given by:
\begin{equation}
\label{Eq: cP def}
\forall j,k \in [r], \enspace \cP^{(i)}_{j,k} = P_{q,l}(x^{(0,j)};\vartheta^{(k)}) \, ,
\end{equation}
is invertible. Here, we suppress the dependence of $\cP^{(i)}$ on $r$, the server's training data, and the parameter sequence for simplicity.
\end{enumerate}

With all our assumptions in place, the ensuing theorem presents the federated oracle complexity of the $\lrgd$ algorithm to generate an $\epsilon$-approximate solution. This is the key technical result of this work.

\begin{theorem}[Federated Oracle Complexity of $\lrgd$]
\label{Thm: Federated Oracle Complexity of FedLRGD}
Suppose that the two assumptions above and the assumptions in \secref{Assumptions on Loss Function} hold. Assume further that the approximation accuracy satisfies $0 < \epsilon \leq \frac{9 (B+3)^4 p}{8 \mu}$, the H\"{o}lder smoothness parameter satisfies $\eta > (2 \alpha + 2)d$, and the proportion of active clients is $\tau = 1$. Then, the $\lrgd$ algorithm described in \algoref{Alg: FedLRGD} with approximate rank given by
\begin{equation}
\label{Eq: Rank lower bound condition}
\begin{aligned}
r & = \! \left\lceil\rule{0cm}{1.1cm}\right. \!\! \max\!\left\{\rule{0cm}{1.1cm}\right. \!\!\! e \sqrt{d} \, 2^d (\eta + d)^d , \! \left(\rule{0cm}{0.9cm}\right. \!\!\!\! \left(\!\frac{L_2^d \eta^{d} e^{\eta (d + 1)} d^{\eta / 2} (2\eta + 2d)^{\eta d}}{(\eta - 1)^{\eta d}}\!\right) \\
& \qquad\quad\, \cdot \! \left(\!\frac{\kappa \! \left(F(\theta^{(0)}) - F_* + \frac{9 (B + 3)^4 p}{2 \mu}\right)}{(\kappa - 1)\epsilon}\!\right)^{\!\! \frac{d}{2}} \!\! \left.\rule{0cm}{0.9cm}\right)^{\!\! \frac{1}{\eta - (2 \alpha + 2)d}} \!\! \left.\rule{0cm}{1.1cm}\right\} \!\!\left.\rule{0cm}{1.1cm}\right\rceil ,
\end{aligned}
\end{equation}
and number of inexact GD iterations given by
\begin{equation}
\label{Eq: Iteration count}
S = \left\lceil \left(\log\!\left(\frac{\kappa}{\kappa-1}\right)\right)^{\! -1} \log\!\left(\frac{F(\theta^{(0)}) - F_* + \frac{9 (B + 3)^4 p}{2 \mu}}{\epsilon} \right) \right\rceil \! ,
\end{equation}
produces an $\epsilon$-approximate solution in the sense of \eqref{Eq: Approx solution 1} with a federated oracle complexity of
$$ \Gamma(\lrgd) = r^2 + rs + rS + \phi m r \, , $$
where $s$ is the number of training data samples per client, $m$ is the number of clients, $\phi$ is the communication-to-computation ratio, and the various other constants and problem variables are defined in \secref{Introduction} and \secref{Formal Setup}.
\end{theorem}

\thmref{Thm: Federated Oracle Complexity of FedLRGD} is established in \secref{Analysis of Federated Low Rank Gradient Descent}. It roughly states that $\Gamma(\lrgd)$ scales like $\phi m (p/\epsilon)^{\Theta(d/\eta)}$ (neglecting typically sub-dominant factors; see, e.g., the assumptions and proof of \propref{Prop: Comparison between FedLRGD and FedAve}) as mentioned in \secref{Main Contributions}. At a high level, the proof of \thmref{Thm: Federated Oracle Complexity of FedLRGD} has three parts. First, we use the uniform piecewise polynomial approximations from \lemref{Lemma: Uniform Piecewise Polynomial Approximation} to show that partial derivatives of the loss function are approximately low rank. We emphasize that this approximate low rankness is deduced from the smoothness of loss functions in the data\textemdash a property we had set out to exploit since \secref{Introduction}. This approximate low rankness implies that $\widehat{\nabla F}$ and $\nabla_{\theta} F$ are close to each other. In particular, we show that under the assumptions of the theorem statement, the error in approximating the true gradient of the empirical risk at the server satisfies $\big\|\widehat{\nabla F} - \nabla_{\theta} F\big\|_2^2 = O(p)$ (see \eqref{Eq: Control on Approximation Error of true Gradient} and \eqref{Eq: Bound on delta param} for details). Second, we control the error $F(\theta^*) - F_*$ using the aforementioned bound on $\big\|\widehat{\nabla F} - \nabla_{\theta} F\big\|_2^2$ from the first part and standard tools from the inexact GD literature, cf. \cite{FriedlanderSchmidt2012}. Finally, we impose the condition in \eqref{Eq: Approx solution 1}, which naturally engenders certain choices of $r$ and $S$ (specifically, \eqref{Eq: Rank lower bound condition} and \eqref{Eq: Iteration count}), and from there, a direct application of \defref{Def: Federated oracle complexity} yields the aforementioned federated oracle complexity.

We remark that it is the process of rigorously translating the smoothness of partial derivatives of loss functions in the data to an approximate low rankness property that requires us to make many of the assumptions at the outset of this subsection and in \secref{Assumptions on Loss Function}. Once we have approximate low rankness of partial derivatives of loss functions, however, it is reasonably straightforward to intuitively understand why \eqref{Eq: Key approx} holds, and in turn, why \algoref{Alg: FedLRGD} works and a result like \thmref{Thm: Federated Oracle Complexity of FedLRGD} holds. Furthermore, canonical machine learning models can exhibit the desired approximate low rankness property that lies at the heart of our method. For example, in \appref{Low Rank Structure in Neural Networks}, we illustrate that gradients of loss functions associated with neural networks used to train classifiers for the MNIST and CIFAR-10 datasets (see \cite{LeCunCortesBurgesMNIST,KrizhevskyNairHintonCIFAR10}) are approximately low rank. This gives further credence to the potential utility of $\lrgd$.

\subsection{Comparison to Federated Averaging}
\label{Comparison to Federated Averaging}

We finally compare the federated oracle complexity of $\lrgd$ in \thmref{Thm: Federated Oracle Complexity of FedLRGD} with the widely used benchmark $\FedAve$ algorithm from the literature \cite{McMahanetal2017}. Specifically, we consider the precise formulation of $\FedAve$ in \cite[Algorithm 1, Theorem 1]{Reisizadehetal2020a}. Recall that in the $\FedAve$ algorithm, active clients perform several iterations of SGD using their private data in each epoch, and the server aggregates individual client updates at the end of each epoch and broadcasts the aggregated update back to clients at the beginning of the next epoch. The authors of \cite{Reisizadehetal2020a} provide a sharp convergence analysis of $\FedAve$.\footnote{In fact, a more sophisticated version of $\FedAve$ with quantization is analyzed in \cite{Reisizadehetal2020a}.} Building upon their results, we derive the federated oracle complexity of $\FedAve$ in the ensuing theorem. For simplicity, we assume that the number of clients $m$ grows (i.e., $m \rightarrow \infty$), and other non-constant problem parameters vary at different rates with respect to $m$.

\begin{theorem}[Federated Oracle Complexity of $\FedAve$]
\label{Thm: Federated Oracle Complexity of FedAve}
Assume that:
\begin{enumerate}
\item The smoothness in $\theta$ and strong convexity assumptions of \secref{Assumptions on Loss Function} hold.
\item The parameters $\tau \in (0,1]$ and $L_1 > \mu > 0$ are constants, i.e., $\Theta(1)$, while other problem parameters may vary with $m \in \N$. 
\item Any stochastic gradient $\widetilde{\nabla} f_{c} (\theta)$ computed by a client $c \in [m]$ in the $\FedAve$ algorithm is unbiased and has variance bounded by $\sigma^2 > 0$ (see \cite[Assumption 3]{Reisizadehetal2020a}):\footnote{The scaling $\sigma^2 = O(p)$ is natural since the squared $\ell^2$-norm in the expectation has $p$ terms.}
\begin{align*}
& \forall c \in [m], \, \forall \theta \in \R^p, \\
& \E\!\left[\left\|\widetilde{\nabla} f_{c} (\theta) - \frac{1}{s} \sum_{j = 1}^{s}{\nabla_{\theta} f (x^{(c,j)};\theta)} \right\|_2^2\right] \leq \sigma^2 = O(p) \, . 
\end{align*}
\item The number of gradient computations $b \in \N$ per client in an epoch of $\FedAve$ satisfies $b = O(m)$.
\item $\E\big[\big\|\theta^{(T_0)} - \theta^{\star}\big\|_2^2\big] = O(p)$, where $\theta^{(T_0)} \in \R^p$ denotes the parameter vector at the server after $T_0 = \Theta\big(\frac{m}{b}\big)$ epochs of $\FedAve$.\footnote{The $O(p)$ scaling here is also natural because the parameter $\theta$ usually resides in an $\ell^{\infty}$-norm ball, e.g., $[0,1]^p$, rather than all of $\R^p$ in applications. Since the diameter of such an $\ell^{\infty}$-norm ball scales as $O(p)$, $\E\big[\big\|\theta^{(T_0)} - \theta^{\star}\big\|_2^2\big]$ is upper bounded by $O(p)$.\label{Footnote: Scaling in p}}
\end{enumerate}
Then, for all $\epsilon \in (0,1]$, all $\phi > 0$ larger than some constant, and all sufficiently large $m \in \N$, we have
$$ \Gamma(\FedAve) = O\!\left(\phi m \left(\frac{p}{\epsilon}\right)^{\! 3/4}\right) , $$
where we compute the federated oracle complexity above by minimizing over all choices of $b$ and all numbers of epochs $T \in \N$ in $\FedAve$ such that it produces an $\epsilon$-approximate solution in the sense of \eqref{Eq: Approx solution 2}.
\end{theorem}

\thmref{Thm: Federated Oracle Complexity of FedAve} is established in \appref{Proof of Thm Federated Oracle Complexity of FedAve}. To easily compare the federated oracle complexities of the $\lrgd$ and $\FedAve$ algorithms, we assume as before that the number of clients $m \rightarrow \infty$, all the problem variables $d,p,s,r,n,\epsilon^{-1},\eta,\phi$ are non-decreasing functions of $m$, and the other problem parameters $\tau,L_1,L_2,\mu,\alpha,B$ are constants with respect to $m$. (We refer readers back to \secref{Introduction} and \secref{Formal Setup} for the definitions of these parameters.) In particular, this means that the condition number $\kappa$ is also constant. Although improving the dependence of oracle complexity on condition number is a vast area of research in optimization theory (see, e.g., accelerated methods \cite{Nesterov1983}, variance reduction \cite{JohnsonZhang2013}, etc.), we do not consider this aspect of the problem in this work. The proposition below elucidates a broad and important regime where $\lrgd$ has better federated oracle complexity than $\FedAve$.

\begin{proposition}[Comparison between $\lrgd$ and $\FedAve$]
\label{Prop: Comparison between FedLRGD and FedAve}
Suppose that the assumptions of \thmref{Thm: Federated Oracle Complexity of FedLRGD} and \thmref{Thm: Federated Oracle Complexity of FedAve} hold. Fix any absolute constants $\lambda \in (0,1)$, $\beta \in \big(0,\frac{1}{2}\big]$, and $\xi \in \big(0,\frac{3}{4}\big)$, and then consider any sufficiently large constants $c_1,c_2 > 0$ such that $c_2 > c_1 \geq 2/\xi$. Assume further that the number of samples per client $s = O(\phi m)$, the data dimension $d = O(\log(n)^{\lambda})$, the approximation accuracy $\epsilon = \Theta(n^{-\beta})$, the H\"{o}lder smoothness parameter $\eta = \Theta(d)$ such that $(c_1 + 2 \alpha + 2) d \leq \eta \leq (c_2 + 2 \alpha + 2) d$, and $F(\theta^{(0)}) - F_* = O(p)$, where $n = ms + r$ is the total number of training data samples, and $p \in \N$ is the parameter dimension.\footnote{Note that $F(\theta^{(0)}) - F_* = O(p)$, because
\begin{align*}
F(\theta^{(0)}) - F_* & = F(\theta^{(0)}) - F(\theta^{\star}) \\
& = \nabla_{\theta} F (\tau \theta^{(0)} + (1-\tau) \theta^{\star})^{\T} (\theta^{(0)} - \theta^{\star}) \\
& =  \left(\nabla_{\theta} F (\tau \theta^{(0)} + (1-\tau) \theta^{\star}) - \nabla_{\theta} F (\theta^{\star})\right)^{\! \T} (\theta^{(0)} - \theta^{\star}) \\
& \leq \left\| \nabla_{\theta} F (\tau \theta^{(0)} + (1-\tau) \theta^{\star}) - \nabla_{\theta} F (\theta^{\star})\right\|_2 \left\|\theta^{(0)} - \theta^{\star}\right\|_2 \\
& \leq L_1 \left\|\theta^{(0)} - \theta^{\star}\right\|_2^2 = O(p) \, ,
\end{align*}
where $\theta^{\star} \in \R^p$ is the optimal parameter argument defined in \secref{Assumptions on Loss Function}, the second equality follows from the mean value theorem for some $\tau \in (0,1)$ (see \cite[Theorem 12.14]{Apostol1974}), the third equality follows from the fact that $\nabla_{\theta} F (\theta^{\star}) = 0$ at the minimum (which holds because $F$ is strongly convex due to the strong convexity assumption in \secref{Assumptions on Loss Function}, cf. \cite[Section 4.1]{JadbabaieMakurShah2021}), the next inequality follows from the Cauchy-Schwarz inequality, the last inequality holds because $\nabla_{\theta} F$ is $L_1$-Lipschitz continuous (which is inherited from the smoothness in parameter assumption in \secref{Assumptions on Loss Function}, cf. \cite[Section 4.1]{JadbabaieMakurShah2021}), and the scaling with $p$ in the last line follows from the reasoning in Footnote \ref{Footnote: Scaling in p}.} Then, the federated oracle complexities of $\lrgd$ and $\FedAve$ scale as
\begin{align*}
\Gamma(\lrgd) & = O\!\left(\phi m \left(\frac{p}{\epsilon} \right)^{\! \xi} \right) , \\
\Gamma(\FedAve) & \leq \Gamma_+(\FedAve) \triangleq \Theta\!\left( \phi m \left(\frac{p}{\epsilon} \right)^{\! 3/4} \right) , 
\end{align*}
and satisfy
$$ \lim_{m \rightarrow \infty}{\frac{\Gamma(\lrgd)}{\Gamma_+(\FedAve)}} = 0 \, . $$
\end{proposition}

This result is proved in \appref{Proof of Prop Comparison between FedLRGD and FedAve}. We briefly explain the relevance of the regime presented in \propref{Prop: Comparison between FedLRGD and FedAve} where $\lrgd$ beats $\FedAve$. Firstly, we state conditions on the scaling of $d$ and $\epsilon$ in terms of $n$ (rather than $m$), because dimension and statistical error are typically compared to the number of available training samples in high-dimensional statistics and machine learning problems. Secondly, as noted in \cite{JadbabaieMakurShah2021}, we assume that $\epsilon = \Theta(n^{-\beta})$ for $\beta \in \big(0,\frac{1}{2}\big]$, because the empirical risk in \eqref{Eq: Standard formulation} is an approximation of some true risk, and the statistical error between these quantities is known to be $O(n^{-1/2})$ (with high probability). As we only minimize empirical risk in machine learning problems as a proxy for true risk minimization, there is no gain in reducing the approximation accuracy $\epsilon$ of federated optimization algorithms below this $n^{-1/2}$ threshold. Thirdly, our result subsumes one of the most widely studied regimes for training neural networks in recent times\textemdash the \emph{overparametrized} regime where $p \gg n$ (see \cite{PoggioBanburskiLiao2020} and the references therein). Fourthly, as discussed in \cite{JadbabaieMakurShah2021}, when $\lambda$ is close to $1$, e.g., $\lambda = 0.99$, the $d = O(\log(n)^{0.99})$ assumption is not too restrictive, because several machine learning problems have meaningful feature vectors with dimension that is logarithmic in the number of training samples. Indeed, results like the \emph{Johnson-Lindenstrauss lemma}, cf. \cite{DasguptaGupta2003}, illustrate how $n$ high-dimensional feature vectors can be mapped to $O(\log(n))$-dimensional feature vectors while approximately preserving the distances between pairs of vectors, and hence, conserving the relative geometry of the vectors. Fifthly, we informally summarize the main theoretical message of \propref{Prop: Comparison between FedLRGD and FedAve} for clarity: If the data dimension $d = O(\log(n)^{0.99})$ is small and the gradient of the loss function possesses sufficient smoothness in the data $\eta = \Theta(d)$, then our $\lrgd$ algorithm provably beats the benchmark $\FedAve$ algorithm in federated oracle complexity. (In other regimes, the $\FedAve$ algorithm could perform better than our method.) This theoretically portrays the utility of exploiting smoothness of loss functions in the data for federated optimization. Finally, we note that although we show the theoretical advantage of our $\lrgd$ algorithm under several assumptions, running \algoref{Alg: FedLRGD} itself does not require most of these assumptions, e.g., strong convexity, small data dimension, etc., to hold. Therefore, our main conceptual contribution of exploiting smoothness in data, or approximate low rank structure, of loss functions for federated optimization could be of utility in settings beyond the assumptions considered in this work.

\section{Analysis of Latent Variable Model Approximation}
\label{Analysis of Latent Variable Model Approximation}

We prove \lemref{Lemma: Uniform Piecewise Polynomial Approximation} and \thmref{Thm: Low Rank Approximation} in this section.

\subsection{Proof of \lemref{Lemma: Uniform Piecewise Polynomial Approximation}}
\label{Proof of Lemma Uniform Piecewise Polynomial Approximation}

\begin{proof}
We will follow the proof idea of \cite[Lemma 3]{Xu2017}, but the details of our proof are different since we use a more general H\"{o}lder class definition than \cite{Xu2017}. Fix any vector $\vartheta \in \R^p$. For any set $I \in \S_{d,q}$, let $B_I$ denote the closed $\ell^1$-ball corresponding to $I$ in our minimal $\ell^1$-covering of $[0,1]^d$, and $z_I \in \net$ denote the center of $B_I$, i.e., if $I = I_j$ for $j \in [|\net|]$ according to the enumeration defined in \secref{Latent Variable Model Approximation}, then $B_I = B_j$ and $z_I = z^{(j)}$. Moreover, for any $I \in \S_{d,q}$, construct the Taylor polynomial with degree $l$ of $g(\cdot;\vartheta)$ around $z_I$:
\begin{equation}
\label{Eq: Taylor polynomials}
\forall y \in [0,1]^d, \enspace P_{I}(y;\vartheta) \triangleq \sum_{s \in \Z_+^d : |s| \leq l}{\frac{\nabla_{x}^{s}g(z_I;\vartheta)}{s!} (y - z_I)^s} \, . 
\end{equation}
Then, we have
\begin{align*}
& \sup_{y \in [0,1]^d}{\left|g(y;\vartheta) - P_{q,l}(y;\vartheta)\right|} \\
& = \max_{I \in \S_{d,q}} \sup_{y \in I}{\left|g(y;\vartheta) - P_{I}(y;\vartheta)\right|} \\
& \leq \max_{I \in \S_{d,q}} \sup_{y \in B_I \cap [0,1]^d}{\left|g(y;\vartheta) - P_{I}(y;\vartheta)\right|} \\
& \leq \max_{I \in \S_{d,q}} \sup_{y \in B_I \cap [0,1]^d}{L_2 \left\|y - z_I\right\|_{1}^{\eta - l} \sum_{s \in \Z_+^d : |s| = l}{\frac{|(y - z_I)^s|}{s!}}} \\
& = \frac{L_2}{l!} \max_{I \in \S_{d,q}} \sup_{y \in B_I \cap [0,1]^d}{\left\|y - z_I\right\|_{1}^{\eta}} \\
& \leq \frac{L_2}{l! q^{\eta}} \, ,
\end{align*}
where the piecewise polynomial function $P_{q,l}(\cdot;\vartheta)$ is defined using the Taylor polynomials $P_I(\cdot;\vartheta)$, the second inequality follows from the fact that $I \subseteq B_I \cap [0,1]^d$ by construction, the fourth equality follows from the multinomial theorem, the last inequality holds because each $B_I$ is an $\ell^1$-ball with center $z_I$ and radius $q^{-1}$, and the third inequality holds because
\begin{align*}
& \left|g(y;\vartheta) - P_{I}(y;\vartheta)\right| \\
& = \left|\sum_{s \in \Z_+^d : |s| = l}{\!\!\!\!\frac{\nabla_{x}^{s}g(z_I + \tau(y - z_I);\vartheta) - \nabla_{x}^{s}g(z_I;\vartheta)}{s!} (y - z_I)^s}\right| \\
& \leq \!\!\sum_{s \in \Z_+^d : |s| = l}{\!\!\!\!\frac{|\nabla_{x}^{s}g(z_I + \tau(y - z_I);\vartheta) - \nabla_{x}^{s}g(z_I;\vartheta)|}{s!} |(y - z_I)^s|} \\
& \leq L_2 \left\|y - z_I\right\|_{1}^{\eta - l} \sum_{s \in \Z_+^d : |s| = l}{\frac{|(y - z_I)^s|}{s!}} \, ,
\end{align*}
where we use Taylor's theorem with Lagrange remainder term which holds for some $\tau \in (0,1)$ \cite[Theorem 12.14]{Apostol1974}, the triangle inequality, and the $(\eta,L_2)$-H\"{o}lder class assumption. This completes the proof.
\end{proof}

\subsection{Proof of \thmref{Thm: Low Rank Approximation}}
\label{Proof of Thm Low Rank Approximation}

\begin{proof}
As before, we follow the proof idea of \cite[Proposition 1]{Xu2017}, but the details of our proof are again different since we use a more general H\"{o}lder class definition than \cite{Xu2017}. For any $q > 0$ (to be selected later), construct the family of piecewise polynomial functions $P_{q,l} : [0,1]^d \times \R^p \rightarrow \R$ as in \lemref{Lemma: Uniform Piecewise Polynomial Approximation} using the families of Taylor polynomials $P_{I} : [0,1]^d \times \R^p \rightarrow \R$ defined in \eqref{Eq: Taylor polynomials} for $I \in \S_{d,q}$. Moreover, define the latent variable model $N \in \R^{n \times k}$ such that:
$$ \forall i \in [n], \, \forall j \in [k], \enspace N_{i,j} = P_{q,l}(y^{(i)};\vartheta^{(j)}) \, , $$
where $y^{(1)},\dots,y^{(n)} \in [0,1]^d$ and $\vartheta^{(1)},\dots,\vartheta^{(k)} \in \R^p$ are the latent parameters that define $M$ in \eqref{Eq: LVM}. It is straightforward to verify that $N$ is ``good'' induced $(1,\infty)$-norm approximation of $M$. Indeed, applying \lemref{Lemma: Uniform Piecewise Polynomial Approximation}, we get
\begin{equation}
\label{Eq: (1,infinity) bound}
\max_{i \in [n], \, j \in [k]}{\left|M_{i,j} - N_{i,j}\right|} \leq \frac{L_2}{l! q^{\eta}} \, . 
\end{equation}

We next bound the rank of $N$. Let $U(x) = [x^s : s \in \Z_+^d, \, |s| \leq l]^{\T} \in \R^{D}$ be the vector of monomials with degree at most $l$ for $x \in [0,1]^d$, where $D = \binom{l + d}{d}$ is the number of such monomials. For each $I \in \S_{d,q}$, suppose $P_{I}(x;\theta) = U(x)^{\T} V_{I}(\theta)$ for all $x \in [0,1]^d$ and $\theta \in \R^p$, where $V_{I}(\theta) \in \R^D$ is the vector of coefficients of the Taylor polynomial $P_{I}(\cdot;\theta)$. Then, we have
\begin{align*}
N = \sum_{I \in \S_{d,q}} & \underbrace{\left[\I\{y^{(1)} \in I\} U(y^{(1)}) \, \cdots \, \I\{y^{(n)} \in I\} U(y^{(n)}) \right]^{\T}}_{\in \, \R^{n \times D} \text{ with rank } \leq D} \\
& \cdot \underbrace{\left[V_I(\vartheta^{(1)}) \, \cdots \, V_I(\vartheta^{(k)}) \right]}_{\in \, \R^{D \times k} \text{ with rank } \leq D} \, ,
\end{align*}
which implies that
$$ \rank(N) \leq |\S_{d,q}| D = |\net| \binom{l + d}{d} \, , $$
where we use the sub-additivity of rank, and $\net$ is the $q^{-1}$-net introduced in \secref{Latent Variable Model Approximation}. Hence, to upper bound $\rank(N)$, it suffices to upper bound the $\ell^1$-covering number $|\net|$ of the hypercube $[0,1]^d$. Let $B$ denote the closed $\ell^1$-ball with radius $(2q)^{-1}$ that is centered at the origin. Using the volumetric argument in \cite[Proposition 4.2.12]{Vershynin2018} and \cite[Theorem 14.2]{Wu2020}, we obtain
\begin{align*}
|\net| & \leq \frac{\vol([0,1]^d + B)}{\vol(B)} \\
& \leq d! q^d \vol\!\left(\left[-\frac{1}{2q},1+\frac{1}{2q}\right]^d\right) \\
& = d! (q + 1)^d \, ,
\end{align*}
where $\vol(\cdot)$ denotes the volume in $\R^d$ and $+$ denotes the \emph{Minkowski sum} of sets in $\R^d$, the second inequality follows from substituting the well-known volume of the $\ell^1$-ball as well as the fact that $[0,1]^d + B \subseteq \big[-\frac{1}{2q},1+\frac{1}{2q}\big]^d$, and the final equality follows from substituting the volume of a hypercube. This produces the following upper bound on the rank of $N$:
\begin{equation}
\label{Eq: rank bound}
\rank(N) \leq d! \binom{l + d}{d} (q + 1)^d \, . 
\end{equation}

Finally, for any $r \in [\min\{n,k\}]$ satisfying $r \geq e \sqrt{d} \, 2^d (l + d)^d$, choose $q = r^{1/d} (d!)^{-1/d} \binom{l + d}{d}^{\! -1/d} - 1$ such that $r = d! \binom{l + d}{d} (q + 1)^d$. Note that $q > 0$ because
\begin{align*}
r \geq e \sqrt{d} \, 2^d (l + d)^d \quad & \Rightarrow \quad r > e \sqrt{d} \, (l + d)^d \\
& \Rightarrow \quad r > d! \binom{l + d}{d} \\
& \Leftrightarrow \quad q > 0 \, , 
\end{align*}
where the second implication follows from the bound
\begin{equation}
\label{Eq: Bound on nPr} 
d! \binom{l + d}{d} \leq e \sqrt{d} \left(\frac{d}{e}\right)^{\! d} \left(\frac{e (l + d)}{d} \right)^{\! d} = e \sqrt{d} (l + d)^d \, , 
\end{equation}
which follows from a standard upper bound on binomial coefficients and Stirling's approximation \cite[Chapter II, Section 9, Equation (9.15)]{Feller1968}. Then, observe that
\begin{align*}
\frac{1}{n k} \left\|M - M_{r}\right\|_{\Fro}^2 & \leq \frac{1}{n k} \left\|M - N\right\|_{\Fro}^2 \\
& \leq \frac{L_2^2}{(l!)^2 q^{2\eta}} \\
& \leq \frac{L_2^2}{(l!)^2} \left( \left(\frac{r}{d! \binom{l + d}{d}}\right)^{\! 1/d} - 1 \right)^{\! -2\eta} \\
& \leq \frac{L_2^2}{(l!)^2} \left( \frac{r^{1/d}}{e^{1/d} d^{1/(2d)} (l + d)} - 1 \right)^{\! -2\eta} \\
& \leq \frac{L_2^2 e^{2\eta/d} d^{\eta/d} 4^{\eta} (l + d)^{2 \eta}}{(l!)^2} \left( \frac{1}{r} \right)^{\! 2\eta/d} ,
\end{align*}
where the first inequality follows from \eqref{Eq: rank bound} and the \emph{Eckart-Young-Mirsky theorem} \cite[Section 7.4.2]{HornJohnson2013}, the second inequality follows from \eqref{Eq: (1,infinity) bound}, the fourth inequality uses \eqref{Eq: Bound on nPr}, and the fifth inequality follows from the assumption that 
\begin{align*}
& r \geq e \sqrt{d} \, 2^d (l + d)^d \quad \Leftrightarrow \\
& \qquad \qquad \frac{r^{1/d}}{e^{1/d} d^{1/(2d)} (l + d)} - 1 \geq \frac{r^{1/d}}{2 e^{1/d} d^{1/(2d)} (l + d)} \, . 
\end{align*}
This completes the proof.
\end{proof}

\section{Analysis of Federated Low Rank Gradient Descent}
\label{Analysis of Federated Low Rank Gradient Descent}

In this section, we establish the federated oracle complexity of the $\lrgd$ algorithm as stated in \thmref{Thm: Federated Oracle Complexity of FedLRGD}. To do this, we need the ensuing lemma from the literature which bounds how close the ERM objective value $F(\theta^{(S)})$, at the output parameter $\theta^* = \theta^{(S)}$ of \algoref{Alg: FedLRGD}, is to the global minimum in \eqref{Eq: ERM Minimum} \cite[Lemma 2.1]{FriedlanderSchmidt2012} (cf. \cite[Theorem 2, Equation (19)]{SoZhou2017} and \cite[Lemma 4]{JadbabaieMakurShah2021}).

\begin{lemma}[Inexact GD Bound {\cite[Lemma 2.1]{FriedlanderSchmidt2012}}]
\label{Lemma: Inexact GD Bound}
Suppose the empirical risk $F : \R^p \rightarrow \R$ is continuously differentiable and $\mu$-strongly convex, and its gradient $\nabla_{\theta} F : \R^p \rightarrow \R^p$ is $L_1$-Lipschitz continuous. Then, for all $S \in \N$, we have
\begin{align*}
& 0 \leq F(\theta^{(S)}) - F_* \leq \left(1 - \frac{1}{\kappa}\right)^{\! S} \!\left(F(\theta^{(0)}) - F_*\right) \\
& + \frac{1}{2 L_1} \sum_{\gamma = 1}^{S}{\left(1 - \frac{1}{\kappa}\right)^{\! S-\gamma} \left\|\widehat{\nabla F}(\theta^{(\gamma-1)}) - \nabla_{\theta} F(\theta^{(\gamma-1)})\right\|_2^2} \, , 
\end{align*}
where $\{\theta^{(\gamma)} \in \R^p : \gamma \in \{0,\dots,S\}\}$ are the inexact GD updates in \eqref{Eq: Inexact GD update} in \algoref{Alg: FedLRGD} with arbitrary initialization $\theta^{(0)}$, $\widehat{\nabla F}(\theta^{(\gamma)})$ is our approximation of $\nabla_{\theta} F(\theta^{(\gamma)})$ for $\gamma \in \{0,\dots,S-1\}$ in \algoref{Alg: FedLRGD} as shown in \eqref{Eq: Inexact Gradient}, and the condition number $\kappa = L_1/\mu$ is defined in \secref{Assumptions on Loss Function}.
\end{lemma}

We next prove \thmref{Thm: Federated Oracle Complexity of FedLRGD} using \lemref{Lemma: Uniform Piecewise Polynomial Approximation}, \lemref{Lemma: Inexact GD Bound}, and the various assumptions in \secref{Assumptions on Loss Function} and \secref{Federated Oracle Complexity of FedLRGD}.

\begin{proof}[Proof of \thmref{Thm: Federated Oracle Complexity of FedLRGD}]
First, to avoid notational clutter in this proof, we re-index the training datasets of the server and clients using $[n]$ with $n = ms + r$ (cf. \secref{Federated Learning}). Specifically, we let $x^{(j)} = x^{(0,j)}$ for $j \in [r]$ (i.e., the server's training data is indexed first), and $x^{(b)} = x^{(c,j)}$ where $b - r = (c-1)s + j$ for $b \in [n]\backslash [r]$, $c \in [m]$, and $j \in [s]$. Correspondingly, we re-index the weights $\{w_{j,k}^{(i,c)} \in \R: i \in [p], \, j \in [s] , \, k \in [r], \, c \in [m]\}$, which are defined by \eqref{Eq: Non-cumulative weights} and implicitly used in the $\lrgd$ algorithm, so that $w^{(i)}_{b,k} = w_{j,k}^{(i,c)}$ where $b - r = (c-1)s + j$ for $b \in [n]\backslash [r]$. 

The majority of this proof focuses on analyzing the iteration complexity of the inexact GD method employed in the last epoch of \algoref{Alg: FedLRGD}. Once we have control over this iteration complexity, we can easily deduce federated oracle complexity. Our main tool for studying iteration complexity will be \lemref{Lemma: Inexact GD Bound}. But to use this lemma, we require a bound on $\|\widehat{\nabla F} - \nabla_{\theta} F\|_2^2$. So, we begin our analysis by fixing any (sufficiently large) $r \in \N$, any index $i \in [p]$, and any parameter $\theta \in \R^p$, and noticing that
\begin{align}
& \left| \frac{\partial F}{\partial \theta_i}(\theta) - \widehat{\nabla F}_i(\theta) \right| \nonumber \\
& = \frac{1}{n} \left| \sum_{k = 1}^{r}{g_i(x^{(0,k)};\theta)} + \sum_{c = 1}^{m}{\sum_{j = 1}^{s}{g_i(x^{(c,j)};\theta)}} \right. \nonumber \\
& \quad \quad \enspace \left. - \sum_{k = 1}^{r}{g_i(x^{(0,k)};\theta) \left(1 + \sum_{c = 1}^{m} \sum_{j = 1}^{s}{w_{j,k}^{(i,c)}} \right)} \right| \nonumber \\
& =  \frac{1}{n} \left| \sum_{k = 1}^{n}{g_i(x^{(k)};\theta)} - \sum_{k = 1}^{r}{g_i(x^{(k)};\theta) \left(1 + \sum_{j = r + 1}^{n}{w_{j,k}^{(i)}} \right)} \right| \nonumber \\
& = \frac{1}{n} \left| \sum_{j = r+1}^{n}{g_i(x^{(j)};\theta)} - \sum_{k = 1}^{r}{g_i(x^{(k)};\theta) \sum_{j = r+1}^{n}{w_{j,k}^{(i)}} } \right| \nonumber \\
& \leq \frac{1}{n} \sum_{j = r+1}^{n} \left| g_i(x^{(j)};\theta) - \sum_{k = 1}^{r} g_i(x^{(k)};\theta) w_{j,k}^{(i)} \right| \nonumber \\
& \leq \max_{j \in [n]\backslash [r]} \left| g_i(x^{(j)};\theta) - \sum_{k = 1}^{r} g_i(x^{(k)};\theta) w_{j,k}^{(i)} \right| ,
\label{Eq: step 1}
\end{align}
where the first equality follows from \eqref{Eq: Full gradient}, \eqref{Eq: v and w}, and \eqref{Eq: Inexact Gradient}, and we use the weights $\{w_{j,k}^{(i,c)} \in \R: j \in [s] , \, k \in [r], \, c \in [m]\}$ defined by \eqref{Eq: Non-cumulative weights}, the second equality follows from the aforementioned re-indexing, the fourth inequality follows from swapping the order of summations in the second term on the right hand side and applying the triangle inequality, and the fifth inequality holds because $n \geq m s$. To upper bound \eqref{Eq: step 1}, we next develop some properties of a uniform piecewise polynomial approximation of $g_i$.

Let $D = \binom{l + d}{d}$ and $\S_{d,q}$ be the partition of the hypercube $[0,1]^d$ defined in \eqref{Eq: partition} for any $q > 0$. Moreover, fix the parameter $q > 0$ such that $r = D |\S_{d,q}|$. Since $g_i(\cdot;\theta) : [0,1]^d \rightarrow \R$ satisfies the smoothness in data assumption (i.e., the $(\eta,L_2)$-H\"{o}lder class assumption) in \secref{Assumptions on Loss Function}, \lemref{Lemma: Uniform Piecewise Polynomial Approximation} implies that there exists a piecewise polynomial function $P_{q,l}(\cdot;\theta) : [0,1]^d \rightarrow \R$ with $l = \lceil \eta \rceil - 1$ such that
\begin{equation}
\label{Eq: uniform step}
\sup_{x \in [0,1]^d}{\left|g_i(x;\theta) - P_{q,l}(x;\theta)\right|} \leq \frac{L_2}{l! q^{\eta}} \, . 
\end{equation}
Here, as shown in the proof of \thmref{Thm: Low Rank Approximation} in \secref{Proof of Thm Low Rank Approximation}, we have
\begin{align*}
P_{q,l}(x;\theta) & = \sum_{I \in \S_{d,q}}{\I\!\{x \in I\} U(x)^{\T} V_I(\theta)} \\
& = \sum_{I \in \S_{d,q}} \sum_{v \in \Z_+^d: \, |v| \leq l}{\I\!\{x \in I\} x^v [V_I(\theta)]_v} 
\end{align*}
for all $x \in [0,1]^d$, where $U(x) = [x^v : v \in \Z_+^d, \, |v| \leq l]^{\T} \in \R^{D}$ is the vector of monomials with degree at most $l$, and $V_{I}(\theta) \in \R^D$ are vectors of coefficients (see \secref{Proof of Thm Low Rank Approximation}) which we may index with $\{v \in \Z_+^d : |v| \leq l\}$. Equivalently, we can write:
\begin{equation}
\label{Eq: low rank form}
\forall x \in [0,1]^d, \enspace P_{q,l}(x;\theta) = \sum_{w = 1}^{r}{\phi_w(x) \psi_w(\theta)} \, , 
\end{equation}
where we enumerate $\S_{d,q} \times \{v \in \Z_+^d: \, |v| \leq l\}$ using the index set $[r]$, i.e., we employ a bijection $[r] \mapsto \S_{d,q} \times \{v \in \Z_+^d: \, |v| \leq l\}$, and for all $I \in \S_{d,q}$ and $v \in \Z_+^d$ with $|v| \leq l$ such that $[r] \ni w \mapsto (I,v)$, we let:
\begin{align}
\label{Eq: monomial}
\forall x \in [0,1]^d, \enspace \phi_w(x) & = \I\!\{x \in I\} x^v \, , \\
\forall \vartheta \in \R^d, \enspace \psi_w(\vartheta) & = [V_I(\vartheta)]_{v} \, . \nonumber
\end{align}
(Hence, as noted in \secref{Federated Oracle Complexity of FedLRGD}, the family of monomial functions $\{\phi_w : w \in [r]\}$ are linearly independent.) For every $j \in [n]\backslash [r]$, define the coefficients $a_{j,1},\dots,a_{j,r} \in \R$ via
$$ \left[a_{j,1} \enspace a_{j,2} \, \cdots \, a_{j,r}\right] \triangleq \left[\phi_{1}(x^{(j)}) \enspace \phi_{2}(x^{(j)}) \, \cdots \, \phi_r(x^{(j)}) \right] \Phi^{-\T} \, , $$
where the matrix $\Phi \in \R^{r \times r}$ is defined in \eqref{Eq: Phi def} and is assumed to be invertible (cf. \secref{Federated Oracle Complexity of FedLRGD}). Then, observe that for every $w \in [r]$ and $j \in [n]\backslash [r]$, 
$$ \phi_w(x^{(j)}) = \sum_{k = 1}^{r}{a_{j,k} \phi_w(x^{(k)})} \, , $$
which implies that
\begin{align}
P_{q,l}(x^{(j)};\theta) & = \sum_{w = 1}^{r}{\psi_w(\theta) \sum_{k = 1}^{r}{a_{j,k} \phi_w(x^{(k)})}} \nonumber \\
& = \sum_{k = 1}^{r}{a_{j,k} \sum_{w = 1}^{r}{ \phi_w(x^{(k)}) \psi_w(\theta)}} \nonumber \\
& = \sum_{k = 1}^{r}{a_{j,k} P_{q,l}(x^{(k)};\theta)} 
\label{Eq: Perfect low rank decomposition}
\end{align}
using \eqref{Eq: low rank form}. Therefore, letting $\cP^{(i)} \in \R^{r \times r}$ be the matrix given by \eqref{Eq: cP def} in \secref{Federated Oracle Complexity of FedLRGD} for the arbitrary set of parameter vectors $\vartheta^{(1)},\dots,\vartheta^{(r)} \in \R^p$ chosen by the server in the first epoch of \algoref{Alg: FedLRGD}, we obtain that for all $j \in [n]\backslash [r]$,
\begin{equation}
\label{Eq: Alternate form of true weights}
\begin{aligned}
& \left[a_{j,1} \enspace a_{j,2} \, \cdots \, a_{j,r}\right] \\
& = \left[P_{q,l}(x^{(j)};\vartheta^{(1)}) \enspace P_{q,l}(x^{(j)};\vartheta^{(2)}) \, \cdots \, P_{q,l}(x^{(j)};\vartheta^{(r)})\right] \\
& \quad \enspace \cdot (\cP^{(i)})^{-1} \, , 
\end{aligned}
\end{equation}
where we use \eqref{Eq: Perfect low rank decomposition} and the assumption in \secref{Federated Oracle Complexity of FedLRGD} that $\cP^{(i)}$ is invertible. From the relation in \eqref{Eq: Alternate form of true weights}, we now show that the $a_{j,k}$'s are ``close'' to the $w^{(i)}_{j,k}$'s.

To this end, notice that for any $j \in [n]\backslash [r]$,
\begin{align}
& \left\|\left[a_{j,1}  \enspace a_{j,2} \, \cdots \, a_{j,r} \right] - \left[w^{(i)}_{j,1} \enspace w^{(i)}_{j,2} \, \cdots \, w^{(i)}_{j,r}\right] \right\|_2 \nonumber \\
& \overset{\eqmakebox[A][c]{\footnotesize (i)}}{=} \left\|\left[P_{q,l}(x^{(j)};\vartheta^{(1)}) \, \cdots \, P_{q,l}(x^{(j)};\vartheta^{(r)})\right] (\cP^{(i)})^{-1} \right. \nonumber \\
& \quad \quad \quad \left. - \left[g_i(x^{(j)};\vartheta^{(1)}) \, \cdots \, g_i(x^{(j)};\vartheta^{(r)})\right] (G^{(i)})^{-1} \right\|_2 \nonumber \\
& \overset{\eqmakebox[A][c]{\footnotesize (ii)}}{\leq} \left\|\left[P_{q,l}(x^{(j)};\vartheta^{(1)}) \, \cdots \, P_{q,l}(x^{(j)};\vartheta^{(r)})\right] \right. \nonumber \\
& \quad \quad \quad \left. \cdot \left((\cP^{(i)})^{-1} - (G^{(i)})^{-1}\right)  \right\|_2 \nonumber \\
& \quad \enspace \,\, + \left\|\left(\left[P_{q,l}(x^{(j)};\vartheta^{(1)}) \, \cdots \, P_{q,l}(x^{(j)};\vartheta^{(r)})\right] \right.\right. \nonumber \\
& \quad \quad \quad \quad \left.\left. - \left[g_i(x^{(j)};\vartheta^{(1)}) \, \cdots \, g_i(x^{(j)};\vartheta^{(r)})\right]\right) (G^{(i)})^{-1} \right\|_2 \nonumber \\
& \overset{\eqmakebox[A][c]{\footnotesize (iii)}}{\leq} \left\|\left[P_{q,l}(x^{(j)};\vartheta^{(1)}) \, \cdots \, P_{q,l}(x^{(j)};\vartheta^{(r)})\right] \right\|_2 \nonumber \\
& \quad \enspace \,\, \cdot \left\| (\cP^{(i)})^{-1} - (G^{(i)})^{-1}  \right\|_{\op} \nonumber \\
& \quad \enspace \,\, + \left\| \left[P_{q,l}(x^{(j)};\vartheta^{(1)}) \, \cdots \, P_{q,l}(x^{(j)};\vartheta^{(r)})\right] \right. \nonumber \\
& \quad \quad \quad \enspace \left. - \left[g_i(x^{(j)};\vartheta^{(1)}) \, \cdots \, g_i(x^{(j)};\vartheta^{(r)})\right]\right\|_2 \left\|(G^{(i)})^{-1} \right\|_{\op} \nonumber \\ 
& \overset{\eqmakebox[A][c]{\footnotesize (iv)}}{\leq} \left\| \left[P_{q,l}(x^{(j)};\vartheta^{(1)}) \, \cdots \, P_{q,l}(x^{(j)};\vartheta^{(r)})\right] \right. \nonumber \\
& \quad \quad \quad \left. - \left[g_i(x^{(j)};\vartheta^{(1)}) \, \cdots \, g_i(x^{(j)};\vartheta^{(r)})\right]\right\|_2 \nonumber \\
&\quad \enspace \,\, \cdot \left(\left\|(G^{(i)})^{-1} \right\|_{\op} +  \left\| (\cP^{(i)})^{-1} - (G^{(i)})^{-1}  \right\|_{\op}\right) \nonumber \\ 
& \quad \enspace \,\, + \left\|\left[g_i(x^{(j)};\vartheta^{(1)}) \, \cdots \, g_i(x^{(j)};\vartheta^{(r)})\right] \right\|_2 \nonumber \\
& \quad \quad \quad \cdot \left\| (\cP^{(i)})^{-1} - (G^{(i)})^{-1}  \right\|_{\op} \nonumber \\
& \overset{\eqmakebox[A][c]{\footnotesize (v)}}{\leq} \frac{L_2 \sqrt{r}}{l! q^{\eta}} \left\|(G^{(i)})^{-1} \right\|_{\op} \nonumber \\
& \quad \enspace \,\, + \sqrt{r} \left( B + \frac{L_2}{l! q^{\eta}} \right) \! \left\| (\cP^{(i)})^{-1} - (G^{(i)})^{-1}  \right\|_{\op} \nonumber \\
& \overset{\eqmakebox[A][c]{\footnotesize (vi)}}{\leq} \frac{L_2 \sqrt{r}}{l! q^{\eta}} \left\|(G^{(i)})^{-1} \right\|_{\op} \nonumber \\
& \quad \enspace \,\, + \sqrt{r} \left( B + \frac{L_2}{l! q^{\eta}} \right) \!\! \left(\!\frac{ \left\|(G^{(i)})^{-1}\right\|_{\op}^2 \left\|G^{(i)} - \cP^{(i)}\right\|_{\Fro} }{1  -  \left\|(G^{(i)})^{-1}\right\|_{\op} \left\|G^{(i)} - \cP^{(i)}\right\|_{\Fro}}\!\right) \nonumber \\
& \overset{\eqmakebox[A][c]{\footnotesize (vii)}}{\leq} \frac{L_2 \sqrt{r}}{l! q^{\eta}} \left\|(G^{(i)})^{-1} \right\|_{\op} \nonumber \\
& \quad \enspace \,\, + \frac{L_2 r^{3/2}}{l! q^{\eta}} \left( B + \frac{L_2}{l! q^{\eta}} \right) \! \left(\frac{ \left\|(G^{(i)})^{-1}\right\|_{\op}^2 }{1  - \left( \frac{\left\|(G^{(i)})^{-1}\right\|_{\op} r L_2 }{l! q^{\eta}}\right) }\right) ,
\label{Eq: step 2}
\end{align}
where (i) uses \eqref{Eq: Alternate form of true weights} and \eqref{Eq: Non-cumulative weights}, (ii) and (iv) follow from the triangle inequality, (iii) follows from the definition of operator norm, (v) follows from \eqref{Eq: uniform step} and the boundedness assumption in \eqref{Eq: Bounded gradient}, (vii) follows from \eqref{Eq: uniform step}, and (vi) holds because for any two non-singular matrices $X,Y$ with common dimensions, we have
\begin{align*} 
& \left\|X^{-1} - Y^{-1}\right\|_{\op} \\
& = \|X^{-1} (Y - X) Y^{-1}\|_{\op} \\
& \leq \left\|X^{-1}\right\|_{\op} \left\|Y^{-1}\right\|_{\op} \left\|X - Y\right\|_{\op} \\
& \leq  \left\|X^{-1}\right\|_{\op} \left( \left\|Y^{-1} - X^{-1}\right\|_{\op} +  \left\|X^{-1}\right\|_{\op}\right) \left\|X - Y\right\|_{\op} \\
& = \left\|X^{-1}\right\|_{\op}^2 \left\|X - Y\right\|_{\op} \\
& \quad \, + \left\|X^{-1}\right\|_{\op} \left\|Y^{-1} - X^{-1}\right\|_{\op}  \left\|X - Y\right\|_{\op} 
\end{align*}
using the sub-multiplicativity of operator norms, which implies that
\begin{align*}
\left\|X^{-1} - Y^{-1}\right\|_{\op} & \leq \frac{ \left\|X^{-1}\right\|_{\op}^2 \left\|X - Y\right\|_{\op} }{1  -  \left\|X^{-1}\right\|_{\op} \left\|X - Y\right\|_{\op}} \\
& \leq \frac{ \left\|X^{-1}\right\|_{\op}^2 \left\|X - Y\right\|_{\Fro} }{1  -  \left\|X^{-1}\right\|_{\op} \left\|X - Y\right\|_{\Fro}} 
\end{align*}
as long as $\|X^{-1}\|_{\op} \|X - Y\|_{\Fro} < 1$ (also see \cite[Section 5.8]{HornJohnson2013}). Note that \eqref{Eq: step 2} only holds when $\|(G^{(i)})^{-1}\|_{\op} r L_2 < l! q^{\eta}$.

Next, we proceed to upper bounding \eqref{Eq: step 1}. Starting from \eqref{Eq: step 1}, observe that
\begin{align}
& \left| \frac{\partial F}{\partial \theta_i}(\theta) - \widehat{\nabla F}_i(\theta) \right| \nonumber \\
& \overset{\eqmakebox[B][c]{\footnotesize }}{\leq} \max_{j \in [n]\backslash [r]} \left| g_i(x^{(j)};\theta) - \sum_{k = 1}^{r} g_i(x^{(k)};\theta) w_{j,k}^{(i)} \right| \nonumber \\
& \overset{\eqmakebox[B][c]{\footnotesize (i)}}{\leq} \frac{L_2}{l! q^{\eta}} \left( 1 + \max_{j \in [n]\backslash [r]} \sum_{k = 1}^{r}{\left| w_{j,k}^{(i)}\right|} \right) \nonumber \\
& \quad \enspace \enspace + \max_{j \in [n]\backslash [r]} \left| P_{q,l}(x^{(j)};\theta) - \sum_{k = 1}^{r} P_{q,l}(x^{(k)};\theta) w_{j,k}^{(i)} \right| \nonumber \\
& \overset{\eqmakebox[B][c]{\footnotesize (ii)}}{\leq} \frac{L_2}{l! q^{\eta}} \left( 1 + \max_{j \in [n]\backslash [r]} \sum_{k = 1}^{r}{\left| w_{j,k}^{(i)}\right|} \right) \nonumber \\
& \quad \enspace \enspace + \max_{j \in [n]\backslash [r]} \left| \sum_{k = 1}^{r}{\left(a_{j,k} - w_{j,k}^{(i)} \right) P_{q,l}(x^{(k)};\theta)} \right| \nonumber \\
& \overset{\eqmakebox[B][c]{\footnotesize (iii)}}{\leq} \frac{L_2}{l! q^{\eta}} \left( 1 + \max_{j \in [n]\backslash [r]} \sum_{k = 1}^{r}{\left| w_{j,k}^{(i)}\right|} \right) \nonumber \\
& \quad \enspace \enspace + \max_{j \in [n]\backslash [r]} \left( \sum_{k = 1}^{r}{\left(a_{j,k} - w_{j,k}^{(i)} \right)^{\!2}} \right)^{\! \frac{1}{2}}\!\! \left(\sum_{k = 1}^{r}{P_{q,l}(x^{(k)};\theta)^2}\right)^{\! \frac{1}{2}} \nonumber \\
& \overset{\eqmakebox[B][c]{\footnotesize (iv)}}{\leq} \frac{L_2}{l! q^{\eta}} \left( 1 + \max_{j \in [n]\backslash [r]} \sum_{k = 1}^{r}{\left| w_{j,k}^{(i)}\right|} \right) + \left(\sum_{k = 1}^{r}{P_{q,l}(x^{(k)};\theta)^2}\right)^{\! \frac{1}{2}} \nonumber \\
& \quad \enspace \enspace \cdot \left(\rule{0cm}{0.7cm}\right. \! \frac{L_2 \sqrt{r}}{l! q^{\eta}} \left\|(G^{(i)})^{-1} \right\|_{\op} + \frac{L_2 r^{3/2}}{l! q^{\eta}} \left( B + \frac{L_2}{l! q^{\eta}} \right) \nonumber \\
& \qquad \qquad \qquad \qquad \cdot \left(\frac{ \left\|(G^{(i)})^{-1}\right\|_{\op}^2 }{1  -  \left(\left\|(G^{(i)})^{-1}\right\|_{\op} r L_2 \right)\! / \!\left(l! q^{\eta}\right) }\right) \! \left.\rule{0cm}{0.7cm}\right) \nonumber \\
& \overset{\eqmakebox[B][c]{\footnotesize (v)}}{\leq} \frac{L_2}{l! q^{\eta}} \left( 1 + \max_{j \in [n]\backslash [r]} \sum_{k = 1}^{r}{\left| w_{j,k}^{(i)}\right|} \right) + \left(B + \frac{L_2}{l! q^{\eta}} \right) \! \sqrt{r} \nonumber \\
& \quad \enspace \enspace \cdot \left(\rule{0cm}{0.7cm}\right. \! \frac{L_2 \sqrt{r}}{l! q^{\eta}} \left\|(G^{(i)})^{-1} \right\|_{\op} + \frac{L_2 r^{3/2}}{l! q^{\eta}} \left( B + \frac{L_2}{l! q^{\eta}} \right) \nonumber \\
& \qquad \qquad \qquad \qquad \cdot \left(\frac{ \left\|(G^{(i)})^{-1}\right\|_{\op}^2 }{1  -  \left(\left\|(G^{(i)})^{-1}\right\|_{\op} r L_2 \right)\! / \!\left(l! q^{\eta}\right) }\right) \! \left.\rule{0cm}{0.7cm}\right) \nonumber \\
& \overset{\eqmakebox[B][c]{\footnotesize (vi)}}{\leq} \frac{L_2}{l! q^{\eta}} \left( 1 +  \left\|(G^{(i)})^{-1} \right\|_{\op} B r \right) + \left(B + \frac{L_2}{l! q^{\eta}} \right) \nonumber \\
& \quad \enspace \enspace \cdot \left(\rule{0cm}{0.7cm}\right. \! \frac{L_2 r}{l! q^{\eta}} \left\|(G^{(i)})^{-1} \right\|_{\op} + \frac{L_2 r^{2}}{l! q^{\eta}} \left( B + \frac{L_2}{l! q^{\eta}} \right) \nonumber \\
& \qquad \qquad \qquad \qquad \cdot \left(\frac{ \left\|(G^{(i)})^{-1}\right\|_{\op}^2 }{1  -  \left(\left\|(G^{(i)})^{-1}\right\|_{\op} r L_2 \right)\! / \!\left(l! q^{\eta}\right) }\right) \! \left.\rule{0cm}{0.7cm}\right) \nonumber \\
& \overset{\eqmakebox[B][c]{\footnotesize (vii)}}{\leq} \frac{L_2}{l! q^{\eta}} \left( 1 + B r^{\alpha + 1} \right) + \left(B + \frac{L_2}{l! q^{\eta}} \right) \! \Bigg( \frac{L_2 r^{\alpha + 1}}{l! q^{\eta}}  \nonumber \\
&  \quad \enspace \enspace + \frac{L_2 r^{2 \alpha + 2}}{l! q^{\eta}} \left( B + \frac{L_2}{l! q^{\eta}} \right) \! \left(\frac{ 1 }{1  -  \left(L_2 r^{\alpha + 1} \right)\! / \!\left(l! q^{\eta}\right) }\right) \Bigg) \nonumber \\
& \overset{\eqmakebox[B][c]{\footnotesize (viii)}}{\leq} \frac{L_2 r^{2 \alpha + 2}}{l! q^{\eta}} \left( 2 + B + \frac{L_2}{l! q^{\eta}} \right)^{\! 2} \nonumber \\
& \quad \enspace \enspace \cdot \left(\frac{ 1 }{1  -  \left(L_2 r^{\alpha + 1} \right)\! / \!\left(l! q^{\eta}\right) } + 1\right) ,
\label{Eq: step 3}
\end{align}
where (i) follows from \eqref{Eq: uniform step} and the triangle inequality, (ii) follows from \eqref{Eq: Perfect low rank decomposition}, (iii) follows from the Cauchy-Schwarz inequality, (iv) utilizes \eqref{Eq: step 2}, (v) follows from the triangle inequality, \eqref{Eq: uniform step}, and the boundedness assumption in \eqref{Eq: Bounded gradient}, (vi) holds because
\begin{align*}
& \max_{j \in [n]\backslash [r]} \sum_{k = 1}^{r}{\left| w_{j,k}^{(i)}\right|} \\
& \qquad \leq \max_{j \in [n]\backslash [r]} \sqrt{r \sum_{k = 1}^{r}{\left(w_{j,k}^{(i)}\right)^{\! 2}}} \\
& \qquad \leq \max_{j \in [n]\backslash [r]} \left\|(G^{(i)})^{-1} \right\|_{\op} \sqrt{r \sum_{k = 1}^{r}{g_i(x^{(j)};\vartheta^{(k)})^{2}}} \\
& \qquad \leq \left\|(G^{(i)})^{-1} \right\|_{\op} B r 
\end{align*}
using the Cauchy-Schwarz inequality (or the equivalence of $\ell^1$ and $\ell^2$-norms), \eqref{Eq: Non-cumulative weights}, the definition of operator norm, and \eqref{Eq: Bounded gradient}, (vii) follows from the conditioning assumption in \eqref{Eq: Min singular value bound}, and (viii) follows from basic algebraic manipulations. As before, note that \eqref{Eq: step 3} only holds when $L_2 r^{\alpha + 1} < l! q^{\eta}$. 

We are finally in a position to bound $\|\widehat{\nabla F} - \nabla_{\theta} F\|_2^2$. Indeed, for any $\delta \in \big(0,\frac{1}{2}\big]$, suppose that 
\begin{equation}
\label{Eq: delta 1}
\frac{L_2 r^{2 \alpha + 2}}{l! q^{\eta}} \leq \delta \, .
\end{equation}
(Notice that $\delta \in \big(0,\frac{1}{2}\big]$ implies the desired condition $L_2 r^{\alpha + 1} < l! q^{\eta}$.) Then, we have from \eqref{Eq: step 3} that 
\begin{align*}
\left| \frac{\partial F}{\partial \theta_i}(\theta) - \widehat{\nabla F}_i(\theta) \right| & \leq \delta ( 2 + B + \delta )^{2} \! \left(\frac{ 1 }{1  - \delta } + 1\right) \\
& \leq 3 (B + 3)^2 \delta \, , 
\end{align*}
which implies that for any $\theta \in \R^p$,
\begin{equation}
\label{Eq: Control on Approximation Error of true Gradient}
\left\|\widehat{\nabla F}(\theta) - \nabla_{\theta} F(\theta)\right\|_2^2 \leq 9 (B + 3)^4 \delta^2 p \, . 
\end{equation}
Therefore, applying \lemref{Lemma: Inexact GD Bound}, we obtain
\begin{align}
& F(\theta^{(S)}) - F_* \nonumber \\
& \leq \left(1 - \frac{1}{\kappa}\right)^{\! S} \!\left(F(\theta^{(0)}) - F_*\right) \nonumber \\
& \quad \, + \frac{1}{2 L_1} \sum_{\gamma = 1}^{S}{\left(1 - \frac{1}{\kappa}\right)^{\! S-\gamma} \left\|\widehat{\nabla F}(\theta^{(\gamma-1)}) - \nabla_{\theta} F(\theta^{(\gamma-1)})\right\|_2^2} \nonumber \\
& \leq \left(1 - \frac{1}{\kappa}\right)^{\! S} \!\left(F(\theta^{(0)}) - F_*\right) \nonumber \\
& \quad \, + \frac{9 (B + 3)^4 \delta^2 p}{2 L_1} \underbrace{\sum_{\gamma = 1}^{S}{\left(1 - \frac{1}{\kappa}\right)^{\! S-\gamma}} }_{\leq \, \kappa \, = \, L_1/\mu} \nonumber \\
& \leq \left(1 - \frac{1}{\kappa}\right)^{\! S} \!\left(F(\theta^{(0)}) - F_*\right) + \frac{9 (B + 3)^4 \delta^2 p}{2 \mu} \, ,
\label{Eq: step 4}
\end{align}
where $\{\theta^{(\gamma)} \in \R^p : \gamma \in \{0,\dots,S\}\}$ are the inexact GD updates in \eqref{Eq: Inexact GD update} in \algoref{Alg: FedLRGD}, $S \in \N$ is the number of iterations for which inexact GD is run in the last epoch of \algoref{Alg: FedLRGD}, and the third inequality follows from the usual geometric series formula. We remark that to apply \lemref{Lemma: Inexact GD Bound} here, we need to verify that the empirical risk $F : \R^p \rightarrow \R$ is continuously differentiable and $\mu$-strongly convex, and its gradient $\nabla_{\theta} F$ is $L_1$-Lipschitz continuous. These properties of $F$ are in fact inherited from the corresponding strong convexity and smoothness in parameter assumptions imposed on the loss function in \secref{Assumptions on Loss Function}; we refer readers to \cite[Section 4.1]{JadbabaieMakurShah2021} for the formal arguments.

We can easily obtain the iteration complexity of the inexact GD method in \algoref{Alg: FedLRGD} from \eqref{Eq: step 4}. Letting
\begin{equation}
\label{Eq: delta 2}
\delta = \left(1 - \frac{1}{\kappa} \right)^{\! S/2} ,
\end{equation} 
we see that 
$$ F(\theta^{(S)}) - F_* \leq \left(1 - \frac{1}{\kappa}\right)^{\! S} \!\left(F(\theta^{(0)}) - F_* + \frac{9 (B + 3)^4 p}{2 \mu}\right) . $$
Hence, for \algoref{Alg: FedLRGD} to produce an $\epsilon$-approximate solution $\theta^* = \theta^{(S)}$ in the sense of \eqref{Eq: Approx solution 1}, it suffices to ensure that
$$ \left(1 - \frac{1}{\kappa}\right)^{\! S} \!\left(F(\theta^{(0)}) - F_* + \frac{9 (B + 3)^4 p}{2 \mu}\right) \leq \epsilon \, , $$
or equivalently, 
$$ S \geq \left(\log\!\left(\frac{\kappa}{\kappa-1}\right)\right)^{\! -1} \log\!\left(\frac{F(\theta^{(0)}) - F_* + \frac{9 (B + 3)^4 p}{2 \mu}}{\epsilon} \right) . $$
Thus, setting $S$ as in \eqref{Eq: Iteration count} yields an $\epsilon$-approximate solution $\theta^* = \theta^{(S)}$. Furthermore, it is straightforward to verify that $\delta \leq \frac{1}{2}$. Indeed, from \eqref{Eq: delta 2} and \eqref{Eq: Iteration count}, we have
\begin{equation}
\label{Eq: Bound on delta param}
\delta^2 \leq \frac{\epsilon}{F(\theta^{(0)}) - F_* + \frac{9 (B + 3)^4 p}{2 \mu}} \leq \frac{\epsilon}{\left(\frac{9 (B + 3)^4 p}{2 \mu}\right)} \leq \frac{1}{4} \, , 
\end{equation}
where the last inequality holds because we have assumed in the theorem statement that $\epsilon \leq \frac{9 (B+3)^4 p}{8 \mu}$.

Having computed the number of iterations \eqref{Eq: Iteration count}, we close this proof by calculating the federated oracle complexity of the $\lrgd$ algorithm. To execute this final step, we must first ensure that \eqref{Eq: delta 1} holds by choosing an appropriate value of $r \in \N$. Recall that $q$ was chosen so that $r = D |\S_{d,q}|$. Using \eqref{Eq: rank bound}, \eqref{Eq: Bound on nPr}, and the fact that $l \leq \eta$, we have
$$ r = D |\S_{d,q}| \leq d! \binom{l + d}{d} (q+1)^d \leq e \sqrt{d} (\eta + d)^d (q+1)^d \, , $$
which in turn produces the following lower bound on $q$:
$$ q \geq \frac{r^{1/d}}{e^{1/d} d^{1/(2d)} (\eta + d)} - 1 \, . $$
Since we require $q > 0$, or to be more precise, $q \geq1$ (see \secref{Latent Variable Model Approximation}), we can check to see that 
\begin{align*}
q \geq 1 & \quad \Leftarrow \quad \frac{r^{1/d}}{e^{1/d} d^{1/(2d)} (\eta + d)} \geq 2 \\
& \quad \Leftrightarrow \quad r \geq e \sqrt{d} \, 2^d (\eta + d)^d 
\end{align*}
by assumption \eqref{Eq: Rank lower bound condition} in the theorem statement. Furthermore, similar to the argument in the proof \thmref{Thm: Low Rank Approximation} in \secref{Proof of Thm Low Rank Approximation}, \eqref{Eq: Rank lower bound condition} also yields:
\begin{align*}
& r \geq e \sqrt{d} \, 2^d (\eta + d)^d \\
& \Leftrightarrow \quad q \geq \frac{r^{1/d}}{e^{1/d} d^{1/(2d)} (\eta + d)} - 1 \geq \frac{r^{1/d}}{2 e^{1/d} d^{1/(2d)} (\eta + d)} \, . 
\end{align*}
Substituting this lower bound on $q$ into the condition \eqref{Eq: delta 1}, it suffices to ensure that 
$$ \frac{L_2 2^{\eta} e^{\eta/d} d^{\eta/(2d)} (\eta + d)^{\eta}}{l! \, r^{(\eta/d) - 2 \alpha - 2}} \leq \delta \, . $$
Then, using Stirling's approximation \cite[Chapter II, Section 9, Equation (9.15)]{Feller1968} and the fact that $\eta - 1 \leq l \leq \eta$, the left hand side of the above inequality can be upper bounded by
\begin{align*}
& \frac{L_2 2^{\eta} e^{\eta/d} d^{\eta/(2d)} (\eta + d)^{\eta}}{l! \, r^{(\eta/d) - 2 \alpha - 2}} \\
& \qquad \qquad \leq \frac{L_2 \eta 2^{\eta} e^{\eta (d + 1)/d} d^{\eta/(2d)} (\eta + d)^{\eta}}{(\eta - 1)^{\eta} r^{(\eta/d) - 2 \alpha - 2}} \, . 
\end{align*}
This means that ensuring 
\begin{align*}
& \frac{L_2 \eta 2^{\eta} e^{\eta (d + 1)/d} d^{\eta/(2d)} (\eta + d)^{\eta}}{(\eta - 1)^{\eta} r^{(\eta/d) - 2 \alpha - 2}} \leq \delta \\
& \Leftrightarrow \quad r \geq \left(\frac{L_2^d \eta^{d} e^{\eta (d + 1)} d^{\eta / 2} (2\eta + 2d)^{\eta d}}{(\eta - 1)^{\eta d}}\right)^{\! 1/(\eta - (2 \alpha + 2)d)} \\
& \qquad  \quad \quad \enspace \cdot \left(\frac{1}{\delta}\right)^{\! d/(\eta - (2 \alpha + 2)d)}
\end{align*}
implies that \eqref{Eq: delta 1} is satisfied, where we utilize the assumption in the theorem statement that $\eta > (2 \alpha + 2)d$. Now, using \eqref{Eq: delta 2} and \eqref{Eq: Iteration count}, notice that
\begin{align*}
\delta & = \left(1 - \frac{1}{\kappa} \right)^{\! S/2} \\
& \geq \left(\sqrt{\frac{\kappa - 1}{\kappa}}\right)^{\! 1 + \frac{\log\left(\left(F(\theta^{(0)}) - F_* + \frac{9 (B + 3)^4 p}{2 \mu}\right) \epsilon^{-1} \right)}{\log\left(\frac{\kappa}{\kappa-1}\right)}} \\
& = \sqrt{\frac{(\kappa - 1)\epsilon}{\kappa \! \left(F(\theta^{(0)}) - F_* + \frac{9 (B + 3)^4 p}{2 \mu}\right)}} \, . 
\end{align*}
Hence, to satisfy \eqref{Eq: delta 1}, it suffices to ensure that
\begin{align*}
r & \geq \left(\frac{L_2^d \eta^{d} e^{\eta (d + 1)} d^{\eta / 2} (2\eta + 2d)^{\eta d}}{(\eta - 1)^{\eta d}}\right)^{\! 1/(\eta - (2 \alpha + 2)d)} \\
& \quad \enspace \cdot \left(\frac{\kappa \! \left(F(\theta^{(0)}) - F_* + \frac{9 (B + 3)^4 p}{2 \mu}\right)}{(\kappa - 1)\epsilon}\right)^{\! d/(2\eta - (4 \alpha + 4)d)} , 
\end{align*}
which is already assumed in \eqref{Eq: Rank lower bound condition}. Finally, for $r$ and $S$ given by \eqref{Eq: Rank lower bound condition} and \eqref{Eq: Iteration count}, respectively, observe using Definition \ref{Def: Federated oracle complexity} that the federated oracle complexity of \algoref{Alg: FedLRGD} is 
\begin{align*}
\Gamma(\lrgd) & = \underbrace{r^2}_{\text{epoch } 1} \! + \,\, \underbrace{rs + \phi m}_{\text{epoch } 2} \,\, + \underbrace{(r-1) \phi m}_{\text{epochs } 3 \text{ to } r+1} + \!\! \underbrace{r S}_{\text{epoch } r+2} \\
& = r^2 + rs + rS + \phi m r \, .
\end{align*}
This completes the proof.
\end{proof}

\section{Conclusion}
\label{Conclusion}

In closing, we summarize the main technical contributions of this work. Firstly, as a counterpart to traditional oracle complexity of iterative optimization methods, we provided a formalism for measuring running time of federated optimization algorithms by presenting the notion of federated oracle complexity in \secref{Federated Oracle Complexity}. This notion afforded us the ability to theoretically compare the performance of different algorithms. Secondly, in order to exploit the smoothness of loss functions in training data\textemdash a generally unexplored direction in the optimization for machine learning literature \cite{JadbabaieMakurShah2021}, we proposed the $\lrgd$ algorithm for federated learning in \secref{Federated Low Rank Gradient Descent}. This algorithm crucially used the approximate low rankness induced by the smoothness of gradients of loss functions in data. Thirdly, we analyzed the federated oracle complexity of $\lrgd$ in \thmref{Thm: Federated Oracle Complexity of FedLRGD} under various assumptions, including strong convexity, smoothness in parameter, smoothness in data, non-singularity, etc. In particular, we proved that $\Gamma(\lrgd)$ scales like $\phi m (p/\epsilon)^{\Theta(d/\eta)}$ (neglecting typically sub-dominant factors). Moreover, to compare $\lrgd$ to the standard $\FedAve$ method in the literature \cite{McMahanetal2017}, we also evaluated the federated oracle complexity of $\FedAve$ in \thmref{Thm: Federated Oracle Complexity of FedAve} and saw that $\Gamma(\FedAve)$ scales like $\phi m (p/\epsilon)^{3/4}$. Then, we demonstrated that when data dimension is small and the loss function is sufficiently smooth in the data, $\Gamma(\lrgd)$ is indeed smaller that $\Gamma(\FedAve)$. Finally, in a complementary direction, we built upon aspects of our analysis of $\lrgd$ to generalize a well-known result from high-dimensional statistics regarding low rank approximations of smooth latent variable models in \thmref{Thm: Low Rank Approximation}. 

We also mention a few future research directions to conclude. It is natural to try and weaken some of the assumptions made to derive the technical results in this work, cf. \secref{Assumptions on Loss Function}. For example, the strong convexity assumption could potentially be weakened to convex or even non-convex loss functions. As another example, the non-singularity and conditioning assumptions could perhaps be replaced by simpler assumptions on the gradient of the loss function, e.g., by developing variants of positive definite kernel ideas and assuming these on the gradient perceived as a bivariate kernel. The non-singularity and conditioning assumptions could then be deduced from such assumptions. In a different vein, we also leave the thorough empirical evaluation of practical versions of $\lrgd$ and their comparisons to known federated learning methods, such as $\FedAve$, in various senses (including federated oracle complexity and wall-clock running time) as future work.

\appendices

\section{Smoothness Assumptions for Logistic Regression with Soft Labels}
\label{App: Logistic Example}

In this appendix, we delineate an example of a commonly used loss function in logistic regression satisfying the important smoothness and strong convexity assumptions in \secref{Assumptions on Loss Function}. Let the parameter space be $[-1,1]^{d-1}$ so that any parameter vector $\theta \in [-1,1]^{d-1}$, and each labeled training data sample $(x,y) \in [0,1]^{d-1} \times [0,1]$ be such that the label $y$ is a ``soft'' belief. The compactness of the parameter space is reasonable since training usually takes place in a compact subset of the Euclidean space $\R^{d-1}$, and the use of soft labels instead of canonical binary classification labels has also been studied in the literature, cf. \cite{NguyenValizadeganHauskrecht2011}. Consider the $\ell^2$-regularized \emph{cross-entropy} loss function:
\begin{equation}
\label{Eq: Reg Cross Entropy}
\begin{aligned}
f_{\mathsf{ent}}(x,y;\theta) & \triangleq y \log\!\left(1 + e^{- \theta^{\T} x}\right) \\
& \quad \, + (1-y) \log\!\left(1 + e^{\theta^{\T} x}\right) + \frac{\mu}{2} \left\|\theta\right\|_2^2 \, ,
\end{aligned}
\end{equation}
where $\mu > 0$ is a hyper-parameter that determines the level of regularization, and we assume for convenience that the bias parameter in logistic regression is $0$. Given a training dataset $(x^{(1)},y_1),\dots,(x^{(n)},y_n) \in [0,1]^{d-1} \times  [0,1]$, the associated empirical risk is
\begin{equation}
\begin{aligned}
F(\theta) & = \frac{1}{n} \sum_{i = 1}^{n}{f_{\mathsf{ent}}(x^{(i)},y_i;\theta)} \\
& = \frac{1}{n} \sum_{i = 1}^{n} \bigg( y_i \log\!\left(1 + e^{- \theta^{\T} x^{(i)}}\right) \\
& \qquad \qquad \enspace + (1-y_i) \log\!\left(1 + e^{\theta^{\T} x^{(i)}}\right)\!\bigg) \\
& \quad \, + \frac{\mu}{2} \left\|\theta\right\|_2^2 \, ,
\end{aligned}
\end{equation}
which corresponds to the well-known setting of $\ell^2$-regularized \emph{logistic regression} \cite[Section 4.4]{HastieTibshiraniFriedman2009}. We next verify that $f_{\mathsf{ent}}$ in \eqref{Eq: Reg Cross Entropy} satisfies the smoothness and strong convexity assumptions in \secref{Assumptions on Loss Function}:
\begin{enumerate}
\item Observe that $\nabla_{\theta} f_{\mathsf{ent}}(x,y;\theta) = \big(\frac{(1-y) }{1 + e^{-\theta^{\T} x}} - \frac{ y }{1 + e^{\theta^{\T} x}} \big) x + \mu \theta$, which means that for all $x,y$ and all $\theta^{(1)},\theta^{(2)} \in \R^{d-1}$, 
\begin{align}
& \left\|\nabla_{\theta} f_{\mathsf{ent}}(x,y;\theta^{(1)}) - \nabla_{\theta} f_{\mathsf{ent}}(x,y;\theta^{(2)})\right\|_2 \nonumber \\
& \quad \leq \left\| x \right\|_2 \left|\frac{(1-y) }{1 + e^{-{\theta^{(1)}}^{\T} x}} - \frac{ y }{1 + e^{{\theta^{(1)}}^{\T} x}} \right. \nonumber \\
& \quad \qquad \quad \enspace \left. - \frac{(1-y) }{1 + e^{-{\theta^{(2)}}^{\T} x}} + \frac{ y }{1 + e^{{\theta^{(2)}}^{\T} x}} \right| \nonumber \\
& \quad \quad \, + \mu \left\|\theta^{(1)} - \theta^{(2)} \right\|_2 \nonumber \\
& \quad \leq \left\| x \right\|_2 \bigg(\left|\frac{1 }{1 + e^{-{\theta^{(1)}}^{\T} x}} - \frac{1}{1 + e^{-{\theta^{(2)}}^{\T} x}} \right| \nonumber \\
& \quad \qquad \quad \quad \enspace + \left|\frac{1 }{1 + e^{{\theta^{(2)}}^{\T} x}} - \frac{1 }{1 + e^{{\theta^{(1)}}^{\T} x}}\right| \bigg) \nonumber \\
& \quad \quad \, + \mu \left\|\theta^{(1)} - \theta^{(2)} \right\|_2 \nonumber \\
& \quad \leq \frac{\left\| x \right\|_2}{2} \left|(\theta^{(1)} - \theta^{(2)})^{\T} x \right| + \mu \left\|\theta^{(1)} - \theta^{(2)} \right\|_2 \nonumber \\
& \quad \leq \frac{\left\| x \right\|_2^2}{2} \left\|\theta^{(1)} - \theta^{(2)}\right\|_2 + \mu \left\|\theta^{(1)} - \theta^{(2)} \right\|_2 \nonumber \\
& \quad \leq \left(\frac{d-1}{2} + \mu\right) \! \left\|\theta^{(1)} - \theta^{(2)}\right\|_2 \, ,
\end{align}
where the first inequality follows from the triangle inequality, the second inequality holds because $\max\{|y|,|1-y|\} \leq 1$, the third inequality holds due to the Lipschitz continuity of the sigmoid function and the fact that $\max_{t \in \R}{\big|\frac{\diff}{\diff t}(1 + e^{-t})^{-1}\big|} \allowbreak = \frac{1}{4}$, the fourth inequality follows from the Cauchy-Schwarz inequality, and the final inequality follows from the bound $\| x \|_2^2 \leq d-1$. Hence, the Lipschitz constant $L_1 = \frac{d-1}{2} + \mu$.
\item It is well-known that the (unregularized) cross-entropy loss function $\theta \mapsto y \log\big(1 + e^{- \theta^{\T} x}\big) + (1-y) \log\big(1 + e^{\theta^{\T} x}\big)$ is convex for all $x,y$ \cite[Section 4.4]{HastieTibshiraniFriedman2009}. This can be directly checked by computing derivatives. This implies that $\theta \mapsto f_{\mathsf{ent}}(x,y;\theta)$ is $\mu$-strongly convex for all $x,y$ \cite[Section 3.4]{Bubeck2015}.
\item Observe that for any fixed $\vartheta \in [-1,1]^{d-1}$ and any $i \in [d-1]$, the $i$th partial derivative $g_i(x,y;\vartheta) = \big( \frac{(1-y) }{1 + e^{-\vartheta^{\T} x}} - \frac{ y }{1 + e^{\vartheta^{\T} x}} \big) x_i + \mu \vartheta_i$. Let us arbitrarily choose $\eta = 2$ for the purposes of illustration. Then, for any $s = (0,\dots,0,1,0,\dots,0) \in \Z_+^d$ with a value of $1$ at the $j$th position and $|s| = \eta - 1 = 1$, we have for all $x \in [0,1]^{d-1}$ and $y \in [0,1]$, 
\begin{equation}
\begin{aligned}
& \nabla_{x,y}^s g_i(x,y;\vartheta) \\
& = 
\begin{cases}
\frac{\partial^2 f_{\mathsf{ent}}}{\partial \theta_i \partial x_i} (x,y;\vartheta)  \, , & j = i  \\
\frac{\partial^2 f_{\mathsf{ent}}}{\partial \theta_i \partial x_j} (x,y;\vartheta) \, , & j \in [d-1], \, j \neq i  \\
\frac{\partial^2 f_{\mathsf{ent}}}{\partial \theta_i \partial y} (x,y;\vartheta) \, , & j = d 
\end{cases}
\\
& = 
\begin{cases}
\frac{\vartheta_i x_i e^{-\vartheta^{\T} x}}{\left(1 + e^{-\vartheta^{\T} x}\right)^2} + \frac{1}{1 + e^{-\vartheta^{\T} x}} - y  \, , & j = i  \\
\frac{\vartheta_j x_i e^{-\vartheta^{\T} x}}{\left(1 + e^{-\vartheta^{\T} x}\right)^2} \, , & j \in [d-1], \, j \neq i  \\
- x_i \, , & j = d 
\end{cases}
.
\end{aligned}
\end{equation}
Hence, for all $x^{(1)},x^{(2)} \in [0,1]^{d-1}$ and $y_1,y_2 \in [0,1]$, we obtain
\begin{align}
& \left|\nabla_{x,y}^s g_i(x^{(1)},y_1;\vartheta) - \nabla_{x,y}^s g_i(x^{(2)},y_2;\vartheta)\right| \leq \nonumber \\
& \begin{cases}
\frac{3}{4} \left\|x^{(1)} - x^{(2)}\right\|_1 + \left|y_1 - y_2 \right| , & j = i \\
\frac{1}{2} \left\|x^{(1)} - x^{(2)}\right\|_1 , & j \in [d-1], \, j \neq i \\
 \left|x_i^{(1)} - x_i^{(2)} \right| , & j = d
\end{cases}
,
\end{align}
where the third case follows from direct calculation, the first case follows from a similar argument to the second case, and the second case holds because
\begin{align}
& \left|\nabla_{x,y}^s g_i(x^{(1)},y_1;\vartheta) - \nabla_{x,y}^s g_i(x^{(2)},y_2;\vartheta)\right| \nonumber \\
& \quad = \left|  \frac{\vartheta_j x_i^{(1)} e^{-\vartheta^{\T} x^{(1)}}}{\left(1 + e^{-\vartheta^{\T} x^{(1)}}\right)^2} -  \frac{\vartheta_j x_i^{(2)} e^{-\vartheta^{\T} x^{(2)}}}{\left(1 + e^{-\vartheta^{\T} x^{(2)}}\right)^2}\right| \nonumber \\
& \quad \leq \left|  \frac{ x_i^{(1)} e^{-\vartheta^{\T} x^{(1)}}}{\left(1 + e^{-\vartheta^{\T} x^{(1)}}\right)^2} -  \frac{x_i^{(2)} e^{-\vartheta^{\T} x^{(2)}}}{\left(1 + e^{-\vartheta^{\T} x^{(2)}}\right)^2}\right| \nonumber \\
& \quad \leq \left|  \frac{e^{-\vartheta^{\T} x^{(1)}}}{\left(1 + e^{-\vartheta^{\T} x^{(1)}}\right)^2} -  \frac{e^{-\vartheta^{\T} x^{(2)}}}{\left(1 + e^{-\vartheta^{\T} x^{(2)}}\right)^2}\right| \nonumber \\
& \quad \quad \, + \frac{e^{-\vartheta^{\T} x^{(2)}}}{\left(1 + e^{-\vartheta^{\T} x^{(2)}}\right)^2} \left| x_i^{(1)} - x_i^{(2)}\right| \nonumber \\
& \quad \leq \left|  \frac{e^{-\vartheta^{\T} x^{(1)}}}{\left(1 + e^{-\vartheta^{\T} x^{(1)}}\right)^2} -  \frac{e^{-\vartheta^{\T} x^{(2)}}}{\left(1 + e^{-\vartheta^{\T} x^{(2)}}\right)^2}\right| \nonumber \\
& \quad \quad \, + \frac{1}{4} \left| x_i^{(1)} - x_i^{(2)}\right| \nonumber \\
& \quad \leq \frac{1}{4} \left| \vartheta^{\T} \!\left(x^{(1)} - x^{(2)} \right) \right| + \frac{1}{4} \left| x_i^{(1)} - x_i^{(2)}\right| \nonumber \\
& \quad \leq \frac{1}{4} \left\|x^{(1)} - x^{(2)}\right\|_1 + \frac{1}{4} \left| x_i^{(1)} - x_i^{(2)}\right| \nonumber \\
& \quad \leq \frac{1}{2} \left\|x^{(1)} - x^{(2)}\right\|_1 ,
\end{align}
where the second line follows from the fact that $|\vartheta_j| \leq 1$, the third line follows from the triangle inequality and the fact that $x_i^{(1)} \in [0,1]$, the fourth line holds because $e^{-t}(1+e^{-t})^{-2} \leq \frac{1}{4}$ for all $t \in \R$, the fifth line follows from Lipschitz continuity and the fact that $\max_{t \in \R}{\big|\frac{\diff}{\diff t}e^{-t}(1+e^{-t})^{-2}\big|} \leq \frac{1}{4}$, and the sixth line follows from H\"{o}lder's inequality and the fact that $\|\vartheta\|_{\infty} \leq 1$. This implies that for all $x^{(1)},x^{(2)} \in [0,1]^{d-1}$ and $y_1,y_2 \in [0,1]$, we have
\begin{equation}
\begin{aligned}
& \left|\nabla_{x,y}^s g_i(x^{(1)},y_1;\vartheta) - \nabla_{x,y}^s g_i(x^{(2)},y_2;\vartheta)\right| \\
& \qquad \qquad \qquad \quad \leq \left\|(x^{(1)},y_1) - (x^{(2)},y_2)\right\|_1 .
\end{aligned}
\end{equation}
Thus, $L_2 = 1$ in this example, and each $g_i(\cdot;\vartheta)$ belongs to a $(2,1)$-H\"{o}lder class.
\end{enumerate}

\section{Bounds on Erlang Quantile Function}
\label{Bounds on Erlang Quantile Function}

For any $b \in \N$, let $Y$ be an Erlang distributed random variable with shape parameter $b$, rate $1$, and CDF $F_Y : [0,\infty) \rightarrow [0,1]$ given by \cite[Example 3.3.21]{EmbrechtsKluppelbergMikosch1997}:
\begin{equation}
\label{Eq: Erlang CDF}
\forall y \geq 0, \enspace F_Y(y) = 1 - e^{-y} \sum_{k = 0}^{b-1}{\frac{y^k}{k!}} \, .
\end{equation}
The ensuing proposition provides bounds on the quantile function (or inverse CDF) $F_Y^{-1}:[0,1) \rightarrow [0,\infty)$ in the neighborhood of unity.

\begin{proposition}[Quantile Function of Erlang Distribution]
\label{Prop: Quantile Function of Erlang Distribution}
For any integer $q \in \N$ such that $(q-1)b \geq 55$, we have
\begin{align*}
\frac{1}{2} \log(q-1) + \frac{1}{2} \log(b) & \leq F_Y^{-1}\!\left(1 - \frac{1}{q}\right) \\
& \leq 2\log(q-1) + 2b \log(2b) \, . 
\end{align*}
\end{proposition}

\begin{proof}
Recall from the definition of generalized inverse that
\begin{align}
& F_Y^{-1}\!\left(1 - \frac{1}{q}\right) \nonumber \\
& = \inf\!\left\{y > 0 : F_Y(y) \geq 1 - \frac{1}{q} \right\} \nonumber \\
& = \inf\!\left\{y > 0 : \sum_{k = 0}^{b-1}{\frac{y^k}{k!}} \leq \frac{1}{q} \sum_{k = 0}^{\infty}{\frac{y^k}{k!}} \right\} \nonumber \\
& = \inf\!\left\{y > 0 : (q-1) \sum_{k = 0}^{b-1}{\frac{y^k}{k!}} \leq \sum_{k = b}^{\infty}{\frac{y^k}{k!}} \right\} \nonumber \\
& = \inf\!\left\{y > 0 : (q-1) \sum_{k = 0}^{b-1}{\frac{y^k}{k!}} \leq \frac{y^b}{b!} \sum_{k = 0}^{\infty}{\frac{y^{k}}{k!} \binom{k+b}{b}^{\! -1}} \right\} \nonumber \\
& = \inf\!\left\{y > 0 : b! (q-1) \sum_{k = 0}^{b-1}{\frac{y^{k-b}}{k!}} \leq \sum_{k = 0}^{\infty}{\frac{y^{k}}{k!} \binom{k+b}{b}^{\! -1}} \right\} ,
\label{Eq: Generalized inverse}
\end{align}
where the second equality uses \eqref{Eq: Erlang CDF} and the Maclaurin series of $e^y = \sum_{k = 0}^{\infty}{\frac{y^{k}}{k!}}$, and the fourth equality follows from straightforward manipulations of the summation on the right hand side.

To prove the lower bound, notice that for all $y > 0$,
\begin{align*}
b! (q-1) \sum_{k = 0}^{b-1}{\frac{y^{k-b}}{k!}} & \leq \sum_{k = 0}^{\infty}{\frac{y^{k}}{k!} \binom{k+b}{b}^{\! -1}} \\
\Rightarrow \qquad b! (q-1) \sum_{k = 0}^{b-1}{\frac{y^{k-b}}{k!}} & \leq \sum_{k = 0}^{\infty}{\frac{y^{k}}{k!}} = e^y \\ 
\Rightarrow \qquad b(q-1) & \leq y e^y \\
\Leftrightarrow \qquad \log(b) + \log(q-1) & \leq \log(y) + y \\
\Rightarrow \qquad \frac{1}{2} \log(b) + \frac{1}{2} \log(q-1) & \leq y \, ,
\end{align*}
where the first implication holds because $\binom{k+b}{b} \geq 1$, the second implication follows from only retaining the $k = b-1$ term in the summation on the left hand side, and the last implication holds because $\log(y) \leq y - 1 \leq y$. Hence, using \eqref{Eq: Generalized inverse}, we obtain
\begin{align*}
F_Y^{-1}\!\left(1 - \frac{1}{q}\right) & \geq \inf\!\left\{y \geq \frac{1}{2} \log(q-1) + \frac{1}{2} \log(b) \right\} \\
& = \frac{1}{2} \log(q-1) + \frac{1}{2} \log(b) \, .
\end{align*}

Next, note that since we assume that $(q-1)b \geq 55 \geq e^4$, $F_Y^{-1}(1 - q^{-1}) \geq 2$. So, it suffices to consider $y \geq 2$ within the infimum in \eqref{Eq: Generalized inverse}. To establish the upper bound, observe that for all $y \geq 2$, 
\begin{align*}
b! (q-1) \sum_{k = 0}^{b-1}{\frac{y^{k-b}}{k!}} & \leq \sum_{k = 0}^{\infty}{\frac{y^{k}}{k!} \binom{k+b}{b}^{\! -1}} \\
\Leftarrow \qquad 2^b b! (q-1) \sum_{k = 0}^{b-1}{\frac{y^{k-b}}{k!}} & \leq \sum_{k = 0}^{\infty}{\frac{(y/2)^{k}}{k!}} = e^{y/2} \\
\Leftarrow \qquad \frac{2^b b! (q-1)}{y} \sum_{k = 0}^{b-1}{\frac{1}{y^{k}}} & \leq e^{y/2} \\
\Leftarrow \qquad \frac{2^b b! (q-1)}{y - 1} & \leq e^{y/2} \\
\Leftrightarrow \qquad b \log(2) + \log(b!) + \log(q-1) & \leq \log(y-1) + \frac{y}{2} \\ 
\Leftarrow \qquad 2 b \log(2 b) + 2 \log(q-1) & \leq y \, ,
\end{align*}
where the first implication holds because $\binom{k+b}{b} \leq 2^{k+b}$, the third implication follows from computing the geometric series on the left hand side in the previous line over all $k \in \N\cup\!\{0\}$, and the last implication holds because $b! \leq b^b$ and $\log(y-1) \geq 0$. Therefore, using \eqref{Eq: Generalized inverse}, we get
\begin{align*}
F_Y^{-1}\!\left(1 - \frac{1}{q}\right) & \leq \inf\!\left\{y \geq 2 \log(q-1) + 2 b \log(2 b) \right\} \\
& = 2 \log(q-1) + 2 b \log(2 b) \, . 
\end{align*}
This completes the proof.
\end{proof}

\section{Low Rank Structure in Neural Networks}
\label{Low Rank Structure in Neural Networks}

In this appendix, we present some Keras (TensorFlow) based simulation results which provide evidence that a variant of the approximate low rank structure used in our analysis of the $\lrgd$ algorithm appears in neural network training problems. We perform our experiments using the well-known \emph{MNIST database} for handwritten digit recognition \cite{LeCunCortesBurgesMNIST} as well as the \emph{CIFAR-10 dataset} of tiny images \cite{KrizhevskyNairHintonCIFAR10}. We next describe our experiment for both databases. 

\textbf{MNIST training.} The MNIST database consists of a training dataset of $n = 60,000$ image-label pairs as well as a test dataset of $10,000$ pairs. Each handwritten digit image has $28 \times 28$ grayscale pixels taking intensity values in $\{0,\dots,255\}$, which we can stack and re-normalize to obtain a feature vector $y \in [0,1]^{784}$. Each label belongs to the set of digits $\{0,\dots,9\}$, which we transform into an elementary basis vector $z \in \{0,1\}^{10}$ via \emph{one hot encoding}. So, each data sample $(y,z)$ has dimension $d = 794$. 

To train a classifier that solves the multi-class classification problem of inferring digits based on images, we construct a \emph{fully connected feed-forward neural network with one hidden layer} (cf. \cite[Chapter 5]{Bishop2006} for basic definitions and terminology). The input layer of this network has $784$ nodes (since the feature vectors are $784$-dimensional), the hidden layer has $K \in \N$ nodes (we typically set $K = 784$ as well), and the output layer has $10$ nodes (since there are $10$ possible label values). The network also possesses the usual bias nodes, which are not included in the above counts. We use all \emph{logistic} (or sigmoid) activation functions (and in one case, all \emph{rectified linear unit}, i.e., ReLU, activation functions) in the hidden layer, and the \emph{softmax} activation function in the output layer. Note that these neural networks are highly overparametrized, because $d \ll n \ll p = 623,290$, where $p$ is the number of weight parameters in the network including all bias nodes. Finally, we use the \emph{cross-entropy} loss function, and minimize the empirical risk (with respect to the training data) over the weights of the neural network using the \emph{Adam} stochastic optimization algorithm \cite{KingmaBa2015} (and in one case, using the \emph{root mean square propagation}, i.e., RMSprop, algorithm \cite{TielemanHinton2012}). Specifically, we always use a batch size of $200$ and run the algorithm over $25$ epochs. In almost all instances in our experiments, this setting of hyper-parameters for the optimization algorithm yields test validation accuracies of over $98\%$ (and training accuracies of over $99\%$). (We note that perturbing the batch size or number of epochs does not change our qualitative observations.)

\textbf{CIFAR-10 training.} The CIFAR-10 database consists of a training dataset of $n = 50,000$ image-label pairs as well as a test dataset of $10,000$ pairs. Each image has $32 \times 32$ color pixels with $3$ channels per pixel taking intensity values in $\{0,\dots,255\}$, which we can stack and re-normalize to obtain a feature vector $y \in [0,1]^{3072}$. Each label belongs to a set of $10$ possible values, e.g., airplane, bird, horse, truck, etc., which we transform into an elementary basis vector $z \in \{0,1\}^{10}$ via one hot encoding. So, each data sample $(y,z)$ has dimension $d = 3082$. 

This time, to train a classifier that solves the problem of inferring labels based on images, we construct a \emph{convolutional neural network} (cf. \cite[Chapter 5]{Bishop2006}). The input layer of this network has $32 \times 32 \times 3$ nodes (since there are $3$ channels), the next two hidden convolutional layers use $40$ \emph{filters} per input channel each with $3 \times 3$ kernels with a stride of $1$ and zero padding (to retain the image dimensions of $32 \times 32$), the fourth layer performs sub-sampling via $2 \times 2$ \emph{max-pooling} with a stride of $2$, the fifth layer is fully connected with $150$ output nodes, and the sixth output layer is also fully connected and has $10$ nodes (since there are $10$ possible label values). The network also possesses the usual bias nodes, which are not included in the above counts of nodes. We use all logistic activation functions (and in one case, all ReLU activation functions) in the two convolutional hidden layers and the fifth fully connected layer, and the softmax activation function in the output layer. As before, these deep neural networks are highly overparametrized, because $d \ll n \ll p = 1,553,220$, where $p$ is the number of weight parameters in the network. Finally, we use the cross-entropy loss function, and minimize the empirical risk over the weights of the neural network using the Adam stochastic optimization algorithm (and in one case, using the RMSprop algorithm) as before. Specifically, we always use a batch size of $200$ and run the algorithm over $15$ epochs. In almost all instances in our experiments, this setting of hyper-parameters for the optimization algorithm yields test validation accuracies of around $60\%$ (with slightly higher test accuracies for ReLU activations). (On the other hand, training accuracies are usually higher, e.g., over $98\%$ for ReLU activations. Although these accuracies can be further increased through careful tuning using techniques like dropout for regularization and batch normalization, we do not utilize such ideas here for simplicity. We are able to illustrate low rank structure without these additional bells and whistles. Furthermore, we note that perturbing the batch size, number of epochs, the architecture of hidden layers, etc. does not change our qualitative observations.)

\textbf{Experimental details.} In either the MNIST or the CIFAR-10 settings delineated above, let $\theta^* \in \R^p$ be the optimal network weights output by the training process. Then, for any $k \in \{30,60,90\}$, we randomly choose $k$ test data samples $\{x^{(i)} = (y^{(i)},z^{(i)}) \in [0,1]^{d} : i \in [k]\}$ uniformly among all subsets of $k$ test data samples. Moreover, we fix any $k$ independent random perturbations of $\theta^*$, namely, vectors of network weights $\theta^{(1)},\dots,\theta^{(k)} \in \R^p$ such that
\begin{equation}
\theta^{(i)} = \theta^* + \frac{\left\|\theta^*\right\|_{1}}{p} \, g_i
\end{equation}
for all $i \in \{1,\dots,k\}$, where $g_1,\dots,g_k$ are i.i.d. Gaussian distributed random vectors with zero mean and identity covariance. We next construct a $k \times k \times p$ \emph{tensor $\M \in \R^{k \times k \times p}$ of partial derivatives}:
\begin{equation}
\begin{aligned}
& \forall i , j \in \{1,\dots,k\}, \, \forall q \in \{1,\dots,p\}, \\
& \qquad \qquad \qquad \qquad \M_{i,j,q} \triangleq \frac{\partial f}{\partial \theta_q}(x^{(i)};\theta^{(j)}) \, ,
\end{aligned}
\end{equation}
where $f : [0,1]^d \times \R^p \rightarrow \R$, $f(x;\theta)$ denotes the cross-entropy loss for a data sample $x$ when the neural network has weights $\theta$, the first dimension of $\M$ is indexed by $x^{(1)},\dots,x^{(k)}$, the second dimension is indexed by $\theta^{(1)},\dots,\theta^{(k)}$, and the third dimension is indexed by the network weight coordinates $\theta_1,\dots,\theta_p$. Note that the partial derivatives to obtain each $\M_{i,j,q}$ can be computed in a variety of ways; we use one iteration of a standard GD procedure in Keras with a single data sample. For any $q \in [p]$, let $\M_q \in \R^{k \times k}$ be the $q$th matrix slice of the tensor $\M$ (where the rows of $\M_q$ are indexed by  $x^{(1)},\dots,x^{(k)}$ and the columns are indexed by $\theta^{(1)},\dots,\theta^{(k)}$). Furthermore, letting $\sigma_1(\M_q) \geq \sigma_2(\M_q) \geq \cdots \geq \sigma_k(\M_q) \geq 0$ denote the ordered singular values of $\M_q$, we define the \emph{approximate rank} of the (non-zero) matrix $\M_q$ as the minimum number of ordered singular values that capture $90\%$ of the squared Frobenius norm of $\M_q$:
\begin{equation}
\begin{aligned}
& \text{approximate rank of } \M_q \\
& \qquad \triangleq \min\!\left\{J \in [k] : \sum_{i = 1}^{J}{\sigma_i(\M_q)^2} \geq 0.9 \left\|\M_q\right\|_{\Fro}^2 \right\} .
\end{aligned}
\end{equation} 
(Note that the choice of $90\%$ is arbitrary, and our qualitative observations do not change if we perturb this value.) Next, we fix a sufficiently small $\varepsilon > 0$; specifically, we use $\varepsilon = 0.005$ after eyeballing several instances of optimal weight vectors $\theta^*$ to determine a small enough threshold to rule out zero weights (although the precise choice of this value does not alter our qualitative observations). Ruling out zero weights is a stylistic choice in our experiments since these weights are effectively not used by the final trained model. However, keeping them in the sequel would not change the depicted low rank structure. Then, for every $q \in [p]$ such that $\theta^*_q$ is ``non-trivial,'' by which we mean $|\theta_q^*| \geq \varepsilon$, we compute the approximate rank of $\M_q$. We finally plot a \emph{histogram} of this multiset of approximate rank values, where the multiset contains one value for each $q \in [p]$ such that $\theta^*_q$ is non-trivial.\footnote{To be precise, for computational efficiency, we plot histograms of around $500$ randomly and uniformly chosen approximate rank values from this multiset.} 

\begin{figure*}[t]
\begin{subfigure}{.33\linewidth}
\centering
  
\includegraphics[trim = 0mm 0mm 0mm 0mm, clip, width=\linewidth]{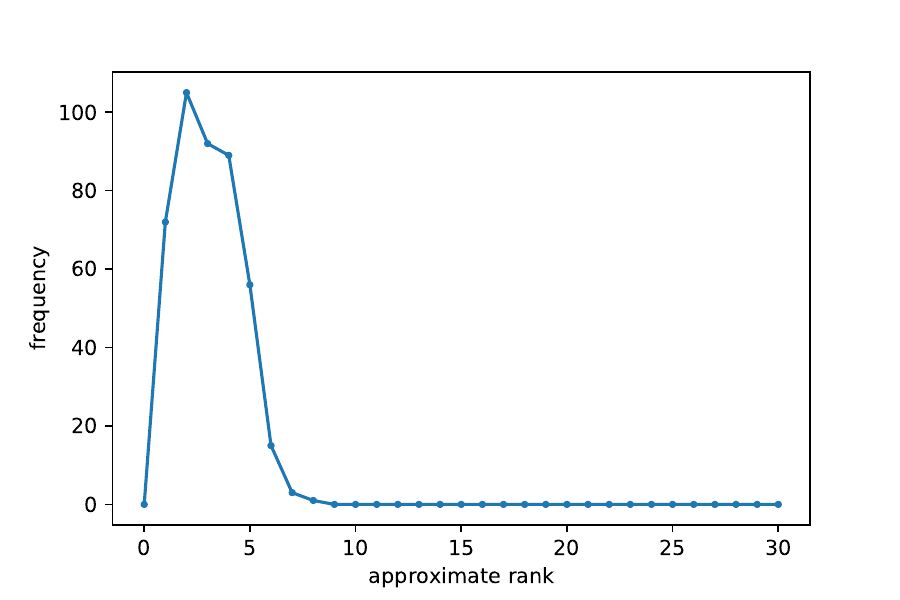}

\caption{$K = 784$, $k = 30$}  
\label{Fig: 784, logistic, 30}
\end{subfigure}
\begin{subfigure}{.33\linewidth}
\centering
  
\includegraphics[trim = 0mm 0mm 0mm 0mm, clip, width=\linewidth]{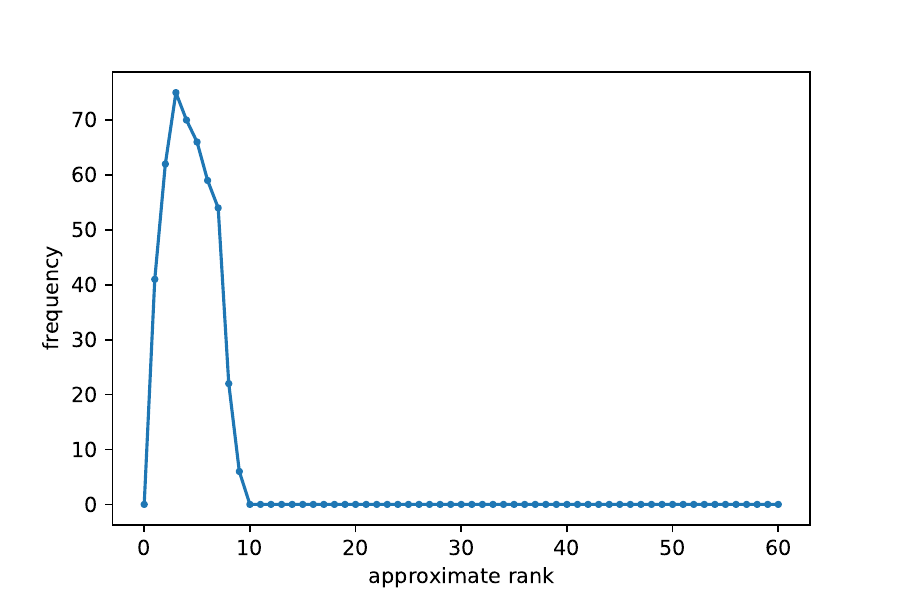}

\caption{$K = 784$, $k = 60$}  
\label{Fig: 784, logistic, 60}
\end{subfigure}
\begin{subfigure}{.33\linewidth}
\centering
  
\includegraphics[trim = 0mm 0mm 0mm 0mm, clip, width=\linewidth]{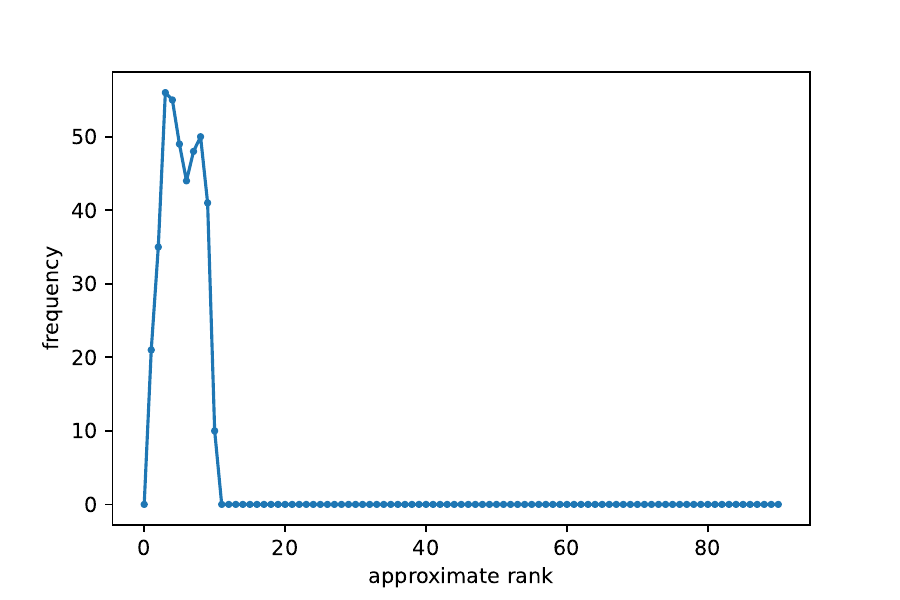}

\caption{$K = 784$, $k = 90$}  
\label{Fig: 784, logistic, 90}
\end{subfigure}
\newline
\begin{subfigure}{.33\linewidth}
\centering
  
\includegraphics[trim = 0mm 0mm 0mm 0mm, clip, width=\linewidth]{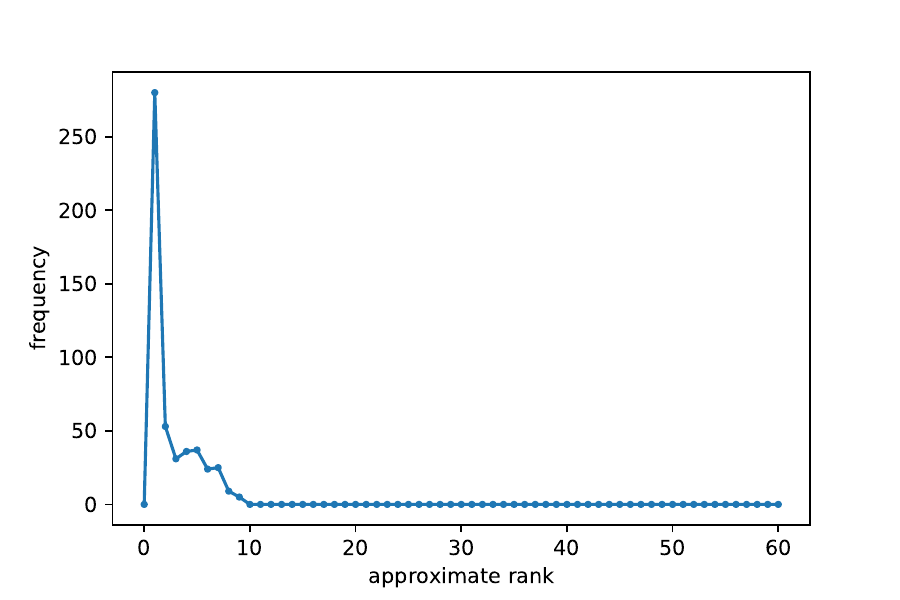}

\caption{$K = 784$, $k = 60$, all weights}  
\label{Fig: 784, logistic, 60, all weights}
\end{subfigure}
\begin{subfigure}{.33\linewidth}
\centering
  
\includegraphics[trim = 0mm 0mm 0mm 0mm, clip, width=\linewidth]{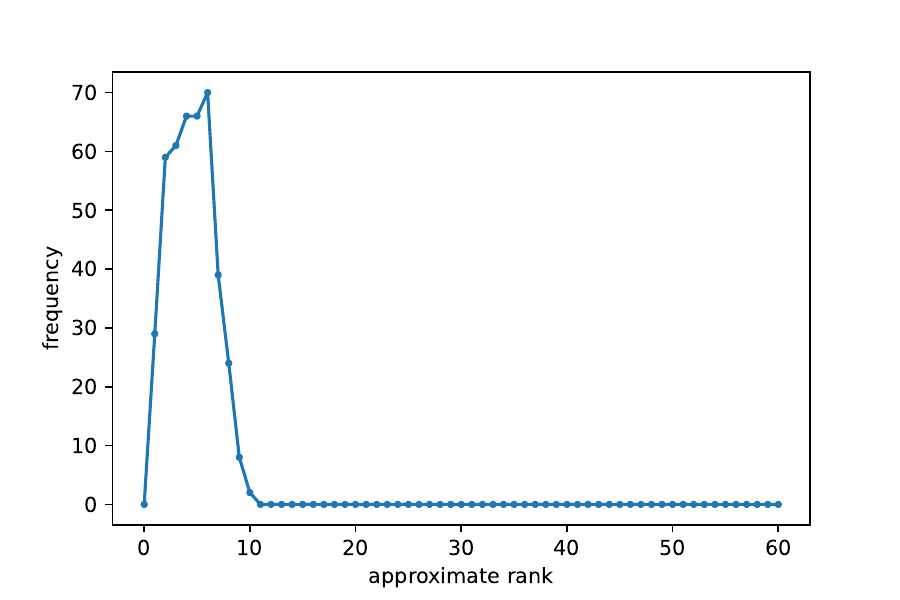}

\caption{$K = 500$, $k = 60$}  
\label{Fig: 500, logistic, 60}
\end{subfigure}
\begin{subfigure}{.33\linewidth}
\centering
  
\includegraphics[trim = 0mm 0mm 0mm 0mm, clip, width=\linewidth]{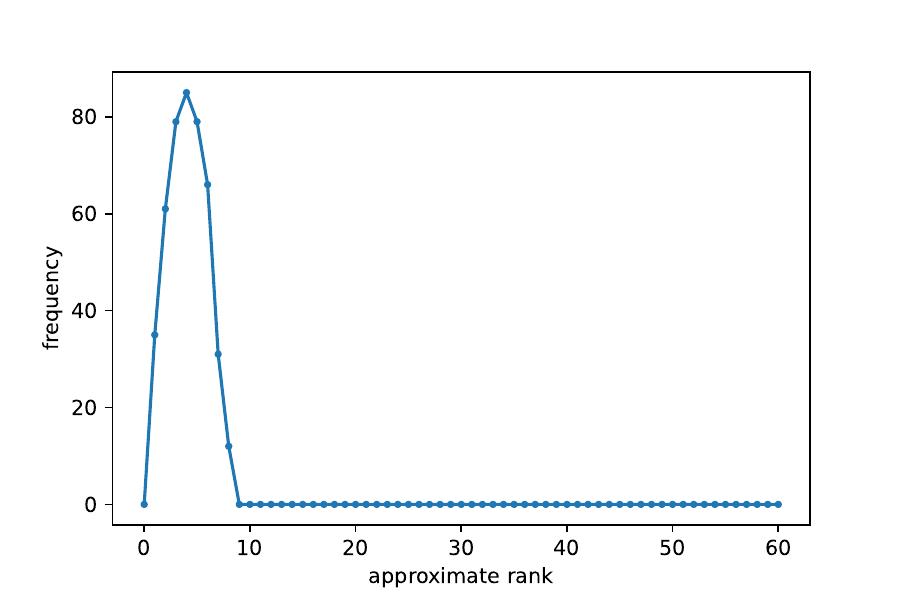}

\caption{$K = 1000$, $k = 60$}  
\label{Fig: 1000, logistic, 60} 
\end{subfigure}
\newline
\begin{subfigure}{.33\linewidth}
\centering
  
\includegraphics[trim = 0mm 0mm 0mm 0mm, clip, width=\linewidth]{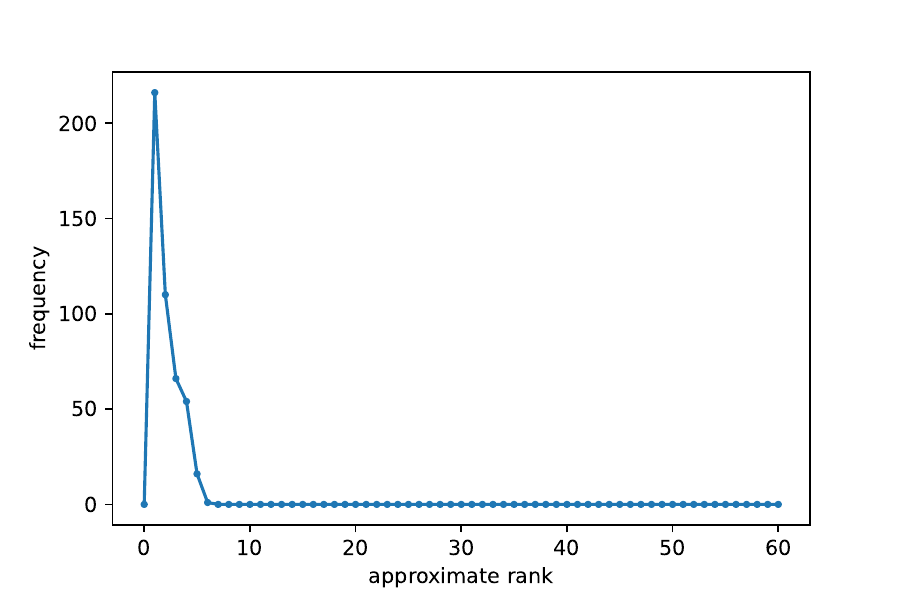}

\caption{$K = 784$, $k = 60$, RMSprop}  
\label{Fig: 784, logistic, 60, RMSprop}
\end{subfigure}
\begin{subfigure}{.33\linewidth}
\centering
  
\includegraphics[trim = 0mm 0mm 0mm 0mm, clip, width=\linewidth]{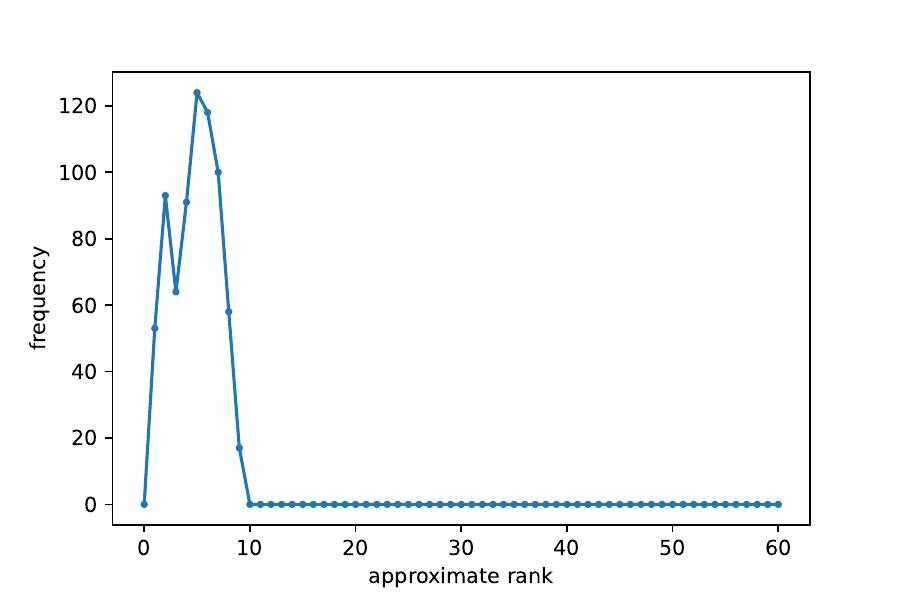}

\caption{$K = 784$, $k = 60$, two hidden layers}  
\label{Fig: 784, logistic, 60, two hidden layers}
\end{subfigure}
\begin{subfigure}{.33\linewidth}
\centering
  
\includegraphics[trim = 0mm 0mm 0mm 0mm, clip, width=\linewidth]{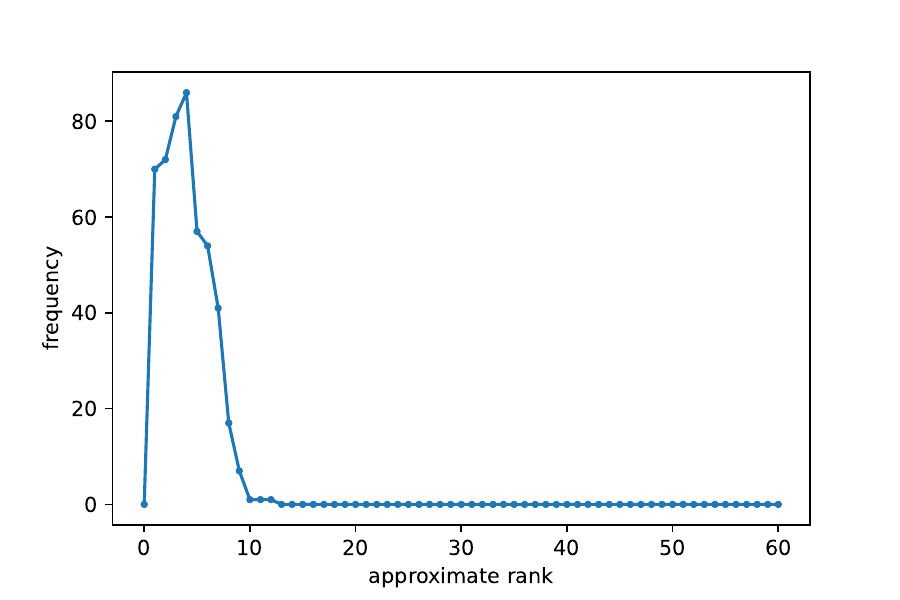}

\caption{$K = 784$, $k = 60$, ReLU}  
\label{Fig: 784, relu, 60}
\end{subfigure}

\caption{Histograms of approximate ranks for MNIST neural networks with different choices of hyper-parameters. The different values of $K$ and $k$ are stated under each plot. Unless stated otherwise, the neural network architectures for all the plots have one hidden layer, use logistic activation functions in all hidden layers, and use the Adam algorithm for training. Moreover, unless stated otherwise, the plots only consider approximate ranks of $\M_q$'s corresponding to non-trivial weights $\theta^*_q$. Note that \figref{Fig: 784, logistic, 60, all weights} displays approximate ranks for $\M_q$'s corresponding to all weights $\theta^*_q$, \figref{Fig: 784, logistic, 60, RMSprop} uses the RMSprop algorithm for training, \figref{Fig: 784, logistic, 60, two hidden layers} has two hidden layers of size $K = 784$ each, and \figref{Fig: 784, relu, 60} uses ReLU activation functions in the hidden layer.}
\label{Fig: Distributions of approximate ranks}
\end{figure*}

\textbf{MNIST plots.} In the MNIST setting, \figref{Fig: Distributions of approximate ranks} plots histograms generated using the aforementioned procedure for various choices of network hyper-parameters. We now explain these precise choices and discuss associated observations. \figref{Fig: 784, logistic, 60} depicts the prototypical histogram, where we use $K = 784$ and logistic activation functions in the hidden layer to define the neural network, Adam to train the network, and $k = 60$ to construct the tensor. The ordinate (vertical axis) of this plot represents frequency, i.e., the number of times we see a particular value, and the abscissa (horizontal axis) represents the set of possible approximate rank values. Moreover, the plot portrays that almost all $\M_q$'s corresponding to non-trivial $\theta_q^*$'s have approximate rank bounded by $10$. In fact, about half of these $\M_q$'s have approximate rank bounded by $5$. This suggests that the gradient of the loss function for this neural network model indeed exhibits some approximate low rank structure. We also note that $f$ is ``smooth'' here, because we have used smooth activation functions and the cross-entropy loss. Furthermore, the observation of low rankness does not change when we use other smooth losses instead of the cross-entropy loss. 

Compared to \figref{Fig: 784, logistic, 60}, \figref{Fig: 784, logistic, 30} and \figref{Fig: 784, logistic, 90} only change the value of $k$ to $k = 30$ and $k = 90$, respectively. It is clear from these plots that the broad qualitative observation explained above remains unchanged if $k$ is varied. In a different vein, compared to \figref{Fig: 784, logistic, 60}, \figref{Fig: 500, logistic, 60} and \figref{Fig: 1000, logistic, 60} only change the number of nodes in the hidden layer to $K = 500$ and $K = 1000$, respectively. These plots illustrate that our observations are qualitatively unchanged by different hidden layer sizes. Finally, compared to \figref{Fig: 784, logistic, 60}, \figref{Fig: 784, logistic, 60, RMSprop} only changes the optimization algorithm for training from Adam to RMSprop, and demonstrates (once again) that the choice of algorithm has essentially no effect on our observations. \figref{Fig: 784, logistic, 60, all weights} uses the same setting of network hyper-parameters as \figref{Fig: 784, logistic, 60}, but it plots a histogram of the multiset of approximate ranks of \emph{all $\M_q$'s}, i.e., $\M_q$'s corresponding to \emph{all weights} $\theta^*_q$. As mentioned earlier, the purpose of this plot is to illustrate that keeping all weights in the histograms does not change the depicted low rank structure, i.e., all $\M_q$'s have reasonably small approximate rank.

All of the experiments until now focus on smooth $f$ and neural networks with one hidden layer. To understand the effect of changing the number of hidden layers, \figref{Fig: 784, logistic, 60, two hidden layers} illustrates the histogram obtained when the neural network has two hidden layers with size $K = 784$ each and all logistic activation functions, and all other hyper-parameters are kept the same as \figref{Fig: 784, logistic, 60}. In this plot, the approximate rank of the bulk of $\M_q$'s corresponding to non-trivial $\theta_q^*$'s is bounded by $10$. This suggests that the addition of a few hidden layers does not increase the approximate rank of the gradient of $f$ by much. On the other hand, to investigate the effect of using non-smooth activation functions, \figref{Fig: 784, relu, 60} uses all ReLU activation functions in the hidden layer, and the same hyper-parameters as \figref{Fig: 784, logistic, 60} in all other respects. This histogram is qualitatively the same as that in \figref{Fig: 784, logistic, 60}. Therefore, the non-smoothness of $f$ does not seem to affect the approximate low rankness of the gradient of $f$. This final observation indicates that algorithms like $\lrgd$, which exploit approximate low rankness, could potentially be useful in non-smooth settings as well.

\begin{figure*}[t]
\begin{subfigure}{0.5\linewidth}
\centering
  
\includegraphics[trim = 0mm 0mm 0mm 0mm, clip, width=0.67\linewidth]{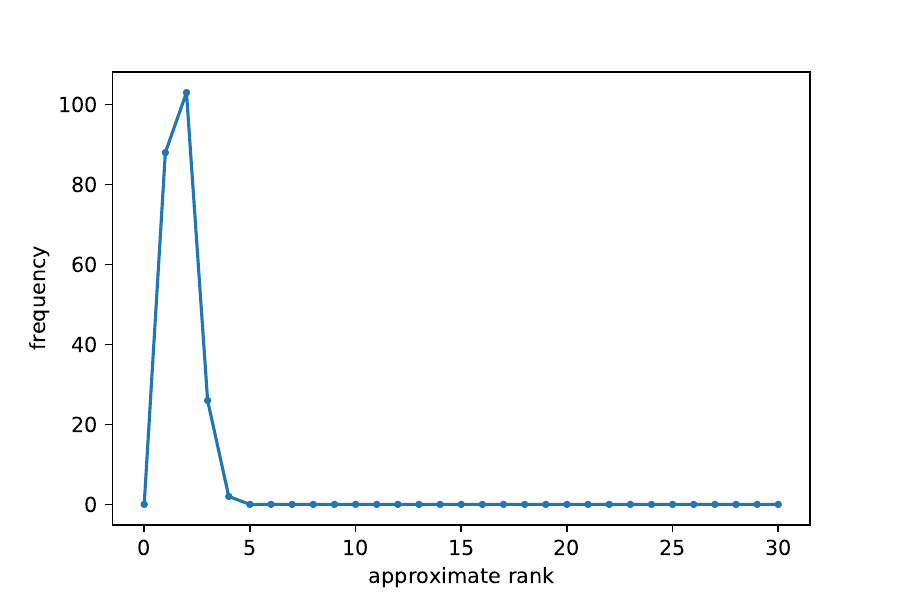}

\caption{$k = 30$}  
\label{Fig: cifar10, logistic, 30}
\end{subfigure}
\begin{subfigure}{0.5\linewidth}
\centering
  
\includegraphics[trim = 0mm 0mm 0mm 0mm, clip, width=0.67\linewidth]{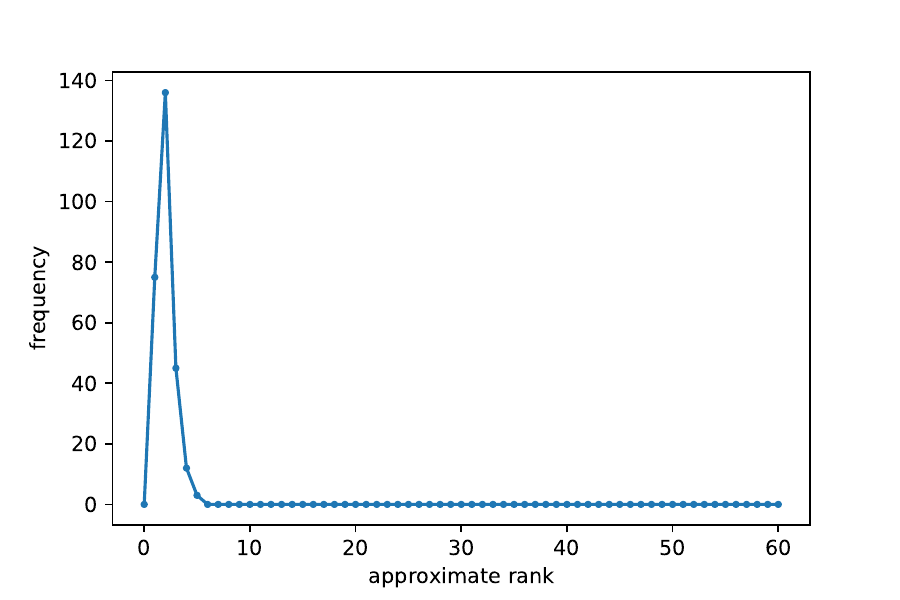}

\caption{$k = 60$}  
\label{Fig: cifar10, logistic, 60}
\end{subfigure}
\newline
\begin{subfigure}{0.5\linewidth}
\centering
  
\includegraphics[trim = 0mm 0mm 0mm 0mm, clip, width=0.67\linewidth]{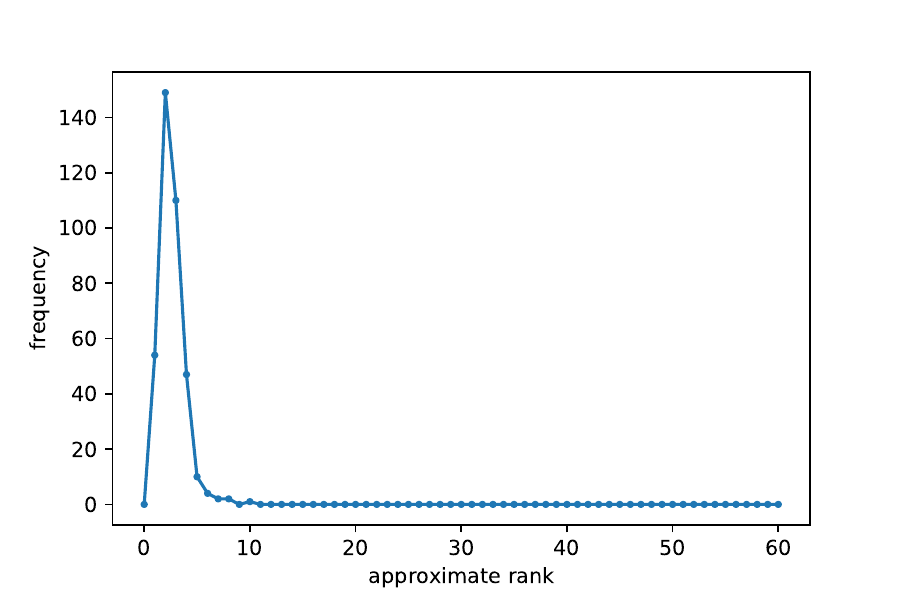}

\caption{$k = 60$, RMSprop}  
\label{Fig: cifar10, RMSprop}
\end{subfigure}
\begin{subfigure}{0.5\linewidth}
\centering
  
\includegraphics[trim = 0mm 0mm 0mm 0mm, clip, width=0.67\linewidth]{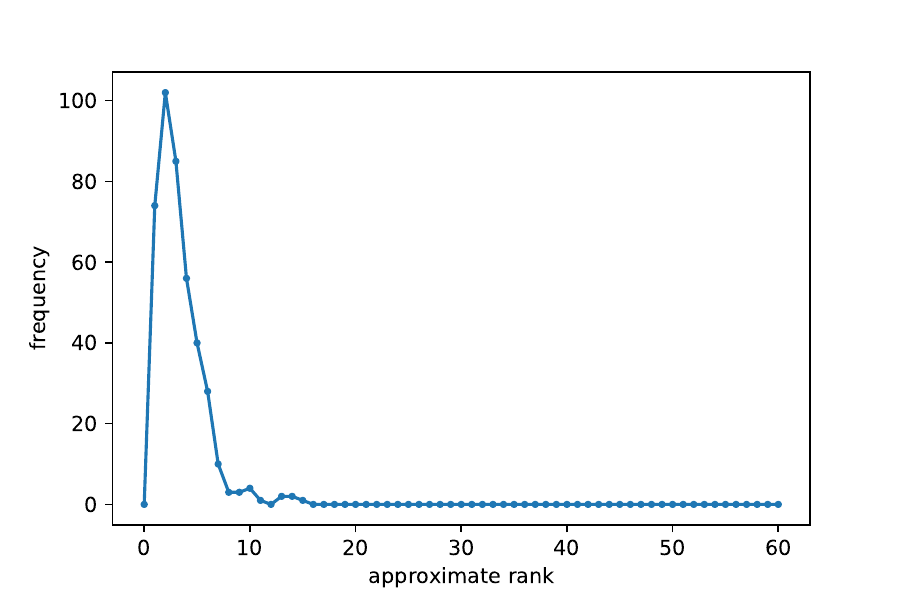}

\caption{$k = 60$, ReLU}  
\label{Fig: cifar10, relu}
\end{subfigure}

\caption{Histograms of approximate ranks for CIFAR-10 neural networks with different choices of hyper-parameters. The different values of $k$ are stated under each plot. Unless stated otherwise, the neural network architectures for all the plots have two hidden convolutional layers, a fourth max-pooling layer, and two final fully connected layers, and they use logistic activation functions in all non-sub-sampling hidden layers and the Adam algorithm for training. Moreover, unless stated otherwise, the plots only consider approximate ranks of $\M_q$'s corresponding to non-trivial weights $\theta^*_q$. Note that \figref{Fig: cifar10, RMSprop} uses the RMSprop algorithm for training and \figref{Fig: cifar10, relu} uses ReLU activation functions in all non-sub-sampling hidden layers.}
\label{Fig: Distributions of approximate ranks 2}
\end{figure*}

\textbf{CIFAR-10 plots.} In the CIFAR-10 setting, \figref{Fig: Distributions of approximate ranks 2} plots histograms generated using the aforementioned procedure for various choices of network hyper-parameters. The purpose of this figure is to demonstrate that the low rank structure depicted in \figref{Fig: Distributions of approximate ranks} can also be observed in more complex neural network architectures used for more challenging supervised learning tasks. As before, we briefly explain these precise choices and discuss associated observations. Many of our comments are similar in spirit to those in the MNIST setting.

\figref{Fig: cifar10, logistic, 60} depicts the prototypical histogram, where we use the convolutional architecture outlined earlier, logistic activation functions in all non-sub-sampling hidden layers, Adam to train the network, and $k = 60$ to construct the tensor. The axes of the plot are the same as those in \figref{Fig: Distributions of approximate ranks}. The plot illustrates that almost all $\M_q$'s corresponding to non-trivial $\theta_q^*$'s have approximate rank bounded by $6$. This corroborates the MNIST experiments that the gradient of the loss function for CIFAR-10 also exhibits some approximate low rank structure. Compared to \figref{Fig: cifar10, logistic, 60}, \figref{Fig: cifar10, logistic, 30} changes the value of $k$ to $k = 30$ to show that our qualitative observations remain unchanged if $k$ is varied, and \figref{Fig: cifar10, RMSprop} only changes the optimization algorithm for training to RMSprop, and demonstrates that the choice of algorithm has essentially no effect on our observations (as before). Lastly, while all of these plots are for a smooth loss function $f$ (because we have used smooth activation functions and the cross-entropy loss), compared to \figref{Fig: cifar10, logistic, 60}, \figref{Fig: cifar10, relu} only changes the activation functions in all non-sub-sampling hidden layers to ReLU. Once again, we observe that there is no qualitative difference in the depicted approximate low rank structure even when non-smooth activation functions are used.

\section{Proof of \thmref{Thm: Federated Oracle Complexity of FedAve}}
\label{Proof of Thm Federated Oracle Complexity of FedAve}

To prove \thmref{Thm: Federated Oracle Complexity of FedAve}, we will need two auxiliary lemmata. The first of these is the following convergence result which is distilled from \cite[Theorem 1]{Reisizadehetal2020a}.

\begin{lemma}[Convergence of $\FedAve$ {\cite[Theorem 1]{Reisizadehetal2020a}}]
\label{Lemma: Convergence of FedAve}
Let $\theta^{(T)} \in \R^p$ be the parameter vector at the server after $T$ epochs of $\FedAve$ (see \cite[Algorithm 1]{Reisizadehetal2020a}). Suppose assumptions 1 and 3 in the statement of \thmref{Thm: Federated Oracle Complexity of FedAve} hold. Then, for every 
$$ T \geq T_0 = 4 \max\!\left\{ \kappa , 4  + \frac{32 \kappa^2 (1 - \tau)}{\tau (m - 1)}, \frac{4 m}{\mu^2 b} \right\} , $$
we have
\begin{align*}
\E\!\left[F(\theta^{(T)})\right] - F_* & \leq \frac{L_1}{2} \, \E\!\left[\left\|\theta^{(T)} - \theta^*\right\|_2^2\right] \\
& \leq \frac{L_1 (T_0 b + 1)^2}{2 (T b + 1)^2} \, \E\!\left[\left\|\theta^{(T_0)} - \theta^*\right\|_2^2\right] \\
& \quad \, + \frac{8 \kappa \sigma^2}{\mu m} \left(1 + \frac{8 (1 - \tau) m}{\tau (m - 1)} \right) \frac{b}{T b + 1} \\
& \quad \, + \left(\frac{8 e L_1 \kappa^2 \sigma^2}{m}\right) \frac{(b-1)^2}{T b + 1} \\
& \quad \, + \frac{128 e \kappa^3 \sigma^2}{\mu} \!\left(\! 1 \!+ \! \frac{8 (1 - \tau)}{\tau (m - 1)}\!\right) \! \frac{b - 1}{(T b + 1)^2} ,
\end{align*}
where the first inequality follows from the smoothness in $\theta$ assumption in \secref{Assumptions on Loss Function}.
\end{lemma}

The second lemma provides a useful estimate of a trigonometric expression we will utilize in the proof.

\begin{lemma}[Taylor Approximation of Root]
\label{Lemma: Taylor Approximation of Root}
For all $t \in [0,1]$, we have\footnote{Note that we use $\arcsin(1-t^2)$ here instead of $\arcsin(1-t)$, because the latter can only be expanded as a \emph{Newton-Puiseux series}.}
\begin{align*}
\frac{1}{2} + \frac{t}{\sqrt{6}} - \frac{5 t^2}{9} & \leq \sin\!\left(\frac{1}{3} \arcsin\!\left(1 - t^2\right) + \frac{2 \pi}{3}\right) \\
& \leq \frac{1}{2} + \frac{t}{\sqrt{6}} + \frac{5 t^2}{9} \, . 
\end{align*}
\end{lemma}

\begin{proof}
Define the function $h:[0,1] \rightarrow [-1,1]$ as
$$ h(t) \triangleq \sin\!\left(\frac{1}{3} \arcsin\!\left(1 - t^2\right) + \frac{2 \pi}{3}\right) . $$
The first and second derivatives of $h$ are
\begin{align*}
h^{\prime}(t) & = -\frac{2}{3 \sqrt{2 - t^2}} \cos\!\left(\frac{1}{3} \arcsin\!\left(1 - t^2\right) + \frac{2 \pi}{3}\right) , \\
h^{\prime\prime}(t) & = - \frac{4}{9 (2 - t^2)} \sin\!\left(\frac{1}{3} \arcsin\!\left(1 - t^2\right) + \frac{2 \pi}{3}\right) \\
& \quad \, - \frac{2 t}{3 (2 - t^2)^{3/2}} \cos\!\left(\frac{1}{3} \arcsin\!\left(1 - t^2\right) + \frac{2 \pi}{3}\right) ,
\end{align*}
for all $t \in [0,1]$. Moreover, we have $h(0) = \frac{1}{2}$, $h^{\prime}(0) = \frac{1}{\sqrt{6}}$, and for any $t \in [0,1]$,
$$ |h^{\prime\prime}(t)| \leq \frac{4}{9 (2 - t^2)} + \frac{2 t}{3 (2 - t^2)^{3/2}} \leq \frac{10}{9} \, , $$
where the first inequality follows from the triangle inequality, and the second inequality follows from setting $t = 1$ since $[0,1] \ni t \mapsto \frac{4}{9 (2 - t^2)} + \frac{2 t}{3 (2 - t^2)^{3/2}}$ is an increasing function. Hence, using Taylor's theorem (with Lagrange remainder term), we get:
$$ \forall t \in [0,1], \, \exists s \in [0,t], \enspace h(t) = \frac{1}{2} + \frac{t}{\sqrt{6}} + \frac{h^{\prime\prime}(s) t^2}{2} \, . $$
This implies that:
$$ \forall t \in [0,1], \enspace \left|h(t) - \frac{1}{2} - \frac{t}{\sqrt{6}} \right| \leq \frac{5 t^2}{9} \, , $$
which completes the proof.
\end{proof}

We next establish \thmref{Thm: Federated Oracle Complexity of FedAve}.

\begin{proof}[Proof of \thmref{Thm: Federated Oracle Complexity of FedAve}]
Since we assume that $\sigma^2 = O(p)$, $b = O(m)$, $\E\big[\big\|\theta^{(T_0)} - \theta^*\big\|_2^2\big] = O(p)$, and $\tau,L_1,\mu = \Theta(1)$ with respect to $m$ (see assumptions 2-5 in the theorem statement), for every $T \geq T_0 = \Theta\big(\frac{m}{b}\big)$, each term on the right hand side of \lemref{Lemma: Convergence of FedAve} satisfies
\begin{align*}
\frac{L_1 (T_0 b + 1)^2}{2 (T b + 1)^2} \, \E\!\left[\left\|\theta^{(T_0)} - \theta^*\right\|_2^2\right] & = O\!\left(\frac{m^2 p}{T^2 b^2}\right) , \\
\frac{8 \kappa \sigma^2}{\mu m} \left(1 + \frac{8 (1 - \tau) m}{\tau (m - 1)} \right) \frac{b}{T b + 1} & = O\!\left(\frac{p}{m T}\right) , \\
\left(\frac{8 e L_1 \kappa^2 \sigma^2}{m}\right) \frac{(b-1)^2}{T b + 1} & = O\!\left(\frac{b p}{m T}\right) , \\
\frac{128 e \kappa^3 \sigma^2}{\mu} \left(1 + \frac{8 (1 - \tau)}{\tau (m - 1)}\right) \frac{b - 1}{(T b + 1)^2} & = O\!\left(\frac{p}{T^2 b}\right) .
\end{align*}
Clearly, the third term above dominates the second term, and the first term dominates the fourth term for all sufficiently large $m$ (as $b = O(m)$). Hence, \lemref{Lemma: Convergence of FedAve} implies that
$$ \E\!\left[F(\theta^{(T)})\right] - F_* \leq C \left(\frac{m^2 p}{T^2 b^2} + \frac{b p}{m T}\right) , $$
where $C > 0$ is a constant that depends on the various constant problem parameters. 

Observe that if we seek an approximation accuracy of $\epsilon > 0$, i.e., $\E\big[F(\theta^{(T)})\big] - F_* \leq \epsilon$, we need to ensure that
\begin{equation} 
\label{Eq: Inequality constraint}
\frac{m^2}{T^2 b^2} + \frac{b}{m T} \leq \frac{\epsilon}{C p} \, . 
\end{equation}
According to \defref{Def: Federated oracle complexity}, the associated federated oracle complexity of the $\FedAve$ algorithm is $T b + \phi \tau m T$, because the active clients each compute $b$ gradients in all $T$ epochs of $\FedAve$. However, we have control over the parameters $b$ and $T$ in $\FedAve$. So, we can minimize this complexity over all $b \in \N$ and $T \in \N$ subject to the approximation accuracy constraint above. Indeed, the (minimum) federated oracle complexity of $\FedAve$ is
\begin{equation}
\label{Eq: Min Fed Or Com}
\Gamma(\FedAve) = \min_{\substack{b,T \in \R_+:\\\frac{m^2}{T^2 b^2} + \frac{b}{m T} \leq \frac{\epsilon}{C p}}}{T (b + \phi \tau m)} \, , 
\end{equation}
where $\phi > 0$, $\tau \in (0,1]$, $m \in \N$, $C > 0$, $\epsilon \in (0,1]$, and $p \in \N$ are held constant in the optimization.\footnote{Note that we allow $b$ and $T$ to be real numbers in the optimization rather than natural numbers, because this does not affect the final scaling of federated oracle complexity with respect to the problem parameters. Moreover, the optimization in \eqref{Eq: Min Fed Or Com} is really an upper bound on $\Gamma(\FedAve)$ if we are to be pedantic.} We next calculate how $\Gamma(\FedAve)$ scales with $m$, $\phi$, $\epsilon$, and $p$. For convenience, we let $\varepsilon = \epsilon/p$ in the sequel.

First, notice that the inequality constraint in \eqref{Eq: Inequality constraint} can be changed into an equality constraint
$$ \frac{m^2}{T^2 b^2} + \frac{b}{m T} = \frac{\varepsilon}{C} \, . $$
Indeed, if the optimal $T$ and $b$ for \eqref{Eq: Min Fed Or Com} do not achieve \eqref{Eq: Inequality constraint} with equality, we can reduce $T$ to achieve \eqref{Eq: Inequality constraint} with equality, and this reduces the objective value of $T (b + \phi \tau m)$ further (producing a contradiction). Rearranging this equality constraint, we obtain the following quadratic equation in $T$:
$$ \frac{\varepsilon T^2}{C} - \frac{b T}{m} - \frac{m^2}{b^2} = 0 \, , $$
whose unique non-negative solution is
\begin{equation}
\label{Eq: Quadratic solution}
T = \frac{b C}{2 m \varepsilon} + \sqrt{\left(\frac{b C}{2 m \varepsilon}\right)^{\! 2} + \frac{C m^2}{\varepsilon b^2}} \, .
\end{equation}
This solution satisfies the bounds
\begin{equation}
\label{Eq: bounds}
\frac{b C}{2 m \varepsilon} + \frac{m}{b} \sqrt{\frac{C}{\varepsilon}} \leq T \leq 2 \left(\frac{b C}{2 m \varepsilon} + \frac{m}{b} \sqrt{\frac{C}{\varepsilon}} \right) , 
\end{equation}
where the lower bound follows from neglecting the $(b C/(2 m \varepsilon))^{2} \geq 0$ term in \eqref{Eq: Quadratic solution}, and the upper bound follows from the sub-additivity of the square root function. (Note that $T \geq T_0 = \Theta\big(\frac{m}{b}\big)$ is satisfied.) Now define the constants
\begin{align}
\label{Eq: const 1}
C_1 & = \frac{C}{2 m \varepsilon} \, , \\\
\label{Eq: const 2}
C_2 & = m \sqrt{\frac{C}{\varepsilon}} \, , \\
\label{Eq: const 3}
C_3 & = \phi \tau m \, ,
\end{align}
and the function $g : (0,\infty) \rightarrow \R$ as:
\begin{align*}
\forall b > 0, \enspace g(b) & \triangleq \left(C_1 b + \frac{C_2}{b} \right) \! (b + C_3) \\
& = C_1 b^2 + C_1 C_3 b + \frac{C_2 C_3}{b} + C_2 \, . 
\end{align*}
Then, using \eqref{Eq: Min Fed Or Com}, \eqref{Eq: bounds}, \eqref{Eq: const 1}, \eqref{Eq: const 2}, and \eqref{Eq: const 3}, we have
\begin{equation}
\label{Eq: Estimate of Federated oracle complexity}
\min_{b > 0}{g(b)} \leq \Gamma(\FedAve) \leq 2 \min_{b > 0}{g(b)} \, .
\end{equation}
So, we focus on evaluating $\min_{b > 0}{g(b)}$ from hereon. 

Since $C_1,C_2,C_3 > 0$, it is straightforward to verify that $g$ is a convex function on $(0,\infty)$. This means that any first order stationary point of $g$ in $(0,\infty)$ is a global minimum. We will find such a stationary point in the sequel. To this end, notice that 
$$ \forall b > 0, \enspace g^{\prime}(b) = 2 C_1 b + C_1 C_3 - \frac{C_2 C_3}{b^2} \, . $$
Setting this equal to $0$ yields the stationarity condition 
\begin{equation}
\label{Eq: Stationary condition}
2 C_1 b^3 + C_1 C_3 b^2 - C_2 C_3 = 0 \, . 
\end{equation}
Using the substitution $b = a - \frac{C_3}{6}$ above, we obtain:
\begin{align*}
0 & = 2 C_1 \left(a - \frac{C_3}{6}\right)^{\! 3} + C_1 C_3 \left(a - \frac{C_3}{6}\right)^{\! 2} - C_2 C_3 \\
& = (2 C_1) a^3 - \left(\frac{ C_1 C_3^2}{6} \right)\! a + \left(\frac{C_1 C_3^3}{54} - C_2 C_3\right) ,
\end{align*}
which we rearrange to get the (monic) \emph{depressed} cubic equation:
\begin{equation}
\label{Eq: Depressed cubic}
a^3 - \left(\frac{C_3^2}{12} \right)\! a + \left(\frac{C_3^3}{108} - \frac{C_2 C_3}{2 C_1}\right) = 0 \, . 
\end{equation}
The discriminant of this depressed cubic equation is
\begin{align*}
& 4 \left(\frac{C_3^2}{12} \right)^{\! 3} - 27 \left(\frac{C_3^3}{108} - \frac{C_2 C_3}{2 C_1}\right)^{\! 2} \\
& \qquad \qquad = \frac{C_2 C_3^4}{4 C_1} - \frac{27 C_2^2 C_3^2}{4 C_1^2} \\
& \qquad \qquad = \frac{\phi^4 \tau^4 m^6 \sqrt{\varepsilon}}{2 \sqrt{C}} - \frac{27 \phi^2 \tau^2 m^6 \varepsilon}{C} \\
& \qquad \qquad = \underbrace{\phi^2 \tau^2 m^6 \sqrt{\frac{\varepsilon}{C}}}_{> 0} \left(\frac{1}{2} \phi^2 \tau^2 - 27 \sqrt{\frac{\varepsilon}{C}} \right) \\
& \qquad \qquad > 0 \, ,
\end{align*}
where the second equality follows from \eqref{Eq: const 1}, \eqref{Eq: const 2}, and \eqref{Eq: const 3}, and the strict positivity holds because $\tau$ and $C$ are constants, $\varepsilon = \epsilon/p \in (0,1]$, and $\phi$ is sufficiently large (compared to these constant problem parameters). In particular, we will assume throughout this proof that
\begin{equation}
\label{Eq: large phi}
\phi \geq \frac{40}{C^{1/4} \tau} \, .
\end{equation}
In the previous case, \eqref{Eq: large phi} clearly implies that
$$ \phi > \frac{3 \sqrt{6}}{\tau} \left(\frac{\varepsilon}{C}\right)^{\! 1/4} \quad \Leftrightarrow \quad \frac{1}{2} \phi^2 \tau^2 - 27 \sqrt{\frac{\varepsilon}{C}} > 0 \, . $$
Thus, the depressed cubic equation in \eqref{Eq: Depressed cubic} has three distinct real solutions. Using \emph{Vi\`{e}te's trigonometric formula} for the roots of a depressed cubic equation \cite{Nickalls1993,Nickalls2006}, we have that 
$$ a^* = \frac{C_3}{3} \sin\!\left(\frac{1}{3} \arcsin\!\left(1 - \frac{54 C_2}{C_1 C_3^2}\right) + \frac{2 \pi}{3}\right) $$
is one of the solutions to \eqref{Eq: Depressed cubic}, and hence,
$$ b^* = \frac{C_3}{3} \left( \sin\!\left(\frac{1}{3} \arcsin\!\left(1 - \frac{54 C_2}{C_1 C_3^2}\right) + \frac{2 \pi}{3}\right) - \frac{1}{2} \right) $$
is one of the solutions to \eqref{Eq: Stationary condition}. Applying \lemref{Lemma: Taylor Approximation of Root} to $b^*$ and simplifying yields
$$ \sqrt{\frac{C_2}{C_1}} - \frac{10 C_2}{C_1 C_3} \leq b^* \leq \sqrt{\frac{C_2}{C_1}} + \frac{10 C_2}{C_1 C_3} \, , $$
where we use \eqref{Eq: large phi} and the fact that $\varepsilon \in (0,1]$, which imply that the condition required by \lemref{Lemma: Taylor Approximation of Root} is satisfied:
$$ \phi \geq \frac{6 \sqrt{3}}{\tau} \left(\frac{\varepsilon}{C}\right)^{\! 1/4}  \quad \Leftrightarrow \quad \frac{3}{C_3} \sqrt{\frac{6 C_2}{C_1}} \leq 1 \, . $$
Then, substituting \eqref{Eq: const 1}, \eqref{Eq: const 2}, and \eqref{Eq: const 3} into these bounds gives
$$ \sqrt{2} \, m \! \left(\frac{\varepsilon}{C}\right)^{\! 1/4} - \frac{20 m}{\phi \tau} \sqrt{\frac{\varepsilon}{C}} \leq b^* \leq \sqrt{2} \, m \! \left(\frac{\varepsilon}{C}\right)^{\! 1/4} + \frac{20 m}{\phi \tau} \sqrt{\frac{\varepsilon}{C}} . $$
Since $\varepsilon \in (0,1]$ and \eqref{Eq: large phi} holds, we have
$$ \phi \geq \frac{40 \varepsilon^{1/4}}{C^{1/4} \tau} \quad \Leftrightarrow \quad \frac{20 m}{\phi \tau} \sqrt{\frac{\varepsilon}{C}} \leq \frac{m}{2} \! \left(\frac{\varepsilon}{C}\right)^{\! 1/4} \, , $$
and we see that $b^* = \Theta(m \varepsilon^{1/4})$; specifically, 
\begin{equation}
\label{Eq: Bounds on root}
\frac{m}{2} \left(\frac{\varepsilon}{C}\right)^{\! 1/4} \leq b^* \leq 2 m \! \left(\frac{\varepsilon}{C}\right)^{\! 1/4} . 
\end{equation}
(Note that $b^* = O(m)$ is satisfied.) Furthermore, $b^* > 0$ is the unique positive solution to the stationarity condition \eqref{Eq: Stationary condition} by \emph{Descartes' rule of signs}. Therefore, using \eqref{Eq: const 1}, \eqref{Eq: const 2}, and \eqref{Eq: const 3}, we have
\begin{align*}
\min_{b > 0}{g(b)} & = C_1 (b^*)^2 + C_1 C_3 b^* + \frac{C_2 C_3}{b^*} + C_2 \\
& = \frac{C}{2 m \varepsilon} (b^*)^2 + \frac{C \phi \tau}{2 \varepsilon}  b^* + \frac{m^2 \phi \tau}{b^*} \sqrt{\frac{C}{\varepsilon}} + m \sqrt{\frac{C}{\varepsilon}} \, .
\end{align*}
We can now estimate the scaling of $\min_{b > 0}{g(b)}$. Indeed, by applying the inequalities in \eqref{Eq: Bounds on root}, we have
\begin{align*}
\frac{3 \phi \tau m}{4} \left(\frac{C}{\varepsilon}\right)^{\! 3/4} \!\!\!\! & \leq \frac{9 m}{8} \sqrt{\frac{C}{\varepsilon}} + \frac{3 \phi \tau m}{4} \left(\frac{C}{\varepsilon}\right)^{\! 3/4} \\
& \leq \min_{b > 0}{g(b)} \\
& \leq 3 m \sqrt{\frac{C}{\varepsilon}} + 3 \phi \tau m \! \left(\frac{C}{\varepsilon}\right)^{\! 3/4} \!\!\!\! \leq 4 \phi \tau m \! \left(\frac{C}{\varepsilon}\right)^{\! 3/4} \!\! , 
\end{align*}
where the last inequality again uses \eqref{Eq: large phi} and the fact that $\varepsilon \in (0,1]$, which imply that
$$ \phi \geq \frac{3}{\tau} \left(\frac{\varepsilon}{C}\right)^{\! 1/4} \quad \Leftrightarrow \quad 3 m \sqrt{\frac{C}{\varepsilon}} \leq \phi \tau m \! \left(\frac{C}{\varepsilon}\right)^{\! 3/4} . $$

Finally, from the bounds in \eqref{Eq: Estimate of Federated oracle complexity}, we can estimate the scaling of $\Gamma(\FedAve)$:
$$ \frac{3 \phi \tau m}{4} \left(\frac{C}{\varepsilon}\right)^{\! 3/4} \leq \Gamma(\FedAve) \leq 8 \phi \tau m \! \left(\frac{C}{\varepsilon}\right)^{\! 3/4} . $$
Since $\varepsilon = \epsilon/p$, this establishes that $\Gamma(\FedAve) = \Theta(\phi m (p/\epsilon)^{3/4})$. More pedantically, this argument shows that $\Gamma(\FedAve) = O(\phi m (p/\epsilon)^{3/4})$, because the constraint \eqref{Eq: Inequality constraint} uses the (worst case $O(p)$ scaling) bounds in assumptions 3 and 5 of the theorem statement. This completes the proof.
\end{proof}

\section{Proof of \propref{Prop: Comparison between FedLRGD and FedAve}}
\label{Proof of Prop Comparison between FedLRGD and FedAve}

\begin{proof}
We first calculate the federated oracle complexity of $\lrgd$. Since $d = O(\log(n)^{\lambda})$, without loss of generality, suppose that $d \leq c_3 \log(n)^{\lambda}$ for some constant $c_3 > 0$. From \eqref{Eq: Rank lower bound condition} in \thmref{Thm: Federated Oracle Complexity of FedLRGD},
\begin{align*}
r & = \! \left\lceil\rule{0cm}{1.1cm}\right. \!\! \max\!\left\{\rule{0cm}{1.1cm}\right. \!\!\! e \sqrt{d} \, 2^d (\eta + d)^d , \! \left(\rule{0cm}{0.9cm}\right. \!\!\!\! \left(\!\frac{L_2^d \eta^{d} e^{\eta (d + 1)} d^{\eta / 2} (2\eta + 2d)^{\eta d}}{(\eta - 1)^{\eta d}}\!\right) \\
& \qquad\quad\, \cdot \! \left(\!\frac{\kappa \! \left(F(\theta^{(0)}) - F_* + \frac{9 (B + 3)^4 p}{2 \mu}\right)}{(\kappa - 1)\epsilon}\!\right)^{\!\! \frac{d}{2}} \!\! \left.\rule{0cm}{0.9cm}\right)^{\!\! \frac{1}{\eta - (2 \alpha + 2)d}} \!\! \left.\rule{0cm}{1.1cm}\right\} \!\!\left.\rule{0cm}{1.1cm}\right\rceil .
\end{align*}
Let us analyze each salient term in this equation:
\begin{align*}
e \sqrt{d} \, 2^d (\eta + d)^d & \leq e (2c_2 + 4 \alpha + 6)^d d^{d + (1/2)} \\
& = O\!\left(d^{2 d + (1/2)}\right) \\
& = O\!\left(\exp\!\left(2 \lambda c_3 \log(n)^{(\lambda+1)/2}\right)\right) , 
\end{align*}
\begin{align*}
& \left(L_2^d \eta^{d} (2 e)^{\eta (d + 1)} d^{\eta / 2} \left(\frac{\eta + d}{\eta - 1}\right)^{\! \eta d}\right)^{\! \frac{1}{\eta - (2 \alpha + 2)d}} \\
& = O\!\left( d^{(c_2 + 2\alpha + 4)/(2 c_1)} \left(\frac{2 e(c_2 + 2\alpha + 3)}{c_1 + 2\alpha + 1} \right)^{\! (c_2 + 2\alpha + 2) d/c_1} \right) \\
& = O\Bigg(\log(n)^{\lambda (c_2 + 2\alpha + 4)/(2 c_1)} \\
& \qquad \quad \cdot \left(\frac{2 e(c_2 + 2\alpha + 3)}{c_1 + 2\alpha + 1} \right)^{\! c_3 (c_2 + 2\alpha + 2) \log(n)^{\lambda} /c_1} \Bigg) , 
\end{align*}
\begin{align*}
& \left(\frac{\kappa \! \left(F(\theta^{(0)}) - F_* + \frac{9 (B + 3)^4 p}{2 \mu}\right)}{(\kappa - 1)\epsilon}\right)^{\!\! \frac{d}{2\eta - (4 \alpha + 4)d}} \\
& \qquad \qquad \qquad \qquad  \qquad \qquad = O\!\left(\left(\frac{p}{\epsilon}\right)^{\! 1/(2 c_1)}\right) ,
\end{align*}
where the first estimate uses the facts that $\eta \leq (c_2 + 2\alpha + 2) d$ and $d \leq c_3 \log(n)^{\lambda}$, the second estimate uses the facts that $(c_1 + 2 \alpha + 2) d \leq \eta \leq (c_2 + 2 \alpha + 2) d$ and $d \leq c_3 \log(n)^{\lambda}$, and the third estimate uses the facts that $(c_1 + 2 \alpha + 2) d \leq \eta$ and $F(\theta^{(0)}) - F_* = O(p)$. Hence, since $\epsilon = \Theta(n^{-\beta})$, $p/\epsilon = \Omega(n^{\beta})$ and $r$ grows at least polynomially in $n$, which implies that
\begin{equation}
\label{Eq: Rank scaling}
r = O\!\left(\subpoly(n) \left(\frac{p}{\epsilon}\right)^{\! 1/(2 c_1)}\right) , 
\end{equation}
where $\subpoly(n)$ denotes a term that is dominated by any polynomial in $n$, i.e., $\subpoly(n) = O(n^{\tau})$ for any $\tau > 0$. 

On the other hand, from \eqref{Eq: Iteration count} in \thmref{Thm: Federated Oracle Complexity of FedLRGD},
\begin{align*}
S & = \left\lceil \left(\log\!\left(\frac{\kappa}{\kappa-1}\right)\right)^{\! -1} \log\!\left(\frac{F(\theta^{(0)}) - F_* + \frac{9 (B + 3)^4 p}{2 \mu}}{\epsilon} \right) \right\rceil \\
& = \Theta\!\left( \log\!\left(\frac{p}{\epsilon} \right)\right) \\
& = O(\log(r)) \, , 
\end{align*}
where the second equality uses the fact that $F(\theta^{(0)}) - F_* = O(p)$, and the third equality follows from \eqref{Eq: Rank scaling}. Therefore, the federated oracle complexity of $\lrgd$ from \thmref{Thm: Federated Oracle Complexity of FedLRGD} scales as
\begin{align*}
\Gamma(\lrgd) & = r^2 + rs + rS + \phi m r \\
& \leq O(\phi m r^2) \\
& = O\!\left(\subpoly(n) \, \phi m \left(\frac{p}{\epsilon}\right)^{\! 1/ c_1} \right) \\
& \leq O\!\left( \phi m \left(\frac{p}{\epsilon}\right)^{\! 2/ c_1} \right) \\
& \leq O\!\left( \phi m \left(\frac{p}{\epsilon}\right)^{\! \xi} \right) ,
\end{align*}
where the second line holds because $s = O(\phi m)$ and $S = \log(r)$, the fourth line holds because $\subpoly(n) = O((p/\epsilon)^{1/c_1})$ since $\epsilon = \Theta(n^{-\beta})$, and the last line holds because $c_1 \geq 2/\xi$.

Furthermore, rewriting \thmref{Thm: Federated Oracle Complexity of FedAve}, the federated oracle complexity of $\FedAve$ scales as
$$ \Gamma(\FedAve) \leq \Gamma_+(\FedAve) = \Theta\!\left(\phi m \left(\frac{p}{\epsilon}\right)^{\! 3/4}\right) . $$
Hence, since $\xi < \frac{3}{4}$, we have
$$ \lim_{m \rightarrow \infty}{\frac{\Gamma(\lrgd)}{\Gamma_+(\FedAve)}} \leq \lim_{m \rightarrow \infty}{c_4 \left(\frac{p}{\epsilon}\right)^{\! \xi - (3/4)}} = 0 \, , $$
as desired, where $c_4 > 0$ is some constant. This completes the proof.
\end{proof}

\bibliographystyle{IEEEtran}
\bibliography{refs.bib}

\end{document}